\documentclass{article} % For LaTeX2e
\usepackage{iclr2026_conference,times}

\usepackage{amsmath,amsfonts,bm}

\def\eqref#1{equation~\ref{#1}}
\def\1{\bm{1}}

\DeclareMathAlphabet{\mathsfit}{\encodingdefault}{\sfdefault}{m}{sl}
\SetMathAlphabet{\mathsfit}{bold}{\encodingdefault}{\sfdefault}{bx}{n}

\newcommand{\pdata}{p_{\rm{data}}}
\usepackage{bbm}
\usepackage{bm}
\usepackage{xcolor}
\usepackage{amsthm}
\usepackage{framed}
\usepackage{mdframed}
\usepackage{graphicx}
\usepackage{algorithmic}
\usepackage{algorithm}
\usepackage{pifont}
\usepackage{mathtools}

\newtheorem{assumption}{Assumption}
\newtheorem{theorem}{Theorem}
\newtheorem{remark}{Remark}
\newtheorem{lemma}{Lemma}
\newtheorem{definition}{Definition}

\newcommand{\blue}{\color{blue}}

\newcommand\bw{\ensuremath{{\bm w}}}

\newcommand\bx{\ensuremath{{\bm x}}}
\newcommand\by{\ensuremath{{\bm y}}}
\newcommand\bu{\ensuremath{{\bm u}}}

\newcommand\pik{\ensuremath{{\pi}^{(k)}}}
\newcommand\tpik{\ensuremath{{\tilde{\pi}}^{(k)}}}
\newcommand\tpio{\ensuremath{{\tilde{\pi}}^{(k)}_0}}

\newcommand\bxk{\ensuremath{{\bm x}^{(k)}}}

\newcommand\bzk{\ensuremath{{\bm z}^{(k)}}}

\newcommand\bsk{\ensuremath{{\bm s}^{(k)}}}
\newcommand\tbsk{\ensuremath{{\tilde{\bm s}}^{(k)}}}
\newcommand\buk{\ensuremath{{\bm u}^{(k)}}}
\newcommand\bvk{\ensuremath{{\bm v}^{(k)}}}
\newcommand\blamk{\ensuremath{{\bm \lambda}^{(k)}}}
\newcommand\bxkk{\ensuremath{{\bm x}^{(k+1)}}}

\newcommand\bzkk{\ensuremath{{\bm z}^{(k+1)}}}
\newcommand\bukk{\ensuremath{{\bm u}^{(k+1)}}}
\newcommand\bvkk{\ensuremath{{\bm v}^{(k+1)}}}
\newcommand\tbz{\ensuremath{\widetilde{\bm z}}}

\newcommand\tbzk{\ensuremath{{\widetilde{\bm z}}^{(k)}}}
\newcommand\tbxk{\ensuremath{{\widetilde{\bm x}}^{(k)}}}
\newcommand\tbxkk{\ensuremath{{\widetilde{\bm x}}^{(k+1)}}}

\newcommand\bzack{\ensuremath{{{\bm z}^{(k)}_{\rm ac}}}}
\newcommand\bzdck{\ensuremath{{{\bm z}^{(k)}_{\rm dc}}}}

\newcommand\bztwk{\ensuremath{{{\bm z}^{(k)}_{\rm tw}}}}
\newcommand\bzodek{\ensuremath{{{\bm z}^{(k)}_{\rm ode}}}}

\renewcommand\pdata{\ensuremath{p_{\rm data}}}

\newcommand\Hack{\ensuremath{{{ H}^{(k)}_{\rm ac}}}}
\newcommand\Hdck{\ensuremath{{{ H}^{(k)}_{\rm dc}}}}
\newcommand\Htwk{\ensuremath{{{ H}^{(k)}_{\rm tw}}}}
\newcommand\Hodek{\ensuremath{{{ H}^{(k)}_{\rm ode}}}}

\newcommand\Rack{\ensuremath{{{ R}^{(k)}_{\rm ac}}}}
\newcommand\Rdck{\ensuremath{{{ R}^{(k)}_{\rm dc}}}}
\newcommand\Rtwk{\ensuremath{{{ R}^{(k)}_{\rm tw}}}}
\newcommand\Rodek{\ensuremath{{{ R}^{(k)}_{\rm ode}}}}

\newcommand\epack{\ensuremath{{{ \epsilon}^{(k)}_{\rm ac}}}}
\newcommand\epdck{\ensuremath{{{ \epsilon}^{(k)}_{\rm dc}}}}
\newcommand\eptwk{\ensuremath{{{ \epsilon}^{(k)}_{\rm tw}}}}
\newcommand\epodek{\ensuremath{{{ \epsilon}^{(k)}_{\rm ode}}}}

\newcommand\delack{\ensuremath{{{ \delta}^{(k)}_{\rm ac}}}}
\newcommand\deldck{\ensuremath{{{ \delta}^{(k)}_{\rm dc}}}}
\newcommand\deltwk{\ensuremath{{{ \delta}^{(k)}_{\rm tw}}}}
\newcommand\delodek{\ensuremath{{{ \delta}^{(k)}_{\rm ode}}}}

\newcommand \nuk {\ensuremath{\nu_{k}}}
\newcommand \ck {\ensuremath{c_{k}}}

\newcommand \betakk {\ensuremath{\beta_{k+1}}}
\newcommand \betak {\ensuremath{\beta_{k}}}

\newcommand \rhokk {\ensuremath{\rho_{k+1}}}
\newcommand \rhok {\ensuremath{\rho_{k}}}

\newcommand\etak{\ensuremath{\eta}^{(k)}}

\newcommand\bstheta{\ensuremath{{{\bm s}_{\theta}}}}

\newcommand\bSigma{\ensuremath{\bm \Sigma}}

\newcommand\bznk{\ensuremath{{{\bm z}}^{(k)}_{\natural}}}

\newcommand\sigmak{\ensuremath{{\sigma}^{(k)}}}
\newcommand\bxi{\ensuremath{{\bm \xi}}}

\newcommand\bz{\ensuremath{{\bm z}}}
\newcommand\bn{\ensuremath{{\bm n}}}

\newcommand\blam{\ensuremath{{\bm \lambda}}}
\newcommand\bmu{\ensuremath{{\bm \mu}}}

\newcommand\bv{\ensuremath{{\bm v}}}

\newcommand\bs{\ensuremath{{\bm s}}}

\newcommand{\Rbb}{\mathbb{R}}

\newcommand{\Nbb}{\mathbb{N}}
\newcommand{\Nbbp}{\mathbb{N}^+}

\newcommand{\setA}{\mathcal{A}}

\newcommand{\setX}{\mathcal{X}}

\newcommand{\setM}{\mathcal{M}}

\newcommand{\setW}{\mathcal{W}}

\newcommand{\setN}{\mathcal{N}}

\newcommand{\Exp}{\mathbb{E}}

\newcommand{\bI}{{\bm I}}

\newcommand\Prox{\ensuremath{{\rm Prox}}}

\newcommand\A[1]{{\sf A#1)}}

\newcommand\norm[1]{\ensuremath{{\left\| #1 \right\|}}}
\newcommand\vecdot[1]{\ensuremath{{\langle #1 \rangle}}}

\newcommand\kl{\ensuremath{\text{KL}}}
\newcommand\tr{\ensuremath{\text{tr}}}

\usepackage[utf8]{inputenc} % allow utf-8 input
\usepackage[T1]{fontenc}    % use 8-bit T1 fonts
\usepackage{hyperref}       % hyperlinks
\usepackage{url}            % simple URL typesetting
\usepackage{booktabs}       % professional-quality tables
\usepackage{amsfonts}       % blackboard math symbols
\usepackage{nicefrac}       % compact symbols for 1/2, etc.
\usepackage{microtype}      % microtypography
\usepackage{xcolor}         % colors
\usepackage{amsmath}
\usepackage{amssymb,wrapfig}
\usepackage{subcaption,comment}
\usepackage{multirow}
\usepackage{makecell}
\usepackage{amssymb}

\usepackage[normalem]{ulem}    % for \underline

\title{Taming Score-Based Denoisers in ADMM: \\A Convergent Plug-and-Play Framework}

\author{
  Rajesh Shrestha \\
  School of EECS\\
  Oregon State University\\
  \texttt{shresthr@oregonstate.edu} \\
  \And
  Xiao Fu \\
  School of EECS\\
  Oregon State University\\
  \texttt{xiao.fu@oregonstate.edu} \\
}

\usepackage{amsmath} % already in the template usually
\renewcommand{\eqref}[1]{(\ref{#1})}

\iclrfinalcopy % Uncomment for camera-ready version, but NOT for submission.
\begin{document}

\maketitle
% =============================================================================
% Main Document Sections
% =============================================================================
\begin{abstract}
While score-based generative models have emerged as powerful priors for solving inverse problems, directly integrating them into optimization algorithms such as ADMM remains nontrivial. Two central challenges arise: i) the mismatch between the noisy data manifolds used to train the score functions and the geometry of ADMM iterates, especially due to the influence of dual variables, and ii) the lack of convergence understanding when ADMM is equipped with score-based denoisers. To address the manifold mismatch issue, we propose ADMM plug-and-play (ADMM-PnP) with the AC-DC denoiser, a new framework that embeds a three-stage denoiser into ADMM: (1) auto-correction (AC) via additive Gaussian noise, (2) directional correction (DC) using conditional Langevin dynamics, and (3) score-based denoising. In terms of convergence, we establish two results: first, under proper denoiser parameters, each ADMM iteration is a weakly nonexpansive operator, ensuring high-probability fixed-point \textit{ball convergence} using a constant step size; second, under more relaxed conditions, the AC-DC denoiser is a bounded denoiser, which leads to convergence under an adaptive step size schedule. Experiments on a range of inverse problems demonstrate that our method consistently improves solution quality over a variety of baselines. Source code is publicly available at \url{https://github.com/rajeshshrestha/ACDC.git}.
\end{abstract}

% =============================================================================
% Main Document Sections
% =============================================================================
\section{Introduction}
Inverse problems arise in many fields, including medical imaging \citep{song2022solving,jin2017deep,arridge1999optical}, remote sensing \citep{entekhabi1994solving,combal2003retrieval}, oceanography \citep{bennett1992inverse}, and computational physics \citep{raissi2019physics,tarantola2005inverse}. Their solutions typically rely on incorporating prior knowledge or structural assumptions about the target signals, either through explicit regularization or data-driven models.

Classical approaches to inverse problems often rely on handcrafted regularizers, such as the $\ell_1$ norm for sparsity~\citep{yang2010imagesuperresolution,elad2006imagedenoising,dabov2007imagedenoising} and the nuclear norm for low-rank structure~\citep{semerci2014tensorbased,hu2017thetwist}. Deep learning introduced a new paradigm of using learned generative models---VAEs, GANs, and normalizing flows---as data-driven regularizers~\citep{ulyanov2020deepimageprior,alkhouri2024image,shah2018solving}, offering more expressive priors by capturing complex distributions. More recently, pre-trained score functions from diffusion models have gained attention for inverse problems~\citep{song2022solving,chung2023diffusion}, as they effectively approximate data distributions and align solutions with the underlying data geometry~\citep{xiao2022tackling}.

The use of pre-trained score functions in inverse problems mainly falls into two categories. The first modifies the MCMC process of diffusion sampling to incorporate observation information, as in DPS \citep{chung2023diffusion} and DDRM \citep{kawar2022denoisingdiffusionrestorationmodels}, where observations guide unconditional score functions to perform posterior sampling. The second integrates score functions into deterministic or stochastic optimization algorithms; for example, \citet{wang2024dmplug} and \citet{song2023solving} use them as “projectors” to keep iterates on the desired data manifold. Furthermore, building on Tweedie’s lemma, which links score functions to signal denoising, works such as \citep{zhu2023denoising,mardani2023avariational,renaud2024plugandplay} employ score as denoisers in proximal-gradient-like steps.

{\bf Challenges.} 
The works in \citep{zhu2023denoising,mardani2023avariational,li2024decoupled} present flexible ``plug-and-play (PnP)’’ paradigms that integrate diffusion models with optimization algorithms. However, two challenges remain in this line of work. First, score functions are trained on noisy data manifolds constructed via Gaussian perturbations, whereas optimization iterates need not lie on such manifolds, leading to geometry mismatch and degraded denoising performance. Remedies such as stochastic regularization \citep{renaud2024plugandplay} or purification \citep{li2024decoupled} add Gaussian noise to the iterates, but this does not guarantee alignment with the score manifolds. Second, the theoretical understanding of these methods---particularly their convergence properties when combining the score denoisers with various optimization paradigms---remains limited.

{\bf Contributions.} We propose to integrate score-based denoisers with the ADMM framework. Using ADMM iterates with score-based denoising is particularly challenging, as the presence of dual variables further distorts the ``noise’’ geometry---likely explaining why score-based denoising has rarely been combined with primal–dual methods. Nevertheless, ADMM remains attractive for its flexibility in handling diverse inverse problems with multiple regularizers. Our contributions are:

$\blacktriangleright$ \textit{Score-Based AC-DC Denoiser}: 
To mitigate the manifold mismatches, we propose a three-stage denoiser consisting of (1) additive Gaussian noise \textit{auto-correction} (AC), (2) conditional Langevin dynamics-based \textit{directional correction} (DC), and (3) score-based denoising. The AC stage pulls ADMM iterates toward neighborhoods of noise-trained manifolds, while DC refines alignment without losing signal information. This combination balances efficiency and accuracy, making score-based denoising effective within ADMM.

$\blacktriangleright$ \textit{Convergence Analysis}: We show that, under proper AC-DC parameters, each ADMM iteration is weakly nonexpansive, ensuring convergence to a fixed-point neighborhood under constant step sizes under strongly convex losses. We further relax convexity and prove that an adaptive step-size scheme \citep{chan2016plugandplayadmmimagerestoration} guarantees convergence with high probability. These results extend prior ADMM-PnP convergence theory \citep{ryu2019plugandplaymethodsprovablyconverge,chan2016plugandplayadmmimagerestoration} to score-based settings.

Our method is validated on diverse applications---including inpainting, phase retrieval, Gaussian and motion deblurring, super-resolution, and high dynamic range (HDR).

{\bf Notation.} The detailed notation designation is listed in Appendix~\ref{appendix: notations}.

\section{Background}
\noindent
{\bf Inverse Problems.} We consider the typical inverse problem setting where
  \begin{equation} \label{eqn: meas model}
      \by = \setA(\bx) +  \bxi
  \end{equation}
 where $\setA: \Rbb^d \to \Rbb^n$ is the \emph{measurement operator}, $n\leq d$, and $\bxi$ is additive noise. In some inverse problems, e.g., signal denoising and image deblurring, we have $n=d$; while for some other problems, e.g., data compression and recovery, we have $n<d$. The goal is to recover $\bx$ from the $\by$, with the knowledge of $\setA$. 
 Structural regularization on $\bx$ is often used to underpin the desired solution:
 \begin{align}\label{eq:inverseproblem}
     \min_{\bx}~\ell( \by~||~\setA(\bx)  ) +h(\bx),
 \end{align}
 where $\ell( \by~||~\setA(\bx)  )$ is a divergence term that measures the similarity of $\by$ and $\setA(\bx)$ (e.g., $\|\by-\setA(\bx)\|^2$), and $h(\bx)$ is a structural regularization term (e.g., $\|\bx\|_1$ for sparse $\bx$).

\noindent
{\bf Diffusion-Based Inverse Problem Solving.} Diffusion models can also be used for solving inverse problems, in ways more subtle than direct regularization. Consider training a diffusion model on $\bx_0 \sim p_{\rm data}$ via \emph{denoising score matching} \citep{song2021scorebasedgenerativemodelingstochastic}, where the forward process is
$
\bx_t|\bx_0 \sim \setN(\bx_0, \sigma(t)\bI), \quad t \in [0,T],
$
with variance schedule $\sigma(0)=0$ and $\sigma(t)$ increasing in $t$. After training, the model provides a score function
$
\bs_{\bm \theta}(\bx,\sigma(t)) \approx \nabla_{\bx_t}\log p(\bx_t),
$
where $p(\bx_t)$ is the marginal density of $\bx_t$. This induces noisy data manifolds
\[
\setM_{\sigma(t)} = {\rm supp}(\bx_t), \quad \forall t\in[T],
\]
which are continuous since $\bx_t$ is generated by Gaussian perturbations of $\bx_0 \sim p_{\rm data}$. These score functions can then be leveraged in different ways to assist inverse problem solving.

$\blacktriangleright$ {\it Posterior Sampling}: Many works formulate inverse problems as posterior sampling from $p(\bx|\by) \propto p(\bx)p(\by|\bx)$. These methods approximate $\nabla_{\bx_t}\log p(\by|\bx_t)$ and combine it with the learned score $\bs_{\bm \theta}(\bx_t,\sigma(t)) \approx \nabla \log p(\bx_t)$ to perform stochastic sampling as $t$ decreases \citep{chung2023diffusion,song2022pseudoinverse,kawar2021snipssolvingnoisyinverse,kawar2022denoisingdiffusionrestorationmodels,wang_2022_zeroshot}. While effective, their performance is often limited by the accuracy of the approximation to $\nabla \log p(\by|\bx_t)$.

$\blacktriangleright$ {\it Plug-and-Play (PnP) Approaches}: 
Instead of sampling schemes, another line of work employs deterministic or stochastic optimization to solve \eqref{eq:inverseproblem}, plugging score functions into the updates as structural regularizers. A representative example is DiffPIR \citep{zhu2023denoising}, which adopts a variable-splitting reformulation,
$
\min_{\bx,\bz} \; \ell(\by\|\setA(\bx)) + \tfrac{\mu}{2}\|\bx-\bz\|_2^2 + h(\bz),
$
where the $\bz$-subproblem at iteration $k$ reduces to a standard denoising step:
   \begin{align}
       \textbf{Denoising: } \bzkk &= \arg \min_{\bz} \frac{\mu}{2}\norm{\bxkk-\bz}_2^2 + h(\bz). \label{eq:denoising-diffpir}
   \end{align}
This step is then tackled using a score-based denoiser:
\begin{align}
    \bzkk  \leftarrow D_{\sigma^{(k)}}(\tbxkk) =   \tbxkk + (\sigmak)^2 \bs_{\bm \theta}(\tbxkk, \sigmak)
\end{align}
where $\tbxkk = \bxkk + \sigmak (\zeta \nicefrac{(\tbxk -\bxkk)}{\sigmak}  + (1-\zeta) \bn )$ with $\bn \sim \setN(\bm 0, \bI)$ and $\zeta\in [0,1]$; also see \citep{li2024decoupled} for a similar method. 
These denoisers are designed following the \textit{Tweedie's lemma} \citep{robbins1992anempirical} (see Appendix~\ref{appendix: preliminaries}). The construction of $\widetilde{\bx}^{(k)}$ is meant to make the inputs to the score function closer to a certain ${\cal M}_{\sigma^{(k)}}$. Using score-based denoising, $h(\bm z)$ is implicitly reflected in the denoising process and thus does not need to be specified analytically.

Another line of approaches explicitly construct regularizers $h(\bx)$ whose gradients correspond to applying the score function. Examples include RED-diff \citep{mardani2023avariational} and SNORE \citep{renaud2024plug}. In SNORE, the regularizer is defined as $h(\bx)= \mathbb{E}_{\widetilde{\bx}|\bx}\log p_{\sigma}(\widetilde{\bx})$, $ \widetilde{\bx} \leftarrow \bx + \sigma \bm \epsilon$, $\bm \epsilon \sim \setN(\bm 0, \bI) $ and taking its gradient in \eqref{eq:inverseproblem} yields:
  \begin{align}
        \textbf{{SNORE} update: } \bxkk &\leftarrow \bxk - \delta \nabla \ell(\by|| {\cal A}(\bxk)) - \eta (\tbxk - D_{\sigma_i}(\tbxk)),\label{eq:snorestep}
\end{align}
where, again, $D_{\sigma}(\widetilde{\bx}) = \widetilde{\bx} + \sigma^2 \bstheta(\widetilde{\bx}, \sigma)$ by the Tweedie's lemma.

{\bf Challenges — Manifold Mismatch and Convergence.} Score-based PnP methods face two main challenges. First, the score is not trained on algorithm-induced iterates (e.g., $\bx^{(k+1)}$ in \eqref{eq:denoising-diffpir}). While both $\bx_t$ and $\bx^{(k)}$ can be seen as noisy versions of $\bx \sim p_{\rm data}$, $\bx_t$ follows Gaussian noise whereas the distribution of $\bx^{(k)}$ is unclear. Many works attempt to bridge this gap by injecting Gaussian noise before applying the score function (cf. $\widetilde{\bx}^{(k+1)}$ in \eqref{eq:denoising-diffpir}, $\widetilde{\bx}^{(k)}$ in \eqref{eq:snorestep}), or by purification-based schemes \citep{nie2022diffusionmodelsadversarialpurification,alkhouri2023robustphysicsbaseddeepmri,meng2022sdeditguidedimagesynthesis}. Yet noise injection alone is insufficient, and overfitting to measurement noise remains an issue \citep{wang2024dmplug}.
Second, the understanding to convergence of score-based PnP remains limited. Unlike classical denoisers with established theory, the geometry mismatch above makes it unclear whether iterates stabilize or under what conditions convergence can be ensured. Existing analyses mostly cover primal algorithms (see, e.g., \citep{renaud2024plug}). Primal–dual methods such as ADMM offer greater flexibility for handling multiple regularizers and constraints, but their convergence with score-based denoisers remains unclear, as dual variables further complicate the manifold geometry of the iterates.

\section{Proposed Approach}

In this section, we propose our score-based denoiser and embed it in the ADMM framework. While we focus on ADMM due to its flexibility, the denoiser can be plugged into any other proximal operator based schemes (e.g. proximal gradient or variable-splitting as in DiffPIR).\\ 
{\bf Preliminaries of ADMM-PnP.} 
ADMM-based inverse problem solvers start by rewriting \eqref{eq:inverseproblem} as
  \begin{equation}\label{eq:slackvariablegeneralproblem}
    \min_{\bx,\bz} \ell( \by||{\cal A}(\bx)) + \gamma h(\bz) ~ \text{ s.t. }~ \bx = \bz
  \end{equation}
The augmented Lagrangian of \eqref{eq:slackvariablegeneralproblem} is given by $L_{\rho}(\bx, \by, \blam; \by) =  \ell( \by||{\cal A}(\bx)) + \gamma h(\bz) + \blam^T(\bx - \bz) + \frac{\rho}{2}\norm{\bx - \bz}_2^2$,
where $\blam$ is the dual variable and $\rho>0$ is the penalty parameter. At iteration $k$, ADMM updates are as follows:
      \begin{subequations}\label{eq:admm}
          \begin{align}
                  \bxkk &= \arg\min_{\bx} \frac{1}{\rho} \ell( \by||{\cal A}(\bx)) + \frac{1}{2} \norm{\bx - \bzk  + \buk}_2^2 \label{eq:mlesubgeneric}\\
                  \bzkk &= \arg\min_{\bz} \frac{\gamma}{\rho} h(\bz) + \frac{1}{2}\norm{\bxkk - \bz + \buk}_2^2\label{eq:denoisesubgeneric} \\
                  \bukk &= \buk + (\bxkk - \bzkk)\label{eq:dualsubgeneric}
          \end{align}
      \end{subequations}
      where $\buk=\blamk/\rho$ is the scaled dual variable. The subproblem \eqref{eq:denoisesubgeneric} is a denoising problem, and thus can be replaced by
  \begin{equation}\label{eq:denoise-step}
      \bzkk = \text{Prox}_{\frac{\gamma}{\rho}h} (\bxkk + \buk) = D_{\sigmak}(\tbzk).
  \end{equation}
  The above is the classical ADMM-PnP method; see \citep{chan2016plugandplayadmmimagerestoration,ryu2019plugandplaymethodsprovablyconverge}. 
  Same as before, Eq.~\eqref{eq:denoise-step} can be replaced by the score-based denoising following the Tweedie's lemma. However, as $\tbzk=\bxkk + \buk$ could be in any of the manifolds ${\cal M}_{\sigma(t)}$ on which the score function was trained, such naive replacement does not ensure effective denoising. The existence of the dual variable $\bm u^{(k)}$ makes the noise distribution in $\tbzk$ even harder to understand.

{\bf Proposed Approach: The AC-DC Denoiser.}
To address the manifold mismatch issues, we propose a three-stage denoiser. To be specific, in the $k$th iteration of the ADMM algorithm, we use the denoising process shown in Algorithm~\ref{algo:acdc}.
Note that the Tweedie's lemma step (line 8) can also be substituted by a score ODE based process \citep{karras2022elucidatingdesignspacediffusionbased} initialized at $\bzdck$. Our algorithm using these two different denoisers will be referred to as \texttt{Ours-tweedie} and \texttt{Ours-ode}, respectively.

 \begin{figure}[t!]
     \centering
     \includegraphics[width=.65\linewidth]{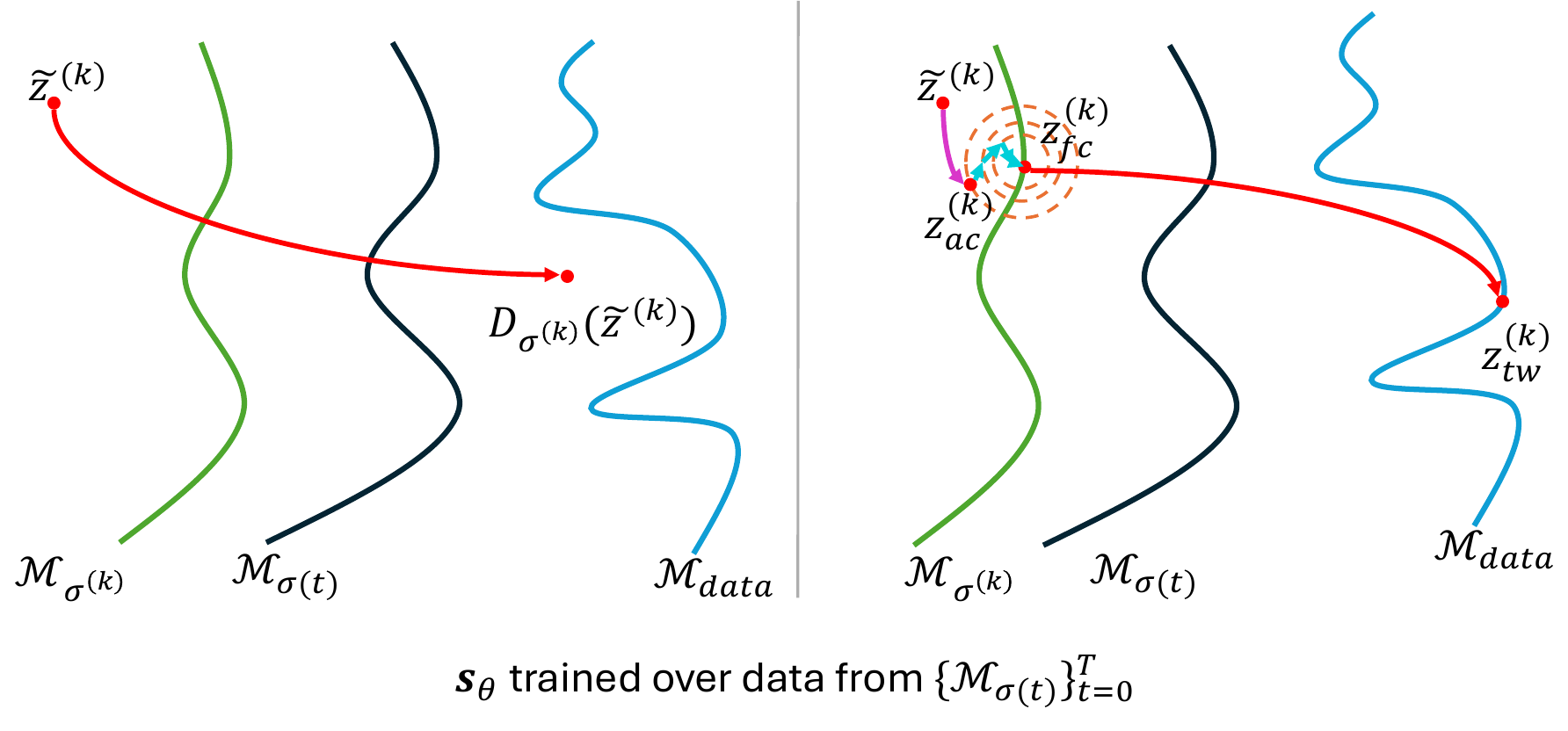}
     \caption{
     Left: direct denoising of $\tbzk$ using score functions could lead to unnatural recovered signals with artifacts. Right: AC-DC denoising brings $\tbzk$ closer to $\setM_{\sigmak}$, and then uses the score function to bring $\tbzk$ to the data manifold $\setM_{\rm data}$. }
     \label{fig:illustration 3-step denoiser}
 \end{figure}

\begin{algorithm}[!t]
\caption{AC-DC Denoiser at iteration $k$ of ADMM in \eqref{eq:admm} }\label{algo:acdc}
\footnotesize
\begin{algorithmic}[1]
    \STATE \uline{\bf {auto correction (AC):}}
    \(
      \quad    \bm z^{(k)}_{\rm ac} \leftarrow \tbzk + \sigmak \bn,~\bn\sim {\cal N}(\bm 0,\bm I)
    \)
    \vspace{1mm}
    \STATE \uline{\bf{directional correction (DC):}}
    \STATE 
    \(
      \bw^{(k, 0)} \;\leftarrow\;\bzack 
    \) 
    \FOR{$j=0$ to $J-1$}
      \STATE \(
         { {\bw}^{(k, j+1)}} \leftarrow \bw^{( { k,} j)}+ \etak \left(\nicefrac{1}{\sigma_{\bsk}^2} (\bz^{(k)}_{\rm ac} - \bw^{({ k,} j)}) + \bstheta(\bw^{({ k, }j)}, \sigmak)\right) + \sqrt{  2 \etak} \bn,
      \)
    \ENDFOR
   \STATE \(
      \bzdck \;\leftarrow\;\bw^{(k, J)}
    \)
    \vspace{1mm}
    \STATE \uline{\bf{Denoising}} \( : \quad     \bm z^{(k)}_{\rm tw} \leftarrow \Exp[\bz_0| \bz_t = \bzdck  ] = \bzdck + (\sigmak)^2 \bstheta(\bzdck, \sigmak)\) \hfill~(Tweedie Denoising) \\ {  (Alternative:~~$\bzodek \leftarrow z_0$ by solving: 
\(
\frac{d\bz_t}{dt}
 = \lambda(t) s_\theta(\bz_t,t),\) 
  with \(\bz_{\sigma^{(k)}} = \bz_{\mathrm{dc}}^{(k)}\)) \hfill~(ODE Denoising)
}

\end{algorithmic}
\end{algorithm}

The rationale of the AC-DC denoiser is illustrated in Fig.\ref{fig:illustration 3-step denoiser}. Recall that the score function is most effective on the noisy data manifolds $\{ {\cal M}_{\sigma(t)}\}_{t=1}^T$, as it is trained over them. Since ADMM-induced iterates $\widetilde{\bm z}^{(k)}$ need not lie on these manifolds, directly applying score-based denoising may be ineffective. The AC step addresses this by adding Gaussian noise, making $\bm z^{(k)}_{\rm ac}$ closer to some ${\cal M}_{\sigma(t)}$ (see Appendix\ref{sec: discussion on influence of AC-step}). This idea is related to the ``purification’’ step in \citep{nie2022diffusionmodelsadversarialpurification,alkhouri2023robustphysicsbaseddeepmri,li2024decoupled} and noise-added denoising in \citep{mardani2023avariational,renaud2024plug,zhu2023denoising} (cf. \eqref{eq:snorestep} and \eqref{eq:denoising-diffpir}). However, AC alone does not guarantee manifold alignment. The proposed DC step, based on Langevin dynamics, further refines $\bm z^{(k)}_{\rm dc}$ toward ${\cal M}_{\sigma^{(k)}}$.

To see the idea, let us break down the three steps.
First, the AC step gives
\begin{equation}
    \bzack=\bz_{\sigmak} + { \tbsk}, \quad \bsk=\tbzk - \bznk, \; \bznk\sim p_{\rm data}
\end{equation}
where $\bznk$ is denoised signal of $\tbzk$, $\bz_{\sigmak}=\bznk + \sigmak \bn_1 $, $ \tbsk = \sqrt{2} \sigmak \bn_2 + \bsk$, and $ \bn_1,\bn_2 \overset{\text{i.i.d}}{\sim} \setN(\bm 0, \bI_d)$. 
Given a sufficiently large $\sigma^{(k)}$, $\bz_{\rm ac}^{(k)}$ would have dominated by Gaussian noise---but not necessarily on any of ${\cal M}_{\sigma_t}$ where the score was trained.
Starting from $\bzack$, the DC step runs a few iterations of \emph{Langevin dynamics} targeting the distribution $p({\bz_{\sigmak}}  | \bzack)$. This is because ${\rm supp}({\bz_{\sigmak}}  | \bzack) \subseteq {\rm supp}(\bz_{\sigmak}) = {\cal M}_{\sigma^{(k)}}$. In addition, $p({\bz_{\sigmak}}  | \bzack)$ at the same time retains the information of $\bzack$ (thereby the information from the measurements). Assume that the forward process used in training the score has sufficiently small time intervals, ${\cal M}_{\sigma^{(k)}}$ is approximately contained in $\{ {\cal M}_{\sigma_t} \}_{t=1}^T$. This way, when applying Tweedie's lemma for denoising, the step is expected to be effective, as the score was trained over $\{ {\cal M}_{\sigma_t} \}_{t=1}^T$.

Note that the conditional score for the Langevin dynamics step can be expressed as follows
\begin{equation}
 {   \nabla \log p(\bz_{\sigmak} | \bzack) = \bs_{\bm \theta}(\bz_{\sigmak}, \sigmak ) + \nabla \log p(\bzack| \bz_{\sigmak}).}
\end{equation}
{
Ideally, one would use the exact $\nabla \log p(\bzack| \bz_{\sigmak})$ for the DC step---which is unavailable.
In practice,
we approximate $p(\bzack | \bz_{\sigmak})$ using a Gaussian distribution. Note that under proper scheduling of $\sigmak$ and mild regularity conditions on $\bsk$, e.g., when $\operatorname{Var}(\bsk)^{\nicefrac{1}{2}} \ll \sigmak$, the likelihood can be well-approximated by a locally quadratic form, leading to $ \nabla \log p(\bzack| \bz_{\sigmak}) \approx -\nicefrac{1}{\sigma_{\bsk}^2} (\bz_{\sigmak} - \bzack)$ and} the DC step in Algorithm~\ref{algo:acdc}.

\section{Convergence Analysis}
\subsection{Convergence of Under Weakly Non-Expansive Residuals }

Following the established convention in ADMM-PnP, e.g., \citep{buzzard2018plugandplay,sun2019anonlineplugandplay,chan2019performance,teodoro2017sceneadapted}, we aim at understanding the convergence properties when the AC-DC denoiser is used.
We will use the following definitions:
\begin{definition}[Fixed point convergence]
  Let $T:\mathcal X\to\mathcal X$ be the update map of an iterative algorithm, and let $\bx^{(0)} \in\mathcal X$ be arbitrary initialization. The algorithm is said to converge to a fixed point $\bx^*$ if for any $\delta>0$ there exists  $K_{\delta}>0$ such that the sequence generated by the algorithm $\{ \bxk \}_{k\in \Nbbp}$ satisfies $\norm{\bxk - \bx^*}_2 < \delta $ for all $k \ge K_{\delta}$. Equivalently, $\lim_{k\to\infty} \bxk=\bx^*$ with  $T(\bx^\ast)=\bx^\ast$.
\end{definition}

\begin{definition}[Sequence convergence to a $\delta$-ball]\label{def:ballconverge}
For a certain $\delta>0$, a sequence $\{\bxk\}_{k\in \Nbbp}$ is said to converge within a $\delta$-ball if there exists $K > 0$ and $\bx^*$ such that the following holds for all $k \ge K$ .
\begin{equation}
    \norm{\bxk - \bx^*}_2 \le \delta
\end{equation}
\end{definition}
Comparing the two definitions, Definition~\ref{def:ballconverge} is a weaker statement; that is, even when $k\rightarrow \infty$, $\bx^{(k)}\rightarrow \bx^\ast$ does not necessarily happen. Nonetheless, convergence to a $\delta$-ball is still meaningful. The notion of $\delta$-ball convergence  is often used in numerical analysis for stability characterization; see, e.g., \citet{ren2021onthecomplexity,ren2009convergenceball,liang2007homocentric}.

\begin{definition}[ADMM convergence to a $\delta$-ball]
ADMM is said to converge within a $\delta$ ball if the sequences $\{ \bx^{(k)} \}_{k \in \Nbbp}$ and $\{ \bu^{(k)} \}_{k \in \Nbbp}$  obtained from ADMM converges within a $\delta$-ball. 
\end{definition}

To proceed, consider the following assumption:

\begin{assumption} \label{assumpt: delta weakly residual}
    For a certain $\delta>0$ there exists $\epsilon \le 1$ such that for all $\tbz_1, \tbz_2 \in \Rbb^d$, the following holds:
    \begin{equation}
    \norm{R_{\sigma}(\tbz_1) - R_{\sigma}(\tbz_2)}_2^2 \le \epsilon^2 \norm{\tbz_1 - \tbz_2}_2^2 + \delta^2
    \end{equation}
    where $R_{\sigma}(\tbz)=(D_{\sigmak}-I)(\tbz)$ with $I$ being the identity function (i.e., $I(\bz)=\bz$). 
\end{assumption}
Here, the notation $(D_{\sigmak}-I)(\tbz)$ denotes the residual of $D_{\sigmak}$ i.e. $D_{\sigmak}(\tbz)-\tbz$. 
The next theorem extends the fixed point convergence of ADMM-PnP in \citet{ryu2019plugandplaymethodsprovablyconverge}. Unlike \citet{ryu2019plugandplaymethodsprovablyconverge} where $R_{\sigma}$ needs to be strictly contractive, our result allows $R_{\sigma}$ to be weakly contractive:

\begin{theorem} \label{thm:convergenceofpnpadmm}
      Under Assumption \ref{assumpt: delta weakly residual}, assume that $\ell$ is $\mu$-strongly convex. Then, there exists $\bx^*$, $\bu^*$ and $K>0$ such that the sequences $\{ \bxk \}_{k\in \Nbbp}$ and $\{ \buk \}_{k \in \Nbbp}$ generated by ADMM-PnP using a fixed step size $\rho$ satisfies $\norm{\buk - \bu^*}_2 \le r$ and $\norm{\bxk - \bx^*}_2 \le r $ with $r=(1+\frac{\rho}{\rho+\mu})\bar{\delta}/\sqrt{1-\bar{\epsilon}^2}$ for all $k \ge K$  when 
     $ \nicefrac{\epsilon}{\mu(1+\epsilon - 2\epsilon^2)} < \nicefrac{1}{\rho} $ 
      where \( \bar{\delta}^2 = \frac{\delta^2 \bar{\epsilon}}{\epsilon}\) and \( \bar{\epsilon}=\frac{\rho + \rho\epsilon + \mu\epsilon + 2\mu\epsilon^2}{\rho + \mu + 2\mu\epsilon}\). 
\end{theorem}
The proof is relegated to Appendix  \ref{proof: theorem 1 on general weakly non-expansive}. Note that when $\delta=0$, it implies the result in \citet{ryu2019plugandplaymethodsprovablyconverge}.

\subsection{Convergence Under Weakly Non-Expansiveness with AC-DC}
In this subsection, we will show that the AC-DC denoiser satisfies Assumption~\ref{assumpt: delta weakly residual} under mild conditions. To this end, consider the following:
  \begin{assumption}[Smoothness of $\log \pdata$] \label{assumpt: smoothness of logpdata}
      The log data density $\log \pdata$ is $M$-smooth for a constant $M>0$, i.e., $
          \norm{\nabla \log \pdata (\bx) - \nabla \log \pdata(\by)}_2 \le M \norm{\bx -\by}_2$
      for all $\bx, \by \in \setX$.
  \end{assumption}

  \begin{assumption} [Coercivity for $-\log \pdata$] \label{assumpt: coercivity of logpdata}
  There exists constants $c_1>0$ and $c_2 \ge 0$ such that
  \begin{equation}
      \norm{\nabla \log \pdata (\bx)}_2^2 \ge - c_1 \log \pdata - c_2, \; \norm{\bx}_2 \le - c_1 \log \pdata(\bx) + c_2, \; \forall \bx \in \setX
  \end{equation}
      
  \end{assumption}
  
  { This coercivity assumption means the negative log-density grows sufficiently fast at infinity, which prevents the Langevin dynamics from "escaping to infinity". This assumption guarantees stability and ensure ergodicity leading to convergence to the stationary distribution \citep{mattingly2002ergodicity}.}
  
  \begin{theorem}\label{theorem: contractive overall denoising for tweedie's lemma}
   Suppose that the assumptions in Theorem~\ref{thm:convergenceofpnpadmm}, Assumption~\ref{assumpt: smoothness of logpdata} and Assumption~\ref{assumpt: coercivity of logpdata} hold. Further, assume that the DC step reaches the stationary distribution for each $k$.
    Let $D_{\sigmak}:\tbzk \mapsto \bm z^{(k)}_{\rm tw}$ denote the AC-DC denoiser. Then, we have:
    
 (a)   With probability at least $1-2e^{-\nu_k}$, the following holds for iteration $k$ of ADMM-PnP:
                \begin{equation}\label{eq:acdccontraction}
                    \norm{(D_{\sigmak} -I)(\bx) - (D_{\sigmak} - I)(\by)}_2^2 \le \epsilon^2_{ k} \norm{\bx-\by}_2^2 + \delta^2_{ k}
                \end{equation}
                for any $\bx,\by \in \setX$  and $k \in \Nbbp$ when $\sigma_{\bsk}^2 + (\sigmak)^2 < 1/M$ with 
                \begin{align}
                \epsilon^2_{ k} &= 3(( \nicefrac{\sqrt{2}M  \sigma_{\bsk}^2}{1- \sigma_{\bsk}^2 M})^2 + (\sigmak)^4 M^2 ) \label{eqn: thm2 epsilonk def} \\ 
            \delta^2_{ k} &= 3 (2 (\sigmak )^2 (d + 2\sqrt{d\nu_k}+2\nu_k) + \nicefrac{32d \sigma_{\bsk}^2}{ (1-M\sigma_{\bsk}^2)}  \log \nicefrac{2}{\nu_k}). \label{eqn: thm2 deltak def}
                \end{align}
In other words, with $\nu_k=\ln \nicefrac{2\pi}{6\eta}+2 \ln k$, the denoiser $D_{\sigmak}$ satisfies part (a) for all $k \in \Nbbp$ with probability at least $1-\eta$. 
            
 (b)  
Assume that $\sigmak$ is scheduled such that 
$\lim_{k\rightarrow \infty } (\sigmak)^2 \nu_k =0$ for $\nu_k=\ln \nicefrac{2\pi}{6\eta}+2 \ln k$, $\epsilon<1$, and $ \nicefrac{\epsilon}{\mu(1+\epsilon - 2\epsilon^2)} < \nicefrac{1}{\rho}$ all hold, where $\epsilon=\lim_{k\to\infty} \sup \epsilon_k$ with $\epsilon_k$ defined in \eqref{eqn: thm2 epsilonk def}. 
Consequently, $\delta=\lim_{k\to\infty} \sup \delta_k$ is finite and ADMM-PnP with the AC-DC denoiser converges to an $r$-ball (see $r$ in Theorem~\ref{thm:convergenceofpnpadmm}) with probability at least $1-\eta$.
\end{theorem}

The proof is relegated to Appendix \ref{proof: theorem contractive overall denoising for tweedie's lemma}.
Theorem~\ref{theorem: contractive overall denoising for tweedie's lemma} (a) establishes that the AD-DC denoiser is weakly non-expansive with probability $1-2e^{-\nu_k}$ in iteration $k$.
The (b) part states that when $\sigma^{(k)}$ is carefully scheduled to approach zero as $k$ grows, then, with high probability, all the iterations satisfy the weakly non-expansiveness together---this leads to the convergence of the ADMM-PnP algorithm.

\subsection{Convergence without Convexity of $\ell$}
The weakly non-expansiveness based convergence analysis holds under fixed step size (i.e., $\rho$) of the ADMM-PnP algorithm, which is consistent with practical implementations in many cases. However, the assumption that the $\ell$ term is $\mu$-strongly convex is 
only met by some inverse problems, e.g., signal denoising and deblurring, but
not met by others such as signal compression/recovery and data completion. In this subsection, we remove the convexity assumption and analyze the AC-DC denoiser's properties under the adaptive $\rho$-scheme following that in \citep{chan2016plugandplayadmmimagerestoration}.

\begin{theorem} \label{thm:increasingrhoconvergence}
 Suppose that Assumptions~\ref{assumpt: smoothness of logpdata}-\ref{assumpt: coercivity of logpdata} hold. Let
    \(D:={\rm diam}(\setX)=\sup_{\bx,\by \in \setX}\norm{\bx - \by}_2<\infty, \; S:=\inf_{\bx \in \setX} \norm{ \nabla \log \pdata(\bx)}_2 < \infty\) and define $L:=MD+S$.
Let $D_{\sigmak}:\tbzk \mapsto \bm z^{(k)}_{\rm tw}$ denote the AC-DC denoiser.  Also, assume that the DC step reaches the stationary distribution for each $k$\footnote{ Note that Theorems~\ref{theorem: contractive overall denoising for tweedie's lemma} and \ref{thm:increasingrhoconvergence} use this stationary distribution assumption for notation conciseness. For their counterparts removing this assumption, see Appendix~\ref{section: theoretical results with finite steps}.}. Then, the following hold:

(a) (\textbf{Boundedness}) With probability at least $1-2e^{-\nuk}$, the denoiser $D_{\sigmak}$ is bounded at each iteration $k$ i.e. \(\frac{1}{d} \norm{(D_{\sigmak} -I)(\bx)}_2^2 \le c_k^2 \) whenever $\sigma_{\bsk}^2 + (\sigmak)^2 < 1/M,$ where $\ck= (\sigmak)^2 (2 +4\sqrt{\nuk} + 4 \nuk) + \nicefrac{16 \sigma_{\bsk}^2}{1-M \sigma_{\bsk}^2} \log \nicefrac{2}{\nuk} + 2\sigma_{\bsk}^4  L^2 + 2 (\sigmak)^4 L^2$, and $\nuk > 0$.\\
    Let $\nu_k = \ln \frac{2\pi^2}{6\eta}+2\ln k$ with $\eta\in (0,1]$. Consequently, the denoiser $D_{\sigmak}$ is bounded for all $k \in \Nbb_+ $ with corresponding $c_k$ and probability at least $1-\eta$.

 (b) (\textbf{Convergence}) Assume there exists $R < \infty$ such that $\norm{\nabla \ell(\bx)}_2/\sqrt{d}\le R$. Apply the $\rho$-increasing rule in \citep{chan2016plugandplayadmmimagerestoration} and schedule $(\sigmak, \sigma_{\bsk})$ such that
    $\lim_{k\to\infty} (\sigmak)^2 (2+4\sqrt{\nu_k} + 4 \nu_k)=0, \; \lim_{k\to\infty} \frac{\sigma_{\bsk}^2}{1-M\sigma_{\bsk}^2} \log \frac{2}{\nu_k}=0,$
    $\lim_{k\to\infty} \sigmak =0, \lim_{k\to\infty} \sigma_{\bsk}=0 \; \sigma_{\bsk}^2 + (\sigmak)^2 < 1/M, \; \forall k\in \Nbb_+$
    for $\nu_k=\ln  \frac{2\pi^2}{6\eta}+2\ln k$ with $\eta \in (0,1]$. Then, the solution sequence converges to a fixed point with probability at least $1-\eta$.
    
\end{theorem}
The proof is relegated to Appendix~\ref{proof:increasingrhoconvergence}. Theorem \ref{thm:increasingrhoconvergence} (a) shows that, with high probability the denoiser is bounded uniformly across all iterations $k$. Part (b) further shows that, under the proper scheduling of $(\sigmak, \sigma_{\bsk})$, the AD-DC ADMM-PnP algorithm converges to a fixed point with high probability. 

The condition of $D<\infty$ implies the data space $\setX$ has bounded support, which is natural in practice: for images, pixel intensities typically lie within a bounded range such as $[0,1]$. Additionally, the condition $S<\infty$ ensures that there exists at least one point in $\setX$ where the score norm is finite. { This prevents pathological cases where the score diverges everywhere (making the distribution degenerate). Together, these conditions guarantee that the score is ``well-behaved''.} 

{ A remark is that all theoretical results in this section focus on fixed-point convergence, which is not the strongest form of convergence guarantees. Establishing stronger convergence results, e.g., stationary-point convergence, for PnP approaches is considered challenging as the objective function is implicit (more specifically, $h(\cdot)$ in \eqref{eq:inverseproblem} is implicit). Nonetheless, in recent years, some efforts have been made towards establishing stationary-point convergence for PnP methods under certain types of denoisers (see, e.g., \citep{huralt2022gradientstep, huralt2022proximal, wei2025learning,xu2025radiomapestimationlatent}); more discussions are in Sec.~\ref{sec:relatedworks}.

}

\section{Related Works}\label{sec:relatedworks}
ADMM-PnP has gained much popularity, due to access to data-driven effective denoisers. It has been used in various applications like image restoration \citep{chan2016plugandplayadmmimagerestoration}, data compression \citep{yuan2022plugandplay}, hyperspectral imaging \citep{liu2022hyperspectral}, and medical imaging \citep{ahmad2020plugandplay}.

{
Theoretical understanding of PnP algorithms with general ``black-box'' denoisers remains limited. Unlike classical proximal operators, the implicit regularizer $h(\cdot)$ in \eqref{eq:inverseproblem} handled by data-driven denoisers is typically unknown. Hence, many results are therefore restricted to fixed-point convergence; see, e.g., \citep{ryu2019plugandplaymethodsprovablyconverge,chan2016plugandplayadmmimagerestoration}. Nonetheless, in certain cases where the denoisers have interesting structures, stronger convergence results can be established. For example, \citep{xu2025radiomapestimationlatent} used classical results from image denoising connecting linear denoisers with quadratic $h(\cdot)$ to show that when linear denoisers are employed, ADMM-PnP converges to KKT points. \citet{huralt2022gradientstep} showed stationary-point convergence of PnP methods gradient-type denoisers, leveraging the fact that this type of denoisers can be written as a proximal operator of a special function (see the nonconvex counterpart in \citep{huralt2022proximal}); \citet{wei2025learning} trained denoisers to satisfy a cocoercive conservativity condition, which also ensures convergence of PnP to stationary points associated with an implicit convex $h(\cdot)$. Nonetheless, these results do not cover diffusion-based denoisers. In this work, we generalize the fixed-point convergence proofs in \citep{chan2016plugandplayadmmimagerestoration,ryu2019plugandplaymethodsprovablyconverge} to accommodate the diffusion score-based AC--DC denoiser.}

Recent advances in score-based generative modeling have motivated their integration into PnP algorithms. One line of work directly replaces the proximal denoiser with a pre-trained scores  \citep{zhu2023denoising, li2024decoupled}. Alternatively, others embed the score function as an explicit regularizer with task-specific loss \citep{mardani2023avariational,renaud2024plug}. 
Deterministic version PnP have also been considered. For example, \citet{wang2024dmplug,song2023solving} use unrolled ODE and consistency model-distilled one-step representation to express the target signal, respectively. These methods are similar to \citep{bora2017compressed}, but with diffusion-driven parameterization. 

{ Prior works have emphasized the importance of matching the residual noise to the operating range of the PnP denoiser. D-AMP \citep{metzler2016Adenoising,eksioglu2018denoising} achieve this via the Onsager correction, which approximately Gaussianizes the residual under compressive sensing problem structures. \citet{wei2021tfpnp} learns a reinforcement learning-based policy to automatically tune all internal parameters, including denoising strength. { Unlike AC–DC that provides a generic correction mechanism for a variety inverse problems, these methods either problem specific or require training additional models.} Score-based inverse problem solvers have also attempted to ``bring'' iterates to noisy data manifolds used during training.} The work \citet{chung2022improving} uses a manifold constraint based on gradient of data-fidelity, while \citet{he2024manifold} uses an off-the-shelf pretrained neural network to impose a manifold constraint. On the other hand, \citet{zirvi2025diffusion} uses the projection of measurement guidance to low-rank subspace, using SVD on the intermediate diffusion state, for similar purposes.

The idea of adding noise before evaluating score functions during optimization procedures (similar to our AC step) has been widely considered \citep{li2024decoupled,graikos2022diffusion, renaud2024plugandplay, mardani2023avariational}. A variant of this called estimation-correction idea proposed in \citep{karras2022elucidatingdesignspacediffusionbased} is used in \citep{zhu2023denoising} for this purpose.

\section{Experiments}

\begin{figure}[!t]
    \centering
    \includegraphics[width=\linewidth]{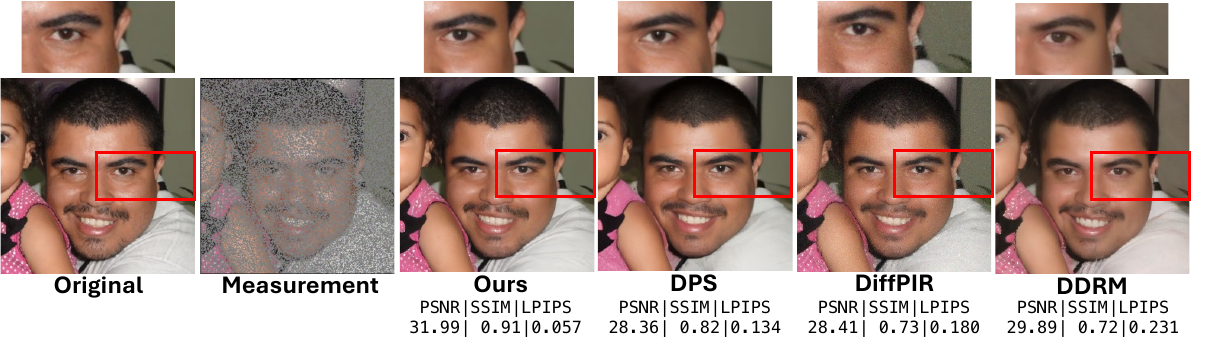}
    \caption{Inpainting under random missings.}
    \label{fig:performance rand inpainting}
\end{figure}
\begin{figure}[t]
    \centering
    \includegraphics[width=\linewidth]{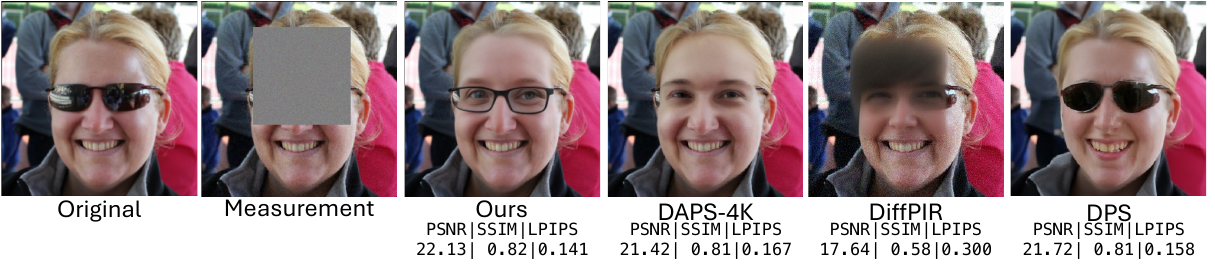}
    \caption{Inpainting under box missing. }
    \label{fig: performance comparison inpainting box}
\end{figure}

\textbf{Dataset and Evaluation Metrics.} For all these tasks, we use two datasets: FFHQ $256 \times 256$ \citep{karras2019stylebasedgeneratorarchitecturegenerative} and ImageNet $256 \times 256$ \citep{deng2009imagenet}. 
During testing, we randomly sample 100 images from the validation set of each dataset. All the methods use the pre-trained score model in \citet{chung2023diffusion}. We use \emph{Peak Signal-to-Noise Ratio } (PSNR) as a pixel-wise similarity metric, and \emph{Structural Similarity Index} (SSIM) and \emph{Learned Perceptual Image Patch Similarity} (LPIPS) \citep{zhang2018theunreasonable} as perceptual similarity metrics. 
We report these metrics averaged over the 100 test images for each method and inverse problem.

\textbf{Task Description.} We consider $\bxi \sim \setN(\bm 0, \sigma_n^2 \bI)$ with $\sigma_n=0.05$ for all the tasks. (a) For \textit{super-resolution}, we use cubic interpolation method with kernel size 4 for downsampling the resolution by $4$ times. (b) For \textit{recovery under Gaussian blurring} (Gaussian deblurring), a kernel of size 61 and standard deviation 3 is used. (c) As for \textit{recovery under motion blurring} (motion deblurring), a kernel of size 61 and standard deviation of 0.5 is used. (d) In \textit{inpainting under box mask} (box inpainting), an approximately centered mask of size $128 \times 128$ is sampled in image while maintaining the $32$ pixel margin in both spatial dimensions of the input image. (e) For \textit{inpainting under random missings} (random inpainting), $70\%$ of the pixels are uniformly sampled to be masked, and a scaling of $2$ was used in high dynamic range (HDR) before clipping the values. (f) For \textit{phase retrieval}, similar as in prior works \citep{wu2024principled,mardani2023avariational}, we use oversampling by factor of $2$. 
g) For deblurring under nonlinear blurring, we use the operator in \citep{tran2021explore} with default settings.

\textbf{Baselines.} We use a set of baselines, namely, DPS \citep{chung2023diffusion}, DAPS \citep{zhang2024improvingdiffusioninverseproblem}, DDRM \citep{kawar2022denoisingdiffusionrestorationmodels}, DiffPIR \citep{zhu2023denoising}, RED-diff \citep{mardani2023avariational},  DPIR \citep{zhang2022plugandplay}, DCDP \citep{li2025decoupleddataconsistencydiffusion}, PMC \citep{sun2024provableprobabilisticimagingusing}. 

{ 
\textbf{Hyperparameter Settings.}
 We adopt a linear schedule for $\sigmak$ with range $[0.1, 10]$ over $W$ decay window i.e. $\sigmak = \max(0.1, 10 - (10-0.1) \cdot k/W)$. The maximal number of iterations for our proposed method is set to $K=W+10$. At iteration $k$, we use $J=10$ DC steps, and the schedules $\etak=5\times10^{-4} \sigmak$ and $\sigma_{\bsk}= \nicefrac{0.1}{\sqrt{\sigmak}} $. We use gradient descent with Adam optimizer \citep{kingma2017adammethodstochasticoptimization} for solving each regularized maximum likelihood subproblem \eqref{eq:mlesubgeneric}. This subproblem is optimized for maximum of $1000$ iterations with convergence detected when the loss value increases more than $\Delta_{\rm tol}=1\times 10^{-1}$ consecutively for $3$ iterations window. We conduct our experiment with two variants based on the third stage: using Tweedie's lemma (denoted as ``Ours-tweedie'') and a 10-step ODE based denoiser \citep{zhang2024improvingdiffusioninverseproblem, karras2022elucidatingdesignspacediffusionbased} (denoted as ``Ours-ode''). We use the preconditioning in \citet{karras2022elucidatingdesignspacediffusionbased} while using the pretrained diffusion models.}

{\bf Qualitative Performance.} Figs~\ref{fig:performance rand inpainting},
 ~\ref{fig: performance comparison inpainting box} and ~\ref{fig: performance comparison motion blur} show reconstructions under inpainting under random missings, inpainting under box missing, and motion deblurring. It can be seen that our method is able to recover the image that is comparatively natural looking with less noise and artifacts, while being consistent with the measurements.  On the other hand, images recovered with DiffPIR appears to suffer from noise and artifacts, whereas DPS leads to measurement-inconsistent reconstructions.

Our method outperforms others while other methods appear to either be blurred or contain noisy artifacts in the recovered image. Recovery by DPS is less consistent with the original image; the pattern on the child's clothing is completely lost.

{\bf Quantitative Performance.}
Table~\ref{tab:reconstruction metrics} summarizes PSNR, SSIM and LPIPS averaged over 100 images on FFHQ and Imagenet datasets. In almost all of the inverse problems, both of our variants (Ours-tweedie and Ours-ode) achieve the best or second-best performance in terms of all metrics. Our method significantly outperforms other PnP baseline methods considered, namely, DDRM, DiffPIR and RED-diff. This demonstrates the effectiveness of our AC-DC denoiser.
\begin{figure}[t]
  \centering
  \begin{minipage}[t]{0.49\textwidth}
       \includegraphics[width=\linewidth]{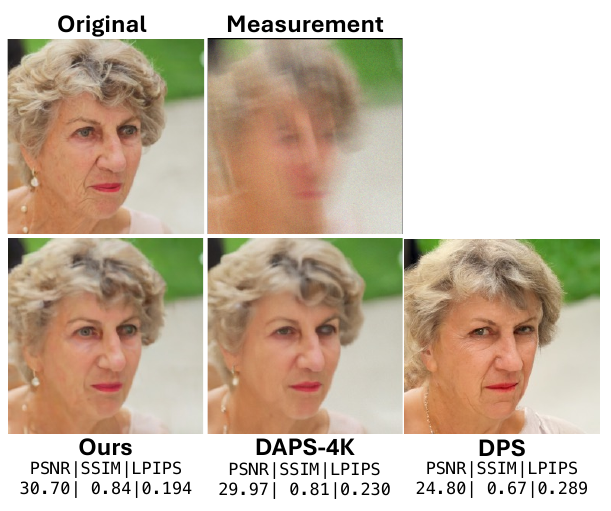}
        \caption{ Recovery under motion blurring.}
        \label{fig: performance comparison motion blur}
  \end{minipage}
  \hfill
  \begin{minipage}[t]{0.49\textwidth}
    \centering
    \includegraphics[width=\linewidth]{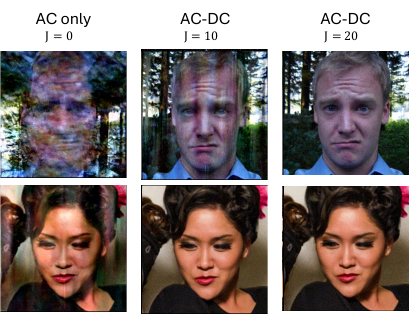}
    \caption{Influence of DC steps in the denoiser. }
    \label{fig: influence of DC in denoiser}
  \end{minipage}
\end{figure}

\begin{table}[t] 
\centering
\tiny
\setlength{\tabcolsep}{2pt}
\renewcommand{\arraystretch}{0.9}
\caption{Reconstruction metrics (100 images) on FFHQ / ImageNet. \textbf{Bold}: best, {\blue blue}: 2nd best.}
\label{tab:reconstruction metrics}
\begin{minipage}{.48\linewidth}
  \centering
  %===== first block: tasks 1–3 =====%
  \begin{tabular}{l l rrr rrr}
    \toprule
    & & \multicolumn{3}{c}{FFHQ} & \multicolumn{3}{c}{ImageNet} \\
    \cmidrule(lr){3-5}\cmidrule(lr){6-8}
    Task & Method & PSNR↑ & SSIM↑ & LPIPS↓ & PSNR↑ & SSIM↑ & LPIPS↓ \\
    \midrule
    \multirow{7}{*}{\rotatebox{90}{\makecell[c]{Superresolution\\ (4×)}}}
      & Ours–tweedie & \bfseries30.439 & \bfseries0.857 & 0.178
                    & \bfseries27.318 & \bfseries0.717 & 0.280 \\
      & Ours–ode     & \blue 29.991 & \blue 0.845 & \bfseries0.156 
                    & \blue 26.919 & \blue 0.700 & \blue 0.276 \\
      & DAPS     & 29.529 & 0.814 & \blue 0.167
                    & 26.653 & 0.680 & \bfseries0.266 \\
      & DPS          & 24.828 & 0.705 & 0.257
                    & 22.785 & 0.549 & 0.411 \\
      & DDRM         & 27.145 & 0.782 & 0.261
                    & 26.105 & 0.683 & 0.306 \\
      & DiffPIR      & 26.771 & 0.749 & 0.208
                    & 23.884 & 0.543 & 0.336 \\
    & RED-diff  & 16.833 & 0.422 & 0.547
                & 18.662 & 0.309 & 0.519 \\
    &  DPIR  &  28.849 &  0.826 &  0.254
                &  26.524 &  0.699 &  0.334 \\
    &  DCDP  &  27.761 &  0.639  &  0.332
            &  24.517  &  0.525  &  0.361  \\
    &  PMC &  23.774 &  0.421 &  0.407 &  22.534 &  0.334 &  0.456 \\
    \midrule
    \multirow{7}{*}{\rotatebox{90}{\makecell[c]{Inpainting\\ (Random)}}}
      & Ours–tweedie & \bfseries32.844 & \bfseries0.906 & \blue 0.122
                    & \bfseries29.564 & \bfseries0.817 & 0.184 \\
      & Ours–ode     & \blue 32.127 & \blue 0.894 & \bfseries0.095
                    & \blue 28.733 & \blue 0.795 & \bfseries0.148 \\
      & DAPS      & 31.652 & 0.847 & 0.124
                    & 28.137 & 0.751 & \blue 0.162 \\
      & DPS          & 29.084 & 0.828 & 0.181
                    & 26.049 & 0.678 & 0.318 \\
      & DDRM         & 28.969 & 0.847 & 0.178
                    & 27.883 & 0.778 & 0.203 \\
      & DiffPIR      & 28.558 & 0.709 & 0.230
                    & 26.923 & 0.639 & 0.222 \\
      & RED-diff & 20.361 & 0.630 & 0.275
                    & 20.948 & 0.464 & 0.315\\
      &  PMC &  23.289 &  0.755 &  0.263
        &  25.965 &  0.636 &  0.342\\
    \midrule
    \multirow{4}{*}{\rotatebox{90}{\makecell[c]{Motion \\ Deblur}}}
      & Ours–tweedie & \bfseries30.003 & \bfseries0.854 & 0.179
                    & \bfseries27.149 & \bfseries0.717 & 0.280 \\
      & Ours–ode     & \blue 29.648 & \blue 0.841 & \bfseries0.154
                    & \blue 26.615 & \blue 0.694 & \bfseries0.275 \\
      & DAPS     & 29.051 & 0.815 & \blue 0.175
                    & 26.571 & 0.689 & \blue 0.276 \\
      & DPS          & 23.257 & 0.663 & 0.265
                    & 19.613 & 0.451 & 0.451 \\
    &  PMC &  19.480 &  0.590 &  0.426 &  21.608 &  0.480 &  0.510 \\
    \bottomrule
  \end{tabular}
\end{minipage}
\hfill
\begin{minipage}{.48\linewidth}
  \centering
  %===== second block: tasks 4–6 =====%
  \begin{tabular}{ll rrr rrr}
    \toprule
     & &\multicolumn{3}{c}{FFHQ} & \multicolumn{3}{c}{ImageNet} \\
    \cmidrule(lr){3-5}\cmidrule(lr){6-8}
    Task& Method & PSNR↑ & SSIM↑ & LPIPS↓ & PSNR↑ & SSIM↑ & LPIPS↓ \\
    \midrule
    \multirow{6}{*}{\rotatebox{90}{\makecell[c]{Gaussian \\ Blur}}}
     &Ours-tweedie & \bfseries30.402 & \bfseries0.853 & 0.175
                    & \bfseries27.199 & \bfseries0.705 & \blue 0.281 \\
       & Ours-ode  & \blue 30.019 & \blue 0.841 & \blue 0.158
                    & \blue 26.899 & 0.690 & 0.282 \\
       &DAPS & 29.790 & 0.813 & \bfseries0.157
                    & 26.886 & 0.678 & \bfseries0.260 \\
        & DPS     & 26.106 & 0.730 & 0.207
                    & 23.995 & 0.575 & 0.328 \\
         & DiffPIR & 25.148 & 0.699 & 0.230
                    & 22.756 & 0.508 & 0.374 \\
        &  DPIR  &  28.875 &  0.833 &  0.228
        &  26.702 &  \blue  0.700 &  0.314 \\
        &  DCDP  &  16.821  &  0.171  &  0.642 
            &  15.102  &  0.136  &  0.620  \\
        &  PMC &  20.172 &  0.638  &  0.344  &  24.103 &  0.545 &  0.415 \\
    \midrule

        \multirow{6}{*}{\rotatebox{90}{\makecell[c]{Inpainting \\ (Box)}}}
       & Ours-tweedie&  \blue  24.025 & \bfseries0.859 & \bfseries0.131
                    & \bfseries21.626 & \bfseries0.789 &  0.222 \\
         &Ours-ode& 23.342 & \blue 0.837 & \blue 0.136
                    & 20.618 & 0.743 & 0.227 \\
         & DAPS&  23.643 & 0.815 & 0.146
                    & \blue 21.303 & \blue 0.774 &  \blue  0.199 \\
          & DPS  & 23.488 & 0.817 & 0.164
                    & 19.933 & 0.677 & 0.309 \\
         & DiffPIR  & 20.934 & 0.561 & 0.294
                    & 19.565 & 0.562 & 0.342 \\
        & RED-diff & 18.713 & 0.523 & 0.364
            & 18.075 & 0.499 & 0.371 \\
        &  DCDP  &  \bfseries 25.230  &  0.754  &  0.163
            &  20.991  &  0.727 &  \bfseries 0.195  \\
        &  PMC &  14.828 &  0.697 &  0.318 &  15.550 &  0.666 &  0.326 \\
    \midrule
    \multirow{5}{*}{\rotatebox{90}{\makecell[c]{Phase \\ Retrieval}}}
     & Ours-tweedie&  \bfseries 27.944 &  \bfseries 0.793 &  \bfseries 0.209
                    &  \bfseries17.770 &  \bfseries0.440 & \bfseries0.471 \\
    & Ours-ode &  \blue  27.095&  \blue  0.757 &  0.237
                    &  16.013 &  0.339 &  0.539 \\ 
        & DAPS &  26.707 &  0.749 &  \blue  0.230
                    & \blue 16.444 & \blue 0.395 & \blue 0.512 \\
   & DPS & 11.627 & 0.366 & 0.658
                    &  9.434 & 0.216 & 0.768 \\
    & RED-diff & 15.411 & 0.490 & 0.480 
                 & 12.852 & 0.204 & 0.695\\
     &  DCDP  &  20.026  &  0.540  &  0.424
            &  12.257  &  0.212  &  0.665  \\
    &  PMC &  10.421 &  0.287 &  0.783 &  8.636 &  0.129 &  0.890 \\
    \bottomrule
  \end{tabular}
\end{minipage}
\end{table}

\textbf{Effectiveness of DC.} To perform ablation study on the DC stage, we consider the challenging phase retrieval problem. Fig.~\ref{fig: influence of DC in denoiser} shows the output of ADMM-PnP with our AC-DC denoiser with different numbers of DC iterations $J$. With $J=0$ (disabling DC step), artifacts remain severe. Increasing $J$ progressively results in cleaner images. 

{\bf More Details and Additional Experiments.} More { details and} experiments are in appendices.

\section{Conclusion}

We introduced the AC-DC denoiser, a score-based denoiser designed for integration within the ADMM-PnP framework. The denoiser adopts a three-stage structure aimed at mitigating the mismatch between ADMM iterates and the noisy manifolds on which score functions are trained. We established convergence guarantees for ADMM-PnP with the AC-DC denoiser under both fixed and adaptive step size schedules. Empirical results across a range of inverse problems demonstrate that the proposed method consistently improves solution quality over existing baselines.

\textbf{Limitations.}  While our analysis provides initial insights, several aspects merit deeper understanding. The second convergence result relaxes convexity by allowing adaptive step sizes, though such schedules are arguably less appealing in practice. Our experiments, however, suggest that constant step sizes also perform well for nonconvex objectives; it is therefore desirable to establish convergence guarantees for constant step sizes in such settings. { In addition, our result ensures the {\it stability} of the ADMM method, but does not directly explain the reason \emph{why} the AC–DC denoiser attains high-quality recovery; recoverability and estimation error analyses are also desirable.} On the implementation side, the noise schedules used in the AC and DC stages are currently guided by empirical heuristics. Designing problem-adaptive scheduling strategies may further improve both convergence speed and robustness. { Additionally, each iteration of AC–DC denoiser needs multiple score evaluations. Reducing the required NFEs could significantly improve its efficiency.

}

\clearpage

{\bf Ethics Statement}: This work focuses exclusively on the theory and methodology of solving inverse problems. It does not involve human subjects, personal data, or any sensitive procedures.

{\bf Reproducibility Statement}:  The source code is provided as a part of the supplementary material. All assumptions, derivations and necessary details regarding the theory and experiments are included in the appendices. 

{\bf Acknowledgment}: This work was supported in part by the National Science Foundation (NSF) under Project NSF CCF-2210004. It was also supported in part by the SciRIS seed funding from the College of Science and College of Engineering at Oregon State University.

\bibliography{iclr2026_conference}

@inproceedings{bora2017compressed,
  title={Compressed sensing using generative models},
  author={Bora, Ashish and Jalal, Ajil and Price, Eric and Dimakis, Alexandros G},
  booktitle={International conference on machine learning},
  pages={537--546},
  year={2017},
  organization={PMLR}
}

@article{chung2022improving,
  title={Improving diffusion models for inverse problems using manifold constraints},
  author={Chung, Hyungjin and Sim, Byeongsu and Ryu, Dohoon and Ye, Jong Chul},
  journal={Advances in Neural Information Processing Systems},
  volume={35},
  pages={25683--25696},
  year={2022}
}

@inproceedings{song2021scorebasedgenerativemodelingstochastic,
  author       = {Yang Song and
                  Jascha Sohl{-}Dickstein and
                  Diederik P. Kingma and
                  Abhishek Kumar and
                  Stefano Ermon and
                  Ben Poole},
  title        = {Score-Based Generative Modeling through Stochastic Differential Equations},
  booktitle    = {9th International Conference on Learning Representations, {ICLR} 2021,
                  Virtual Event, Austria, May 3-7, 2021},
  publisher    = {OpenReview.net},
  year         = {2021},
  url          = {https://openreview.net/forum?id=PxTIG12RRHS},
  bibsource    = {dblp computer science bibliography, https://dblp.org}
}

@InProceedings{ryu2019plugandplaymethodsprovablyconverge,
  title = 	 {Plug-and-Play Methods Provably Converge with Properly Trained Denoisers},
  author =       {Ryu, Ernest and Liu, Jialin and Wang, Sicheng and Chen, Xiaohan and Wang, Zhangyang and Yin, Wotao},
  booktitle = 	 {Proceedings of the 36th International Conference on Machine Learning},
  pages = 	 {5546--5557},
  year = 	 {2019},
  editor = 	 {Chaudhuri, Kamalika and Salakhutdinov, Ruslan},
  volume = 	 {97},
  series = 	 {Proceedings of Machine Learning Research},
  month = 	 {09--15 Jun},
  publisher =    {PMLR},
  url = 	 {https://proceedings.mlr.press/v97/ryu19a.html}
}

@article{chan2016plugandplayadmmimagerestoration,
author = {Chan, Stanley and Wang, Xiran and Elgendy, Omar},
year = {2016},
month = {05},
pages = {},
title = {Plug-and-Play ADMM for Image Restoration: Fixed Point Convergence and Applications},
volume = {PP},
journal = {IEEE Transactions on Computational Imaging},
doi = {10.1109/TCI.2016.2629286}
}

@article{combettes2014compositionsconvexcombinationsaveraged,
title = {Compositions and convex combinations of averaged nonexpansive operators},
journal = {Journal of Mathematical Analysis and Applications},
volume = {425},
number = {1},
pages = {55-70},
year = {2015},
issn = {0022-247X},
doi = {https://doi.org/10.1016/j.jmaa.2014.11.044},
url = {https://www.sciencedirect.com/science/article/pii/S0022247X14010865},
author = {Patrick L. Combettes and Isao Yamada}
}

@book{bauschke2017correction,
  title={Convex analysis and monotone operator theory in Hilbert spaces},
  author={Bauschke, Heinz H and Combettes, Patrick L and Bauschke, Heinz H and Combettes, Patrick L},
  year={2017},
  publisher={Springer}
}

@article{giselsson2017tightgloballinearconvergence,
  title={Tight global linear convergence rate bounds for Douglas–Rachford splitting},
  author={Pontus Giselsson},
  journal={Journal of Fixed Point Theory and Applications},
  year={2015},
  volume={19},
  pages={2241 - 2270},
  url={https://api.semanticscholar.org/CorpusID:56162917}
}

@article{laurent2000adaptiveestimation, title={Adaptive estimation of a quadratic functional by model selection}, volume={28}, ISSN={0090-5364, 2168-8966}, DOI={10.1214/aos/1015957395}, abstractNote={We consider the problem of estimating $|s|^2$ when $s$ belongs to some separable Hilbert space and one observes the Gaussian process $Y(t) = langles, trangle + sigmaL(t)$, for all $t epsilon mathbb{H}$,where $L$ is some Gaussian isonormal process. This framework allows us in particular to consider the classical “Gaussian sequence model” for which $mathbb{H} = l_2(mathbb{N}*)$ and $L(t) = sum_{lambdageq1}t_{lambda}varepsilon_{lambda}$, where $(varepsilon_{lambda})_{lambdageq1}$ is a sequence of i.i.d. standard normal variables. Our approach consists in considering some at most countable families of finite-dimensional linear subspaces of $mathbb{H}$ (the models) and then using model selection via some conveniently penalized least squares criterion to build new estimators of $|s|^2$. We prove a general nonasymptotic risk bound which allows us to show that such penalized estimators are adaptive on a variety of collections of sets for the parameter $s$, depending on the family of models from which they are built.In particular, in the context of the Gaussian sequence model, a convenient choice of the family of models allows defining estimators which are adaptive over collections of hyperrectangles, ellipsoids, $l_p$-bodies or Besov bodies.We take special care to describe the conditions under which the penalized estimator is efficient when the level of noise $sigma$ tends to zero. Our construction is an alternative to the one by Efroïmovich and Low for hyperrectangles and provides new results otherwise.}, number={5}, journal={The Annals of Statistics}, publisher={Institute of Mathematical Statistics}, author={Laurent, B. and Massart, P.}, year={2000}, month=oct, pages={1302–1338} }

@book{wainwright2019high,
  title={High-dimensional statistics: A non-asymptotic viewpoint},
  author={Wainwright, Martin J},
  volume={48},
  year={2019},
  publisher={Cambridge university press}
}

@inproceedings{                   
 chung2023diffusion,                    
 title={Diffusion Posterior Sampling for General Noisy Inverse Problems},                    
 author={Hyungjin Chung and Jeongsol Kim and Michael Thompson Mccann and Marc Louis Klasky and Jong Chul Ye},                    
 booktitle={International Conference on Learning Representations},                    
 year={2023},                    
 url={https://openreview.net/forum?id=OnD9zGAGT0k}                    
}

@inproceedings{song2022pseudoinverse, title={Pseudoinverse-Guided Diffusion Models for Inverse Problems}, url={https://openreview.net/forum?id=9_gsMA8MRKQ}, abstractNote={Diffusion models have become competitive candidates for solving various inverse problems. Models trained for specific inverse problems work well but are limited to their particular use cases, whereas methods that use problem-agnostic models are general but often perform worse empirically. To address this dilemma, we introduce Pseudoinverse-guided Diffusion Models ($Pi$GDM), an approach that uses problem-agnostic models to close the gap in performance. $Pi$GDM directly estimates conditional scores from the measurement model of the inverse problem without additional training. It can address inverse problems with noisy, non-linear, or even non-differentiable measurements, in contrast to many existing approaches that are limited to noiseless linear ones. We illustrate the empirical effectiveness of $Pi$GDM on several image restoration tasks, including super-resolution, inpainting and JPEG restoration. On ImageNet, $Pi$GDM is competitive with state-of-the-art diffusion models trained on specific tasks, and is the first to achieve this with problem-agnostic diffusion models. $Pi$GDM can also solve a wider set of inverse problems where the measurement processes are composed of several simpler ones.}, author={Song, Jiaming and Vahdat, Arash and Mardani, Morteza and Kautz, Jan}, year={2022}, month=sep, language={en} }

@inproceedings{
kawar2021snipssolvingnoisyinverse,
title={{SNIPS}: Solving Noisy Inverse Problems Stochastically},
author={Bahjat Kawar and Gregory Vaksman and Michael Elad},
booktitle={Advances in Neural Information Processing Systems},
editor={A. Beygelzimer and Y. Dauphin and P. Liang and J. Wortman Vaughan},
year={2021},
url={https://openreview.net/forum?id=pBKOx_dxYAN}
}

@inproceedings{
kawar2022denoisingdiffusionrestorationmodels,
title={Denoising Diffusion Restoration Models},
author={Bahjat Kawar and Michael Elad and Stefano Ermon and Jiaming Song},
booktitle={Advances in Neural Information Processing Systems},
editor={Alice H. Oh and Alekh Agarwal and Danielle Belgrave and Kyunghyun Cho},
year={2022},
url={https://openreview.net/forum?id=kxXvopt9pWK}
}

@inproceedings{
wang_2022_zeroshot,
title={Zero-Shot Image Restoration Using Denoising Diffusion Null-Space Model},
author={Yinhuai Wang and Jiwen Yu and Jian Zhang},
booktitle={The Eleventh International Conference on Learning Representations },
year={2023},
url={https://openreview.net/forum?id=mRieQgMtNTQ}
}

@inproceedings{zhu2023denoising,
  author       = {Yuanzhi Zhu and
                  Kai Zhang and
                  Jingyun Liang and
                  Jiezhang Cao and
                  Bihan Wen and
                  Radu Timofte and
                  Luc Van Gool},
  title        = {Denoising Diffusion Models for Plug-and-Play Image Restoration},
  booktitle    = {{IEEE/CVF} Conference on Computer Vision and Pattern Recognition,
                  {CVPR} 2023 - Workshops, Vancouver, BC, Canada, June 17-24, 2023},
  pages        = {1219--1229},
  publisher    = {{IEEE}},
  year         = {2023},
  url          = {https://doi.org/10.1109/CVPRW59228.2023.00129},
  doi          = {10.1109/CVPRW59228.2023.00129},
  bibsource    = {dblp computer science bibliography, https://dblp.org}
}

@inproceedings{karras2022elucidatingdesignspacediffusionbased,
author = {Karras, Tero and Aittala, Miika and Laine, Samuli and Aila, Timo},
title = {Elucidating the design space of diffusion-based generative models},
year = {2022},
isbn = {9781713871088},
publisher = {Curran Associates Inc.},
address = {Red Hook, NY, USA},
booktitle = {Proceedings of the 36th International Conference on Neural Information Processing Systems},
articleno = {1926},
numpages = {13},
location = {New Orleans, LA, USA},
series = {NIPS '22}
}

@inproceedings{renaud2024plug,
author = {Renaud, Marien and Prost, Jean and Leclaire, Arthur and Papadakis, Nicolas},
title = {Plug-and-play image restoration with stochastic denoising regularization},
year = {2024},
publisher = {JMLR.org},
booktitle = {Proceedings of the 41st International Conference on Machine Learning},
articleno = {1728},
numpages = {37},
location = {Vienna, Austria},
series = {ICML'24}
}

@inproceedings{
wu2024principled,
title={Principled Probabilistic Imaging using Diffusion Models as Plug-and-Play Priors},
author={Zihui Wu and Yu Sun and Yifan Chen and Bingliang Zhang and Yisong Yue and Katherine Bouman},
booktitle={The Thirty-eighth Annual Conference on Neural Information Processing Systems},
year={2024},
url={https://openreview.net/forum?id=Xq9HQf7VNV}
}

@inproceedings{
mardani2023avariational,
title={A Variational Perspective on Solving Inverse Problems with Diffusion Models},
author={Morteza Mardani and Jiaming Song and Jan Kautz and Arash Vahdat},
booktitle={The Twelfth International Conference on Learning Representations},
year={2024},
url={https://openreview.net/forum?id=1YO4EE3SPB}
}

@inproceedings{song2023solving, title={Solving Inverse Problems with Latent Diffusion Models via Hard Data Consistency}, url={https://openreview.net/forum?id=j8hdRqOUhN}, abstractNote={Latent diffusion models have been demonstrated to generate high-quality images, while offering efficiency in model training compared to diffusion models operating in the pixel space. However, incorporating latent diffusion models to solve inverse problems remains a challenging problem due to the nonlinearity of the encoder and decoder. To address these issues, we propose ReSample, an algorithm that can solve general inverse problems with pre-trained latent diffusion models. Our algorithm incorporates data consistency by solving an optimization problem during the reverse sampling process, a concept that we term as hard data consistency. Upon solving this optimization problem, we propose a novel resampling scheme to map the measurement-consistent sample back onto the noisy data manifold and theoretically demonstrate its benefits. Lastly, we apply our algorithm to solve a wide range of linear and nonlinear inverse problems in both natural and medical images, demonstrating that our approach outperforms existing state-of-the-art approaches, including those based on pixel-space diffusion models.}, author={Song, Bowen and Kwon, Soo Min and Zhang, Zecheng and Hu, Xinyu and Qu, Qing and Shen, Liyue}, year={2023}, month=oct, language={en} }

@article{li2024decoupled, title={Decoupled Data Consistency with Diffusion Purification for Image Restoration}, url={http://arxiv.org/abs/2403.06054}, DOI={10.48550/arXiv.2403.06054}, abstractNote={Diffusion models have recently gained traction as a powerful class of deep generative priors, excelling in a wide range of image restoration tasks due to their exceptional ability to model data distributions. To solve image restoration problems, many existing techniques achieve data consistency by incorporating additional likelihood gradient steps into the reverse sampling process of diffusion models. However, the additional gradient steps pose a challenge for real-world practical applications as they incur a large computational overhead, thereby increasing inference time. They also present additional difficulties when using accelerated diffusion model samplers, as the number of data consistency steps is limited by the number of reverse sampling steps. In this work, we propose a novel diffusion-based image restoration solver that addresses these issues by decoupling the reverse process from the data consistency steps. Our method involves alternating between a reconstruction phase to maintain data consistency and a refinement phase that enforces the prior via diffusion purification. Our approach demonstrates versatility, making it highly adaptable for efficient problem-solving in latent space. Additionally, it reduces the necessity for numerous sampling steps through the integration of consistency models. The efficacy of our approach is validated through comprehensive experiments across various image restoration tasks, including image denoising, deblurring, inpainting, and superresolution.}, note={arXiv:2403.06054 [eess]}, number={arXiv:2403.06054}, publisher={arXiv}, author={Li, Xiang and Kwon, Soo Min and Alkhouri, Ismail R. and Ravishankar, Saiprasad and Qu, Qing}, year={2024}, month=may, language={en} }

@InProceedings{nie2022diffusionmodelsadversarialpurification,
  title = 	 {Diffusion Models for Adversarial Purification},
  author =       {Nie, Weili and Guo, Brandon and Huang, Yujia and Xiao, Chaowei and Vahdat, Arash and Anandkumar, Animashree},
  booktitle = 	 {Proceedings of the 39th International Conference on Machine Learning},
  pages = 	 {16805--16827},
  year = 	 {2022},
  editor = 	 {Chaudhuri, Kamalika and Jegelka, Stefanie and Song, Le and Szepesvari, Csaba and Niu, Gang and Sabato, Sivan},
  volume = 	 {162},
  series = 	 {Proceedings of Machine Learning Research},
  month = 	 {17--23 Jul},
  publisher =    {PMLR},
  url = 	 {https://proceedings.mlr.press/v162/nie22a.html}
}

@misc{alkhouri2023robustphysicsbaseddeepmri,
      title={Robust Physics-based Deep MRI Reconstruction Via Diffusion Purification}, 
      author={Ismail Alkhouri and Shijun Liang and Rongrong Wang and Qing Qu and Saiprasad Ravishankar},
      year={2023},
      eprint={2309.05794},
      archivePrefix={arXiv},
      primaryClass={eess.IV},
      url={https://arxiv.org/abs/2309.05794}, 
}

@inproceedings{
meng2022sdeditguidedimagesynthesis,
title={{SDE}dit: Guided Image Synthesis and Editing with Stochastic Differential Equations},
author={Chenlin Meng and Yutong He and Yang Song and Jiaming Song and Jiajun Wu and Jun-Yan Zhu and Stefano Ermon},
booktitle={International Conference on Learning Representations},
year={2022},
url={https://openreview.net/forum?id=aBsCjcPu_tE}
}

@article{robbins1992anempirical,
author = {Robbins, Herbert},
year = {1992},
month = {01},
pages = {},
title = {An Empirical Bayes Approach to Statistics},
volume = {1},
isbn = {978-0-387-94037-3},
journal = {Proceedings of the Third Berkeley Symposium on Mathematical and Statistical Probability},
doi = {10.1007/978-1-4612-0919-5_26}
}

@article{buzzard2018plugandplay,
  title={Plug-and-play unplugged: Optimization-free reconstruction using consensus equilibrium},
  author={Buzzard, Gregery T and Chan, Stanley H and Sreehari, Suhas and Bouman, Charles A},
  journal={SIAM Journal on Imaging Sciences},
  volume={11},
  number={3},
  pages={2001--2020},
  year={2018},
  publisher={SIAM}
}

@article{sun2019anonlineplugandplay, title={An Online Plug-and-Play Algorithm for Regularized Image Reconstruction}, volume={5}, ISSN={2333-9403, 2334-0118, 2573-0436}, DOI={10.1109/TCI.2019.2893568}, abstractNote={Plug-and-play priors (PnP) is a powerful framework for regularizing imaging inverse problems by using advanced denoisers within an iterative algorithm. Recent experimental evidence suggests that PnP algorithms achieve state-of-the-art performance in a range of imaging applications. In this paper, we introduce a new online PnP algorithm based on the iterative shrinkage/thresholding algorithm (ISTA). The proposed algorithm uses only a subset of measurements at every iteration, which makes it scalable to very large datasets. We present a new theoretical convergence analysis, for both batch and online variants of PnP-ISTA, for denoisers that do not necessarily correspond to proximal operators. We also present simulations illustrating the applicability of the algorithm to image reconstruction in diffraction tomography. The results in this paper have the potential to expand the applicability of the PnP framework to very large and redundant datasets.}, note={arXiv:1809.04693 [cs]}, number={3}, journal={IEEE Transactions on Computational Imaging}, author={Sun, Yu and Wohlberg, Brendt and Kamilov, Ulugbek S.}, year={2019}, month=sep, pages={395–408} }

@article{teodoro2017sceneadapted, title={Scene-adapted plug-and-play algorithm with convergence guarantees}, url={http://arxiv.org/abs/1702.02445}, DOI={10.48550/arXiv.1702.02445}, abstractNote={Recent frameworks, such as the so-called plug-and-play, allow us to leverage the developments in image denoising to tackle other, and more involved, problems in image processing. As the name suggests, state-of-the-art denoisers are plugged into an iterative algorithm that alternates between a denoising step and the inversion of the observation operator. While these tools offer flexibility, the convergence of the resulting algorithm may be difficult to analyse. In this paper, we plug a state-of-the-art denoiser, based on a Gaussian mixture model, in the iterations of an alternating direction method of multipliers and prove the algorithm is guaranteed to converge. Moreover, we build upon the concept of scene-adapted priors where we learn a model targeted to a specific scene being imaged, and apply the proposed method to address the hyperspectral sharpening problem.}, note={arXiv:1702.02445 [cs]}, number={arXiv:1702.02445}, publisher={arXiv}, author={Teodoro, Afonso M. and Bioucas-Dias, José M. and Figueiredo, Mário A. T.}, year={2017}, month=nov }

@article{chan2019performance, title={Performance Analysis of Plug-and-Play ADMM: A Graph Signal Processing Perspective}, url={http://arxiv.org/abs/1809.00020}, DOI={10.48550/arXiv.1809.00020}, abstractNote={The Plug-and-Play (PnP) ADMM algorithm is a powerful image restoration framework that allows advanced image denoising priors to be integrated into physical forward models to generate high quality image restoration results. However, despite the enormous number of applications and several theoretical studies trying to prove the convergence by leveraging tools in convex analysis, very little is known about why the algorithm is doing so well. The goal of this paper is to fill the gap by discussing the performance of PnP ADMM. By restricting the denoisers to the class of graph filters under a linearity assumption, or more specifically the symmetric smoothing filters, we offer three contributions: (1) We show conditions under which an equivalent maximum-a-posteriori (MAP) optimization exists, (2) we present a geometric interpretation and show that the performance gain is due to an intrinsic pre-denoising characteristic of the PnP prior, (3) we introduce a new analysis technique via the concept of consensus equilibrium, and provide interpretations to problems involving multiple priors.}, note={arXiv:1809.00020 [eess]}, number={arXiv:1809.00020}, publisher={arXiv}, author={Chan, Stanley H.}, year={2019}, month=may }

@article{ahmad2020plugandplay,
   title={Plug-and-Play Methods for Magnetic Resonance Imaging: Using Denoisers for Image Recovery},
   volume={37},
   ISSN={1558-0792},
   url={http://dx.doi.org/10.1109/MSP.2019.2949470},
   DOI={10.1109/msp.2019.2949470},
   number={1},
   journal={IEEE Signal Processing Magazine},
   publisher={Institute of Electrical and Electronics Engineers (IEEE)},
   author={Ahmad, Rizwan and Bouman, Charles A. and Buzzard, Gregery T. and Chan, Stanley and Liu, Sizhuo and Reehorst, Edward T. and Schniter, Philip},
   year={2020},
   month=jan, pages={105–116} }

@article{tran2021explore,
  title={Explore Image Deblurring via Encoded Blur Kernel Space},
  author={Phong Tran and A. Tran and Quynh Phung and Minh Hoai},
  journal={2021 IEEE/CVF Conference on Computer Vision and Pattern Recognition (CVPR)},
  year={2021},
  pages={11951-11960},
  url={https://api.semanticscholar.org/CorpusID:235328539}
}

@article{zhang2024improvingdiffusioninverseproblem,
  publtype={informal},
  author={Bingliang Zhang and Wenda Chu and Julius Berner and Chenlin Meng and Anima Anandkumar and Yang Song},
  title={Improving Diffusion Inverse Problem Solving with Decoupled Noise Annealing},
  year={2024},
  cdate={1704067200000},
  journal={CoRR},
  volume={abs/2407.01521},
  url={https://doi.org/10.48550/arXiv.2407.01521}
}

@INPROCEEDINGS{deng2009imagenet,
  author={Deng, Jia and Dong, Wei and Socher, Richard and Li, Li-Jia and Kai Li and Li Fei-Fei},
  booktitle={2009 IEEE Conference on Computer Vision and Pattern Recognition}, 
  title={ImageNet: A large-scale hierarchical image database}, 
  year={2009},
  volume={},
  number={},
  pages={248-255},
  doi={10.1109/CVPR.2009.5206848}}

@article{karras2019stylebasedgeneratorarchitecturegenerative,
author = {Karras, Tero and Laine, Samuli and Aila, Timo},
title = {A Style-Based Generator Architecture for Generative Adversarial Networks},
year = {2021},
issue_date = {Dec. 2021},
publisher = {IEEE Computer Society},
address = {USA},
volume = {43},
number = {12},
issn = {0162-8828},
url = {https://doi.org/10.1109/TPAMI.2020.2970919},
doi = {10.1109/TPAMI.2020.2970919},
journal = {IEEE Trans. Pattern Anal. Mach. Intell.},
month = dec,
pages = {4217–4228},
numpages = {12}
}

@INPROCEEDINGS {zhang2018theunreasonable,
author = { Zhang, Richard and Isola, Phillip and Efros, Alexei A. and Shechtman, Eli and Wang, Oliver },
booktitle = { 2018 IEEE/CVF Conference on Computer Vision and Pattern Recognition (CVPR) },
title = {{ The Unreasonable Effectiveness of Deep Features as a Perceptual Metric }},
year = {2018},
volume = {},
ISSN = {},
pages = {586-595},
doi = {10.1109/CVPR.2018.00068},
url = {https://doi.ieeecomputersociety.org/10.1109/CVPR.2018.00068},
publisher = {IEEE Computer Society},
address = {Los Alamitos, CA, USA},
month =Jun}

@inproceedings{kingma2017adammethodstochasticoptimization,
  author       = {Diederik P. Kingma and
                  Jimmy Ba},
  editor       = {Yoshua Bengio and
                  Yann LeCun},
  title        = {Adam: {A} Method for Stochastic Optimization},
  booktitle    = {3rd International Conference on Learning Representations, {ICLR} 2015,
                  San Diego, CA, USA, May 7-9, 2015, Conference Track Proceedings},
  year         = {2015},
  url          = {http://arxiv.org/abs/1412.6980},
  bibsource    = {dblp computer science bibliography, https://dblp.org}
}

@inproceedings{
song2022solving,
title={Solving Inverse Problems in Medical Imaging with Score-Based Generative Models},
author={Yang Song and Liyue Shen and Lei Xing and Stefano Ermon},
booktitle={International Conference on Learning Representations},
year={2022},
url={https://openreview.net/forum?id=vaRCHVj0uGI}
}

@article{jin2017deep,
  title={Deep convolutional neural network for inverse problems in imaging},
  author={Jin, Kyong Hwan and McCann, Michael T and Froustey, Emmanuel and Unser, Michael},
  journal={IEEE transactions on image processing},
  volume={26},
  number={9},
  pages={4509--4522},
  year={2017},
  publisher={IEEE}
}

@article{entekhabi1994solving,
  title={Solving the inverse problem for soil moisture and temperature profiles by sequential assimilation of multifrequency remotely sensed observations},
  author={Entekhabi, Dara and Nakamura, Hajime and Njoku, Eni G},
  journal={IEEE Transactions on Geoscience and Remote Sensing},
  volume={32},
  number={2},
  pages={438--448},
  year={1994},
  publisher={IEEE}
}

@article{combal2003retrieval,
  title={Retrieval of canopy biophysical variables from bidirectional reflectance: Using prior information to solve the ill-posed inverse problem},
  author={Combal, B and Baret, Fr{\'e}d{\'e}ric and Weiss, M and Trubuil, Alain and Mac{\'e}, D and Pragnere, A and Myneni, R and Knyazikhin, Y and Wang, L},
  journal={Remote sensing of environment},
  volume={84},
  number={1},
  pages={1--15},
  year={2003},
  publisher={Elsevier}
}

@book{bennett1992inverse,
  title={Inverse methods in physical oceanography},
  author={Bennett, Andrew F},
  year={1992},
  publisher={Cambridge university press}
}

@article{arridge1999optical,
  title={Optical tomography in medical imaging},
  author={Arridge, Simon R},
  journal={Inverse problems},
  volume={15},
  number={2},
  pages={R41},
  year={1999},
  publisher={IOP Publishing}
}

@article{raissi2019physics,
  title={Physics-informed neural networks: A deep learning framework for solving forward and inverse problems involving nonlinear partial differential equations},
  author={Raissi, Maziar and Perdikaris, Paris and Karniadakis, George E},
  journal={Journal of Computational physics},
  volume={378},
  pages={686--707},
  year={2019},
  publisher={Elsevier}
}

@book{tarantola2005inverse,
  title={Inverse problem theory and methods for model parameter estimation},
  author={Tarantola, Albert},
  year={2005},
  publisher={SIAM}
}

@ARTICLE{yang2010imagesuperresolution,
  author={Yang, Jianchao and Wright, John and Huang, Thomas S. and Ma, Yi},
  journal={IEEE Transactions on Image Processing}, 
  title={Image Super-Resolution Via Sparse Representation}, 
  year={2010},
  volume={19},
  number={11},
  pages={2861-2873},
  doi={10.1109/TIP.2010.2050625}}

@ARTICLE{elad2006imagedenoising,
  author={Elad, Michael and Aharon, Michal},
  journal={IEEE Transactions on Image Processing}, 
  title={Image Denoising Via Sparse and Redundant Representations Over Learned Dictionaries}, 
  year={2006},
  volume={15},
  number={12},
  pages={3736-3745},
  doi={10.1109/TIP.2006.881969}}

@ARTICLE{dabov2007imagedenoising,
  author={Dabov, Kostadin and Foi, Alessandro and Katkovnik, Vladimir and Egiazarian, Karen},
  journal={IEEE Transactions on Image Processing}, 
  title={Image Denoising by Sparse 3-D Transform-Domain Collaborative Filtering}, 
  year={2007},
  volume={16},
  number={8},
  pages={2080-2095},
  doi={10.1109/TIP.2007.901238}}

@ARTICLE{semerci2014tensorbased,
  author={Semerci, Oguz and Hao, Ning and Kilmer, Misha E. and Miller, Eric L.},
  journal={IEEE Transactions on Image Processing}, 
  title={Tensor-Based Formulation and Nuclear Norm Regularization for Multienergy Computed Tomography}, 
  year={2014},
  volume={23},
  number={4},
  pages={1678-1693},
  doi={10.1109/TIP.2014.2305840}}

@ARTICLE{hu2017thetwist,
  author={Hu, Wenrui and Tao, Dacheng and Zhang, Wensheng and Xie, Yuan and Yang, Yehui},
  journal={IEEE Transactions on Neural Networks and Learning Systems}, 
  title={The Twist Tensor Nuclear Norm for Video Completion}, 
  year={2017},
  volume={28},
  number={12},
  pages={2961-2973},
  doi={10.1109/TNNLS.2016.2611525}}

@article{ulyanov2020deepimageprior,
   title={Deep Image Prior},
   volume={128},
   ISSN={1573-1405},
   url={http://dx.doi.org/10.1007/s11263-020-01303-4},
   DOI={10.1007/s11263-020-01303-4},
   number={7},
   journal={International Journal of Computer Vision},
   publisher={Springer Science and Business Media LLC},
   author={Ulyanov, Dmitry and Vedaldi, Andrea and Lempitsky, Victor},
   year={2020},
   month=mar, pages={1867–1888} }

@inproceedings{
alkhouri2024image,
title={Image Reconstruction Via Autoencoding Sequential Deep Image Prior},
author={Ismail Alkhouri and Shijun Liang and Evan Bell and Qing Qu and Rongrong Wang and Saiprasad Ravishankar},
booktitle={The Thirty-eighth Annual Conference on Neural Information Processing Systems},
year={2024},
url={https://openreview.net/forum?id=K1EG2ABzNE}
}

@inproceedings{shah2018solving,
  title={Solving linear inverse problems using gan priors: An algorithm with provable guarantees},
  author={Shah, Viraj and Hegde, Chinmay},
  booktitle={2018 IEEE international conference on acoustics, speech and signal processing (ICASSP)},
  pages={4609--4613},
  year={2018},
  organization={IEEE}
}

@inproceedings{
xiao2022tackling,
title={Tackling the Generative Learning Trilemma with Denoising Diffusion {GAN}s},
author={Zhisheng Xiao and Karsten Kreis and Arash Vahdat},
booktitle={International Conference on Learning Representations},
year={2022},
url={https://openreview.net/forum?id=JprM0p-q0Co}
}

@article{renaud2024plugandplay,
  publtype={informal},
  author={Marien Renaud and Jean Prost and Arthur Leclaire and Nicolas Papadakis},
  title={Plug-and-Play image restoration with Stochastic deNOising REgularization},
  year={2024},
  cdate={1704067200000},
  journal={CoRR},
  volume={abs/2402.01779},
  url={https://doi.org/10.48550/arXiv.2402.01779}
}

@inproceedings{
wang2024dmplug,
title={{DMP}lug: A Plug-in Method for Solving Inverse Problems with Diffusion Models},
author={Hengkang Wang and Xu Zhang and Taihui Li and Yuxiang Wan and Tiancong Chen and Ju Sun},
booktitle={The Thirty-eighth Annual Conference on Neural Information Processing Systems},
year={2024},
url={https://openreview.net/forum?id=81IFFsfQUj}
}

@article{ren2021onthecomplexity,
title = {On the complexity of extending the convergence ball of Wang’s method for finding a zero of a derivative},
journal = {Journal of Complexity},
volume = {64},
pages = {101526},
year = {2021},
issn = {0885-064X},
doi = {https://doi.org/10.1016/j.jco.2020.101526},
url = {https://www.sciencedirect.com/science/article/pii/S0885064X20300704},
author = {Hongmin Ren and Ioannis K. Argyros}
}

@article{ren2009convergenceball,
title = {Convergence ball and error analysis of a family of iterative methods with cubic convergence},
journal = {Applied Mathematics and Computation},
volume = {209},
number = {2},
pages = {369-378},
year = {2009},
issn = {0096-3003},
doi = {https://doi.org/10.1016/j.amc.2008.12.057},
url = {https://www.sciencedirect.com/science/article/pii/S0096300308009739},
author = {Hongmin Ren and Qingbiao Wu}
}

@article{liang2007homocentric,
  title={Homocentric convergence ball of the Secant method},
  author={Liang, Kewei},
  journal={Applied Mathematics-A Journal of Chinese Universities},
  volume={22},
  pages={353--365},
  year={2007},
  publisher={Springer}
}

@ARTICLE{yuan2022plugandplay,
  author={Yuan, Xin and Liu, Yang and Suo, Jinli and Durand, Frédo and Dai, Qionghai},
  journal={IEEE Transactions on Pattern Analysis and Machine Intelligence}, 
  title={Plug-and-Play Algorithms for Video Snapshot Compressive Imaging}, 
  year={2022},
  volume={44},
  number={10},
  pages={7093-7111},
  doi={10.1109/TPAMI.2021.3099035}}

@ARTICLE{liu2022hyperspectral,
  author={Liu, Yun-Yang and Zhao, Xi-Le and Zheng, Yu-Bang and Ma, Tian-Hui and Zhang, Hongyan},
  journal={IEEE Transactions on Geoscience and Remote Sensing}, 
  title={Hyperspectral Image Restoration by Tensor Fibered Rank Constrained Optimization and Plug-and-Play Regularization}, 
  year={2022},
  volume={60},
  number={},
  pages={1-17},
  doi={10.1109/TGRS.2020.3045169}}

@inproceedings{
he2024manifold,
title={Manifold Preserving Guided Diffusion},
author={Yutong He and Naoki Murata and Chieh-Hsin Lai and Yuhta Takida and Toshimitsu Uesaka and Dongjun Kim and Wei-Hsiang Liao and Yuki Mitsufuji and J Zico Kolter and Ruslan Salakhutdinov and Stefano Ermon},
booktitle={The Twelfth International Conference on Learning Representations},
year={2024},
url={https://openreview.net/forum?id=o3BxOLoxm1}
}

@inproceedings{
zirvi2025diffusion,
title={Diffusion State-Guided Projected Gradient for Inverse Problems},
author={Rayhan Zirvi and Bahareh Tolooshams and Anima Anandkumar},
booktitle={The Thirteenth International Conference on Learning Representations},
year={2025},
url={https://openreview.net/forum?id=kRBQwlkFSP}
}

@inproceedings{
graikos2022diffusion,
title={Diffusion Models as Plug-and-Play Priors},
author={Alexandros Graikos and Nikolay Malkin and Nebojsa Jojic and Dimitris Samaras},
booktitle={Advances in Neural Information Processing Systems},
editor={Alice H. Oh and Alekh Agarwal and Danielle Belgrave and Kyunghyun Cho},
year={2022},
url={https://openreview.net/forum?id=yhlMZ3iR7Pu}
}

@article{yao2008iterative,
  title={Iterative algorithms with variable anchors for non-expansive mappings},
  author={Yao, Yonghong and Zhou, Haiyun and Liou, Yeong-Cheng},
  journal={Journal of Applied Mathematics and Computing},
  volume={28},
  pages={39--49},
  year={2008},
  publisher={Springer}
}

@article{eckstein1992douglas,
  title={On the Douglas—Rachford splitting method and the proximal point algorithm for maximal monotone operators},
  author={Eckstein, Jonathan and Bertsekas, Dimitri P},
  journal={Mathematical programming},
  volume={55},
  pages={293--318},
  year={1992},
  publisher={Springer}
}

@book{pardo2018statistical,
  title={Statistical inference based on divergence measures},
  author={Pardo, Leandro},
  year={2018},
  publisher={Chapman and Hall/CRC}
}

@article{uhlenbeck1930onthetheory,
  title = {On the Theory of the Brownian Motion},
  author = {Uhlenbeck, G. E. and Ornstein, L. S.},
  journal = {Phys. Rev.},
  volume = {36},
  issue = {5},
  pages = {823--841},
  numpages = {0},
  year = {1930},
  month = {Sep},
  publisher = {American Physical Society},
  doi = {10.1103/PhysRev.36.823},
  url = {https://link.aps.org/doi/10.1103/PhysRev.36.823}
}

@book{titchmarsh1986theory,
  title={The theory of the Riemann zeta-function},
  author={Titchmarsh, Edward Charles and Heath-Brown, David Rodney},
  year={1986},
  publisher={Oxford university press}
}

@article{bowder1965nonexpansive,
 ISSN = {00278424, 10916490},
 URL = {http://www.jstor.org/stable/73047},
 author = {Felix E. Browder},
 journal = {Proceedings of the National Academy of Sciences of the United States of America},
 number = {4},
 pages = {1041--1044},
 publisher = {National Academy of Sciences},
 title = {Nonexpansive Nonlinear Operators in a Banach Space},
 urldate = {2025-05-20},
 volume = {54},
 year = {1965}
}

@article{brascamp1976onextensions,
title = {On extensions of the Brunn-Minkowski and Prékopa-Leindler theorems, including inequalities for log concave functions, and with an application to the diffusion equation},
journal = {Journal of Functional Analysis},
volume = {22},
number = {4},
pages = {366-389},
year = {1976},
issn = {0022-1236},
doi = {https://doi.org/10.1016/0022-1236(76)90004-5},
url = {https://www.sciencedirect.com/science/article/pii/0022123676900045},
author = {Herm Jan Brascamp and Elliott H Lieb}
}

@misc{peng2024improvingdiffusionmodelsinverse,
      title={Improving Diffusion Models for Inverse Problems Using Optimal Posterior Covariance}, 
      author={Xinyu Peng and Ziyang Zheng and Wenrui Dai and Nuoqian Xiao and Chenglin Li and Junni Zou and Hongkai Xiong},
      year={2024},
      eprint={2402.02149},
      archivePrefix={arXiv},
      primaryClass={cs.CV},
      url={https://arxiv.org/abs/2402.02149}, 
}

@misc{dytso2021generalderivativeidentityconditional,
      title={A General Derivative Identity for the Conditional Mean Estimator in Gaussian Noise and Some Applications}, 
      author={Alex Dytso and H. Vincent Poor and Shlomo Shamai},
      year={2021},
      eprint={2104.01883},
      archivePrefix={arXiv},
      primaryClass={cs.IT},
      url={https://arxiv.org/abs/2104.01883}, 
}

@ARTICLE{hatsell1971somegeometric,
  author={Hatsell, C. and Nolte, L.},
  journal={IEEE Transactions on Information Theory}, 
  title={Some geometric properties of the likelihood ratio (Corresp.)}, 
  year={1971},
  volume={17},
  number={5},
  pages={616-618},
  doi={10.1109/TIT.1971.1054672}}

@ARTICLE{palomar2006gradient,
  author={Palomar, D.P. and Verdu, S.},
  journal={IEEE Transactions on Information Theory}, 
  title={Gradient of mutual information in linear vector Gaussian channels}, 
  year={2006},
  volume={52},
  number={1},
  pages={141-154},
  doi={10.1109/TIT.2005.860424}}

@book{boyd2004convex,
  title={Convex optimization},
  author={Boyd, Stephen P and Vandenberghe, Lieven},
  year={2004},
  publisher={Cambridge university press}
}

@article{mattingly2002ergodicity, title={Ergodicity for SDEs and approximations: locally Lipschitz vector fields and degenerate noise}, volume={101}, ISSN={0304-4149}, DOI={10.1016/S0304-4149(02)00150-3}, abstractNote={The ergodic properties of SDEs, and various time discretizations for SDEs, are studied. The ergodicity of SDEs is established by using techniques from the theory of Markov chains on general state spaces, such as that expounded by Meyn–Tweedie. Application of these Markov chain results leads to straightforward proofs of geometric ergodicity for a variety of SDEs, including problems with degenerate noise and for problems with locally Lipschitz vector fields. Applications where this theory can be usefully applied include damped-driven Hamiltonian problems (the Langevin equation), the Lorenz equation with degenerate noise and gradient systems. The same Markov chain theory is then used to study time-discrete approximations of these SDEs. The two primary ingredients for ergodicity are a minorization condition and a Lyapunov condition. It is shown that the minorization condition is robust under approximation. For globally Lipschitz vector fields this is also true of the Lyapunov condition. However in the locally Lipschitz case the Lyapunov condition fails for explicit methods such as Euler–Maruyama; for pathwise approximations it is, in general, only inherited by specially constructed implicit discretizations. Examples of such discretization based on backward Euler methods are given, and approximation of the Langevin equation studied in some detail.}, number={2}, journal={Stochastic Processes and their Applications}, author={Mattingly, J. C. and Stuart, A. M. and Higham, D. J.}, year={2002}, month=oct, pages={185–232} }

@article{chen2020onstationarypoint,
  author  = {Xi Chen and Simon S. Du and Xin T. Tong},
  title   = {On Stationary-Point Hitting Time and Ergodicity of Stochastic Gradient Langevin Dynamics},
  journal = {Journal of Machine Learning Research},
  year    = {2020},
  volume  = {21},
  number  = {68},
  pages   = {1--41},
  url     = {http://jmlr.org/papers/v21/19-327.html}
}

@ARTICLE{zhang2022plugandplay,
  author={Zhang, Kai and Li, Yawei and Zuo, Wangmeng and Zhang, Lei and Van Gool, Luc and Timofte, Radu},
  journal={IEEE Transactions on Pattern Analysis and Machine Intelligence}, 
  title={Plug-and-Play Image Restoration With Deep Denoiser Prior}, 
  year={2022},
  volume={44},
  number={10},
  pages={6360-6376},
  doi={10.1109/TPAMI.2021.3088914}}

@misc{li2025decoupleddataconsistencydiffusion,
      title={Decoupled Data Consistency with Diffusion Purification for Image Restoration}, 
      author={Xiang Li and Soo Min Kwon and Shijun Liang and Ismail R. Alkhouri and Saiprasad Ravishankar and Qing Qu},
      year={2025},
      eprint={2403.06054},
      archivePrefix={arXiv},
      primaryClass={eess.IV},
      url={https://arxiv.org/abs/2403.06054}, 
}

@misc{sun2024provableprobabilisticimagingusing,
      title={Provable Probabilistic Imaging using Score-Based Generative Priors}, 
      author={Yu Sun and Zihui Wu and Yifan Chen and Berthy T. Feng and Katherine L. Bouman},
      year={2024},
      eprint={2310.10835},
      archivePrefix={arXiv},
      primaryClass={eess.IV},
      url={https://arxiv.org/abs/2310.10835}, 
}

@article{dalalyan2019userfriendly, title={User-friendly guarantees for the Langevin Monte Carlo with inaccurate gradient}, volume={129}, ISSN={03044149}, DOI={10.1016/j.spa.2019.02.016}, abstractNote={In this paper, we study the problem of sampling from a given probability density function that is known to be smooth and strongly log-concave. We analyze several methods of approximate sampling based on discretizations of the (highly overdamped) Langevin diffusion and establish guarantees on its error measured in the Wasserstein-2 distance. Our guarantees improve or extend the state-of-the-art results in three directions. First, we provide an upper bound on the error of the first-order Langevin Monte Carlo (LMC) algorithm with optimized varying step-size. This result has the advantage of being horizon free (we do not need to know in advance the target precision) and to improve by a logarithmic factor the corresponding result for the constant step-size. Second, we study the case where accurate evaluations of the gradient of the log-density are unavailable, but one can have access to approximations of the aforementioned gradient. In such a situation, we consider both deterministic and stochastic approximations of the gradient and provide an upper bound on the sampling error of the first-order LMC that quantifies the impact of the gradient evaluation inaccuracies. Third, we establish upper bounds for two versions of the second-order LMC, which leverage the Hessian of the log-density. We provide nonasymptotic guarantees on the sampling error of these second-order LMCs. These guarantees reveal that the second-order LMC algorithms improve on the first-order LMC in ill-conditioned settings.}, note={arXiv:1710.00095 [math]}, number={12}, journal={Stochastic Processes and their Applications}, author={Dalalyan, Arnak S. and Karagulyan, Avetik G.}, year={2019}, month=dec, pages={5278–5311} }

@book{villani2008optimal,
  title={Optimal transport: old and new},
  author={Villani, C{\'e}dric and others},
  volume={338},
  year={2008},
  publisher={Springer}
}

@article{metzler2016adenoising, title={From Denoising to Compressed Sensing}, url={http://arxiv.org/abs/1406.4175}, DOI={10.48550/arXiv.1406.4175}, abstractNote={A denoising algorithm seeks to remove noise, errors, or perturbations from a signal. Extensive research has been devoted to this arena over the last several decades, and as a result, todays denoisers can effectively remove large amounts of additive white Gaussian noise. A compressed sensing (CS) reconstruction algorithm seeks to recover a structured signal acquired using a small number of randomized measurements. Typical CS reconstruction algorithms can be cast as iteratively estimating a signal from a perturbed observation. This paper answers a natural question: How can one effectively employ a generic denoiser in a CS reconstruction algorithm? In response, we develop an extension of the approximate message passing (AMP) framework, called Denoising-based AMP (D-AMP), that can integrate a wide class of denoisers within its iterations. We demonstrate that, when used with a high performance denoiser for natural images, D-AMP offers state-of-the-art CS recovery performance while operating tens of times faster than competing methods. We explain the exceptional performance of D-AMP by analyzing some of its theoretical features. A key element in DAMP is the use of an appropriate Onsager correction term in its iterations, which coerces the signal perturbation at each iteration to be very close to the white Gaussian noise that denoisers are typically designed to remove.}, note={arXiv:1406.4175 [cs]}, number={arXiv:1406.4175}, publisher={arXiv}, author={Metzler, Christopher A. and Maleki, Arian and Baraniuk, Richard G.}, year={2016}, month=apr, language={en} }

@article{eksioglu2018denoising, title={Denoising AMP for MRI Reconstruction: BM3D-AMP-MRI}, volume={11}, DOI={10.1137/18M1169655}, abstractNote={Magnetic resonance imaging (MRI) is a versatile imaging technique that allows different contrasts depending on the acquisition parameters. Many clinical imaging studies acquire MRI data for more than one of these contrasts---such as, for instance, ��1 and ��2 weighted images---which makes the overall scanning procedure very time consuming. As all of these images show the same underlying anatomy, one can try to omit unnecessary measurements by taking the similarity into account during reconstruction. We will discuss two modifications of total variation---based on (i) location and (ii) direction---that take structural a priori knowledge into account and reduce to total variation in the degenerate case when no structural knowledge is available. We solve the resulting convex minimization problem with the alternating direction method of multipliers which separates the forward operator from the prior. For both priors the corresponding proximal operator can be implemented as an extension of the fast gradient projection method on the dual problem for total variation. We tested the priors on six data sets that are based on phantoms and real MRI images. In all test cases, exploiting the structural information from the other contrast yields better results than separate reconstruction with total variation in terms of standard metrics like peak signal-to-noise ratio and structural similarity index. Furthermore, we found that exploiting the two-dimensional directional information results in images with well-defined edges, superior to those reconstructed solely using a priori information about the edge location.}, number={3}, journal={SIAM Journal on Imaging Sciences}, publisher={Society for Industrial and Applied Mathematics}, author={Eksioglu, Ender M. and Tanc, A. Korhan}, year={2018}, month=jan, pages={2090–2109} }

@article{wei2021tfpnp, title={TFPnP: Tuning-free Plug-and-Play Proximal Algorithm with Applications to Inverse Imaging Problems}, url={http://arxiv.org/abs/2012.05703}, DOI={10.48550/arXiv.2012.05703}, abstractNote={Plug-and-Play (PnP) is a non-convex optimization framework that combines proximal algorithms, for example, the alternating direction method of multipliers (ADMM), with advanced denoising priors. Over the past few years, great empirical success has been obtained by PnP algorithms, especially for the ones that integrate deep learning-based denoisers. However, a key challenge of PnP approaches is the need for manual parameter tweaking as it is essential to obtain high-quality results across the high discrepancy in imaging conditions and varying scene content. In this work, we present a class of tuning-free PnP proximal algorithms that can determine parameters such as denoising strength, termination time, and other optimization-specific parameters automatically. A core part of our approach is a policy network for automated parameter search which can be effectively learned via a mixture of model-free and model-based deep reinforcement learning strategies. We demonstrate, through rigorous numerical and visual experiments, that the learned policy can customize parameters to different settings, and is often more efficient and effective than existing handcrafted criteria. Moreover, we discuss several practical considerations of PnP denoisers, which together with our learned policy yield state-of-the-art results. This advanced performance is prevalent on both linear and nonlinear exemplar inverse imaging problems, and in particular shows promising results on compressed sensing MRI, sparse-view CT, single-photon imaging, and phase retrieval.}, note={arXiv:2012.05703 [cs]}, number={arXiv:2012.05703}, publisher={arXiv}, author={Wei, Kaixuan and Aviles-Rivero, Angelica and Liang, Jingwei and Fu, Ying and Huang, Hua and Schönlieb, Carola-Bibiane}, year={2021}, month=sept }

@misc{xu2025radiomapestimationlatent,
      title={Radio Map Estimation via Latent Domain Plug-and-Play Denoising}, 
      author={Le Xu and Lei Cheng and Junting Chen and Wenqiang Pu and Xiao Fu},
      year={2025},
      eprint={2501.13472},
      archivePrefix={arXiv},
      primaryClass={eess.SP},
      url={https://arxiv.org/abs/2501.13472}, 
}

@article{huralt2022gradientstep, title={Gradient Step Denoiser for convergent Plug-and-Play}, url={http://arxiv.org/abs/2110.03220}, DOI={10.48550/arXiv.2110.03220}, abstractNote={Plug-and-Play methods constitute a class of iterative algorithms for imaging problems where regularization is performed by an off-the-shelf denoiser. Although Plug-and-Play methods can lead to tremendous visual performance for various image problems, the few existing convergence guarantees are based on unrealistic (or suboptimal) hypotheses on the denoiser, or limited to strongly convex data terms. In this work, we propose a new type of Plug-and-Play methods, based on half-quadratic splitting, for which the denoiser is realized as a gradient descent step on a functional parameterized by a deep neural network. Exploiting convergence results for proximal gradient descent algorithms in the non-convex setting, we show that the proposed Plug-and-Play algorithm is a convergent iterative scheme that targets stationary points of an explicit global functional. Besides, experiments show that it is possible to learn such a deep denoiser while not compromising the performance in comparison to other state-of-the-art deep denoisers used in Plug-and-Play schemes. We apply our proximal gradient algorithm to various ill-posed inverse problems, e.g. deblurring, super-resolution and inpainting. For all these applications, numerical results empirically confirm the convergence results. Experiments also show that this new algorithm reaches state-of-the-art performance, both quantitatively and qualitatively.}, note={arXiv:2110.03220 [cs]}, number={arXiv:2110.03220}, publisher={arXiv}, author={Hurault, Samuel and Leclaire, Arthur and Papadakis, Nicolas}, year={2022}, month=feb }

@inproceedings{huralt2022proximal, title={Proximal Denoiser for Convergent Plug-and-Play Optimization with Nonconvex Regularization}, ISSN={2640-3498}, url={https://proceedings.mlr.press/v162/hurault22a.html}, abstractNote={Plug-and-Play (PnP) methods solve ill-posed inverse problems through iterative proximal algorithms by replacing a proximal operator by a denoising operation. When applied with deep neural network denoisers, these methods have shown state-of-the-art visual performance for image restoration problems. However, their theoretical convergence analysis is still incomplete. Most of the existing convergence results consider nonexpansive denoisers, which is non-realistic, or limit their analysis to strongly convex data-fidelity terms in the inverse problem to solve. Recently, it was proposed to train the denoiser as a gradient descent step on a functional parameterized by a deep neural network. Using such a denoiser guarantees the convergence of the PnP version of the Half-Quadratic-Splitting (PnP-HQS) iterative algorithm. In this paper, we show that this gradient denoiser can actually correspond to the proximal operator of another scalar function. Given this new result, we exploit the convergence theory of proximal algorithms in the nonconvex setting to obtain convergence results for PnP-PGD (Proximal Gradient Descent) and PnP-ADMM (Alternating Direction Method of Multipliers). When built on top of a smooth gradient denoiser, we show that PnP-PGD and PnP-ADMM are convergent and target stationary points of an explicit functional. These convergence results are confirmed with numerical experiments on deblurring, super-resolution and inpainting.}, booktitle={Proceedings of the 39th International Conference on Machine Learning}, publisher={PMLR}, author={Hurault, Samuel and Leclaire, Arthur and Papadakis, Nicolas}, year={2022}, month=june, pages={9483–9505}, language={en} }

@article{wei2025learning, title={Learning Cocoercive Conservative Denoisers via Helmholtz Decomposition for Poisson Inverse Problems}, url={http://arxiv.org/abs/2505.08909}, DOI={10.48550/arXiv.2505.08909}, abstractNote={Plug-and-play (PnP) methods with deep denoisers have shown impressive results in imaging problems. They typically require strong convexity or smoothness of the fidelity term and a (residual) non-expansive denoiser for convergence. These assumptions, however, are violated in Poisson inverse problems, and non-expansiveness can hinder denoising performance. To address these challenges, we propose a cocoercive conservative (CoCo) denoiser, which may be (residual) expansive, leading to improved denoising. By leveraging the generalized Helmholtz decomposition, we introduce a novel training strategy that combines Hamiltonian regularization to promote conservativeness and spectral regularization to ensure cocoerciveness. We prove that CoCo denoiser is a proximal operator of a weakly convex function, enabling a restoration model with an implicit weakly convex prior. The global convergence of PnP methods to a stationary point of this restoration model is established. Extensive experimental results demonstrate that our approach outperforms closely related methods in both visual quality and quantitative metrics.}, note={arXiv:2505.08909 [cs]}, number={arXiv:2505.08909}, publisher={arXiv}, author={Wei, Deliang and Chen, Peng and Xu, Haobo and Yao, Jiale and Li, Fang and Zeng, Tieyong}, year={2025}, month=oct }
\bibliographystyle{iclr2026_conference}

\clearpage
\appendix
\section{Preliminaries} \label{appendix: preliminaries}

\subsection{Notation} \label{appendix: notations}

\begin{table}[ht]
  \centering
  \setlength{\tabcolsep}{4pt}
  \renewcommand{\arraystretch}{1.1}
  \caption{Summary of notation}
  \label{tab:notation}
  \begin{tabular}{@{}|l|l|@{}}
  \toprule
  \textbf{Symbol} & \textbf{Description} \\
  \midrule
  $\bx \in \Rbb^d$ & The unknown signal or image to be recovered\\ 
  $\by \in \Rbb^n$ & Measurements, with $n\le d$\\
  $y_i$ & $i$th element of measurement $\by$ \\
  $\setA: \Rbb^d \to \Rbb^n$ & Measurement operator \\
  $\setX$ & Support of $\bx$ \\
  $\setX_t$ & Support of $\bx_t$ \\
  $\bxi$ & Additive measurement noise\\ 
  $\bx_t$ & noisy data by using forward diffusion process with noise $\sigma(t)$\\
  $\bx_{\sigmak}$ & noisy data by using forward diffusion process with noise $\sigmak$\\
  $\ell(\by||\setA(\bx))$ & data-fidelity loss (e.g. $\norm{\by - \setA(\bx)}_2^2$) \\
  $h(\bx)$ & structural regularization prior (enforced via denoiser) \\
  $\rho>0$ & ADMM penalty parameter \\
  Prox($\cdot$) & Proximal operator \\
  $\buk$ & Scaled dual variable in iteration $k$ of ADMM \\
  $\bzk$ & Auxiliary variable in iteration $k$ of ADMM \\
  $\tbzk=\bxkk+\buk$ & Pre-denoising input to the PnP denoiser \\
  $\sigmak$ & Noise level schedule for the AC-DC denoiser\\
  $\sigma_{\bsk}$ & \makecell[tl]{Variance parameter in the AC-DC prior for \\ directional correction.}\\
  $\bn \sim \setN(\bm 0, \bI)$ & Multivariate standard gaussian random variable \\
  $\bs_{\bm \theta}(\bx, \sigma)\approx \nabla_{\bx} \log p(\bx + \sigma \bn)$ & Pretrained score function \\
  $K$ & Maximum iteration of ADMM \\
  $J$ & Total iteration of directional correction at for each denoising\\
  $M$ & Smoothness constant of $\nabla \log p_{\rm data}$ (Assumption~\ref{assumpt: smoothness of logpdata})\\
  $M_t$ & Smoothness constant of $\nabla \log p_t$\\
  $T>0$ & Maximum time steps used for diffusion \\
  $\setM_{\sigma(t)}$ & Manifold of $\bx_t$\\
  $\setM_{\sigmak}$ & Manifold of $\bx_t$ where $t\in[0,T]$ such that $\sigma(t)=\sigmak$ \\
  $D_{\sigmak}(\bz)$ & AC-DC denoiser at $k$th iteration  \\
  $R_{\sigmak}(\bz)=D_{\sigmak}(\bz)-\bz$ & Residual of AC-DC denoiser at $k$th iteration \\
  $I(\bx)=\bx$ & Identity mapping function \\
  $\bI$ & Identity matrix \\
  $\bm 0$ & vector of values $0$ \\
  $T_1 \circ T_2 (\bz) = T_1(T_2(\bz))$ & Concatenation of two functions $T_1$ and $T_2$ \\
  $\norm{\bx}_2$ & 2-norm of a vector $\bx$\\
  $\text{Cov}(\cdot)$ & Covariance matrix\\
  \bottomrule
  \end{tabular}
  \end{table}

\subsection{Definitions}

\paragraph{$\mu$-strongly convex function \citep{boyd2004convex}.} A differentiable function $f: \Rbb^m \to \Rbb$ is \emph{$\mu$-strongly convex} for a certain $\mu>0$ if
\begin{equation}
    f(\by) \ge f(\bx) + \nabla f(\bx)^T (\by-\bx) + \frac{\mu}{2} \norm{\by-\bx}^2_2
\end{equation}
for all $\bx,\by \in \Rbb^m $.

The notion  of of \emph{nonexpansive} and \emph{averaged nonexpansive} have been widely used in the convergence analysis of various nonlinear problems \citep{combettes2014compositionsconvexcombinationsaveraged,yao2008iterative,eckstein1992douglas}. We  use the generalized form of both \emph{nonexpansive} and \emph{averaged nonexpansive} operator for establishing the ball convergence in our method.

\paragraph{Nonexpansive function { \citep{bowder1965nonexpansive}}.}
A function $T: \Rbb^d \to \Rbb^d$ is \emph{nonexpansive} if $T$ is \emph{nonexpansive} function if there exists $\epsilon \in [0,1]$ such that
\begin{equation}
    \norm{T(\bx) - T(\by)}_2^2 \le \epsilon^2 \norm{\bx - \by}_2^2 
\end{equation}
for all $\bx,\by \in \Rbb^d$.

\paragraph{$\theta$-averaged function { \citep{combettes2014compositionsconvexcombinationsaveraged}}.}
A mapping $T: \Rbb^d \to \Rbb^d$ is defined to be \emph{$\theta$-averaged} for a constant $\theta \in (0,1)$ if there exists a nonexpansive operator $R: \Rbb^d \to \Rbb^d$ such that $T = (1-\theta) I + \theta R$.

 The notion of relaxed bound \(\norm{T_k(\bx)-T_k(\by)} \le \epsilon^{(k)} \norm{\bx-\by} + \delta^{(k)}, \; \forall \bx,\by \in \setX\) was used to study and show the convergence in \citet{yao2008iterative} when $\sum_{k=1}^{\infty} |\delta^{(k)}| < \infty$. We define a similar weaker form of nonexpansive function and $\theta$-averaged functions below. 

\paragraph{$\delta$-weakly nonexpansive function.}
A mapping $T: \Rbb^d \to \Rbb^d$ is said to be \emph{$\delta$-weakly nonexpansive} for $\delta \ge 0$ if there exists $\epsilon\in [0,1]$ such that
\begin{equation}
    \norm{T(\bx) - T(\by)}_2^2 \le \epsilon^2 \norm{\bx - \by}_2^2  + \delta^2
\end{equation}
for all $\bx,\by \in \Rbb^d$.

\paragraph{$\delta$-weakly $\theta$-averaged function.}
A mapping $T: \Rbb^d \to \Rbb^d$ is defined to be \emph{$\delta$-weakly $\theta$-averaged} function for a certain $\delta\ge0$ and $\theta \in (0,1)$, if  there exists a $\delta$-weakly nonexpansive function $R:\Rbb^d \to \Rbb^d$ such that $T=\theta R + (1-\theta) I$.

{
\paragraph{Sub-Gaussian random vector.}
A random vector $\bx \in \Rbb^d$ (with mean $\mathbb{E}[\bx]$) is called 
\emph{sub-Gaussian with parameter $\sigma^2$} if its Euclidean norm satisfies a
sub-Gaussian tail bound:
\begin{equation}
    \Pr\!\left(\|\bx - \mathbb{E}[\bx]\|_2 > \varepsilon \right)
    \le
    2 \exp\!\left(-\frac{\varepsilon^2}{2\sigma^2}\right),
    \quad \forall\, \varepsilon>0 .
\end{equation}

\paragraph{2-Wasserstein Distance.}
Let $\mu$ and $\nu$ be probability measures on $\mathbb{R}^d$ with finite second moments.
The 2-Wasserstein distance between $\mu$ and $\nu$ is defined as
\begin{align}
W_2(\mu,\nu)
= \left(
\inf_{\gamma \in \Gamma(\mu,\nu)}
\int_{\mathbb{R}^d \times \mathbb{R}^d} \|x - y\|_2^2 \, d\gamma(x,y)
\right)^{1/2},
\end{align}
where $\Gamma(\mu,\nu)$ denotes the set of all couplings of $\mu$ and $\nu$, i.e.,
\begin{align}
\Gamma(\mu,\nu)
= \Big\{ \gamma \in \mathcal{P}(\mathbb{R}^d \times \mathbb{R}^d)
:\; \gamma(A \times \mathbb{R}^d)=\mu(A),\;
\gamma(\mathbb{R}^d \times B)=\nu(B),\;
\forall A,B \subseteq \mathbb{R}^d \text{ measurable} \Big\}.
\end{align}

}

\subsection{Supporting Lemmas}
Tweedie's lemma establishes an important connection between the score of the marginal distribution and expectation of posterior when the likelihood function is gaussian. This allows the score function of the diffusion model to be used as a \emph{minimum-mean-square-error} (MMSE) denoiser.

\begin{lemma}[Tweedie's lemma \citep{robbins1992anempirical}]
    Let $p_0(\bx_0)$ be the prior distribution and then $\bx_t \sim \setN(\bx_0, \bm \Sigma)$ be observed with $\bm \Sigma$ known. Suppose $p_t(\bx_t)$ be the marginal distribution of $\bx_t$. Then, Tweedie's lemma computes the posterior expectation of $\bx_0$ given $\bx_t$ as
    \begin{equation}
        \Exp[\bx_0 | \bx_t] = \bx_t + \bm \Sigma \nabla \log p_t(\bx_t)
    \end{equation}
\end{lemma}

The lemmas related to $\theta$-averaged from \citet{combettes2014compositionsconvexcombinationsaveraged} are used to show fixed point convergence in \citet{ryu2019plugandplaymethodsprovablyconverge}. In the following, we extend all these lemmas to a  more general $\delta$-weakly $\theta$-averaged cases that will be used later to show our ball convergence.

\begin{lemma} 
    $T:\Rbb^d \to \Rbb^d$ be a function. Then, the following statements are equivalent:
    \begin{enumerate}
        \item [(a)] \label{lemma: weaklyaveraged def}$T$ is $\delta$-weakly $\theta$-averaged  for $\delta\ge0$ and $\theta \in (0,1)$.
        \item [(b)] \label{lemma: equivalentconditionofweaklyaveraged prop4} \(\norm{T(\bx)-T(\by)}_2^2 + (1-2\theta) \norm{\bx-\by}_2^2 -2(1-\theta) \vecdot{T(\bx)-T(\by), \bx-\by} \le \delta^2 \theta^2 \), for all $\bx,\by \in \Rbb^d$.
    \item [(c)]  $(1-1/\theta)I + (1/\theta)T$ is $\delta$-weakly nonexpansive.
    \item [(d)] \label{lemma: corresponding to prop2} \( \norm{T(\bx) - T(\by)}_2^2 \le \norm{\bx - \by}_2^2 - \frac{1-\theta}{\theta} \norm{(I - T)(\bx) - (I-T)(\by)}_2^2 + \delta^2 \theta \), for all $\bx,\by \in \Rbb^d$
    \end{enumerate}
   
\end{lemma}
    \begin{proof}
        \underline{Equivalence between (a) and (b)}: Provided $T$ is \emph{$\delta$-weakly $\theta$-averaged}, let's find the LHS - RHS in (b)
        \begin{align*}
             &(1-2\theta)  \norm{\bx - \by}_2^2  + \norm{T(\bx)-T(\by)}_2^2 -2 (1-\theta) \vecdot{ \bx-\by, T(\bx)-T(\by)} \\
              \le &(1-2\theta)  \norm{\bx - \by}_2^2 + \norm{\theta R(\bx) - \theta R(\by) +(1-\theta) (\bx-\by)}_2^2  \nonumber \\
              & \hspace{15mm}  -2 (1-\theta) \vecdot{ \bx - \by, \theta(R(\bx)-R(\by)) + (1-\theta)(\bx-\by) } \\
                = &(1-2\theta)  \norm{\bx - \by}_2^2 + \theta^2 \norm{ R(\bx) -  R(\by)}_2^2  +(1-\theta)^2 \norm{\bx-\by}_2^2 -  2(1-\theta)^2\norm{\bx-\by}_2^2  \nonumber \\
              & \hspace{15mm} + 2 \theta (1-\theta) \vecdot{ \bx-\by, R(\bx)-R(\by) } - 2 (1-\theta)\theta \vecdot{ \bx - \by, R(\bx)-R(\by)}   \\
              \le &(1-2\theta)  \norm{\bx - \by}_2^2 + \theta^2 \norm{ \bx -  \by}_2^2 + \theta^2 \delta^2  +(1-\theta)^2 \norm{\bx-\by}_2^2  -  2(1-\theta)^2\norm{\bx-\by}_2^2  \nonumber \\
              & \hspace{15mm}  + 2 \theta (1-\theta) \vecdot{ \bx-\by, R(\bx)-R(\by) } -2 (1-\theta)\theta \vecdot{ \bx - \by, R(\bx)-R(\by)}  \\
              =& (1-2\theta + \theta^2 +(1-\theta)^2 -2(1-\theta)^2) \norm{\bx -\by}^2_2 + \theta^2 \delta^2 \\
              =& \theta^2 \delta^2
        \end{align*}

         For another direction, let us suppose $T$ satisfies (a). Let $R = \frac{1}{\theta} \left ( T-(1-\theta) I \right)$ so that we have $T = \theta R + (1-\theta)I$. Now, we need to show that $\norm{R(\bx) -R(\by)}_2^2 \le \norm{\bx-\by}_2^2  + \delta^2$ for all $\bx,\by \in \Rbb^d$ i.e. \emph{$\delta$-weakly nonexpansive}.
        \begin{align*}
            &\norm{R(\bx) - R(\by)}_2^2 \\
            =&\frac{1}{\theta^2} \norm{T(\bx) - T(\by) - (1-\theta)(\bx-\by)}_2^2 \\
            =& \frac{1}{\theta^2} \left( \norm{T(\bx)-T(\by)}_2^2 + (1-\theta)^2 \norm{\bx-\by}_2^2 -2 (1-\theta) \vecdot{ T(\bx)-T(\by), \bx - \by } \right) \\
             =& \frac{1}{\theta^2} \left( \norm{T(\bx)-T(\by)}_2^2 + (1-2\theta) \norm{\bx-\by}_2^2 -2 (1-\theta) \vecdot{ T(\bx)-T(\by), \bx - \by } + \theta^2 \norm{\bx-\by}_2^2 \right) \\
             \le & \frac{1}{\theta^2} \left( \theta^2 \delta^2 + \theta^2 \norm{\bx - \by}_2^2 \right) \\
             =& \norm{\bx - \by}_2^2 + \delta^2
        \end{align*}
        where, the inequality is due to $T$ satisfying (b).

    \underline{Equivalence between (a) and (c)}: Note that $T \text{ is } \delta\text{-weakly }\theta\text{-averaged} \iff T=\theta R + (1-\theta) I$ with $R$ being $\delta$-weakly nonexpansive function.
    Now, we have 
    \begin{align*}
        (1-1/\theta) I + (1/\theta)T &= (1-1/\theta)I + 1/\theta \cdot \left( \theta R + (1-\theta)I \right)\\
        &= (1-1/\theta) I + R -  (1-1/\theta)I \\
        &= R
    \end{align*}
    Hence, $T$ being $\delta$-weakly $\theta$-average is equivalent to $(1-1/\theta)I + (1/\theta)T$ being $\delta$-weakly nonexpansive.

    \underline{Equivalence between (a) and (d)}:  From equivalence between (a) and (b), we have
 \begin{align*}
     &\norm{T(\bx)-T(\by)}_2^2 + (1-2\theta) \norm{\bx-\by}_2^2 - 2(1-\theta) \vecdot{ T(\bx)-T(\by), \bx-\by } \le \delta^2 \theta^2 \\
     \iff &\norm{T(\bx)-T(\by)}_2^2 + (1-2\theta) \norm{\bx-\by}_2^2 - \nonumber\\
    & \hspace{15mm}(1-\theta) \left( \norm{T(\bx)-T(\by)}_2^2 + \norm{\bx -\by}_2^2 - \norm{(T-I)(\bx) - (T-I)(\by)}_2^2 \right) \le \delta^2 \theta^2 \\
    \iff &\theta \norm{T(\bx)-T(\by)}_2^2 -\theta \norm{\bx-\by}_2^2 \le \delta^2 \theta^2 - (1-\theta) \norm{(T-I)(\bx) - (T-I)(\by)}_2^2 \\
     \iff  &\norm{T(\bx)-T(\by)}_2^2 \le \norm{\bx-\by}_2^2 - \frac{1-\theta}{\theta} \norm{(I-T)(\bx) - (I-T)(\by)}_2^2 + \delta^2 \theta
 \end{align*}
\end{proof}

\begin{lemma}[Concatenation of $\delta$-weakly $\theta$-averaged functions]\label{lemma: averagedness of composition of functions}
        Assume $T_1: \Rbb^d \to \Rbb^d$ and $T_2: \Rbb^d \to \Rbb^d$ are $\delta_1$-weakly $\theta_1$-averaged and $\delta_2$-weakly $\theta_2$-averaged respectively. Then, $T_1 \circ T_2$ is $\delta$-weakly $\theta$-averaged, with $\theta = \frac{\theta_1 + \theta_2 - 2\theta_1 \theta_2}{1-\theta_1 \theta_2}$, and $\delta^2 = \frac{1}{\theta} (\delta_1^2 \theta_1 + \delta_2^2 \theta_2)$.
    \end{lemma}
    \begin{proof}
    Here, we follow the proof structure of \citet{combettes2014compositionsconvexcombinationsaveraged}. Since $\theta_1 (1-\theta_2) \le (1-\theta_2)$, we have $\theta_1 +  \theta_2 \le 1+ \theta_1 \theta_2$, and therefore, $\theta=\frac{\theta_1 + \theta_2 - 2\theta_1 \theta_2}{1- \theta_1 \theta_2} \in (0,1)$, and let $\delta^2 = \frac{\delta_1^2 \theta_1 + \delta_2^2 \theta_2}{\theta}$

    Now, from Lemma \ref{lemma: corresponding to prop2}, for $i\in\{1,2\}$, we have,
    \begin{equation}
        \norm{T_i(\bx) - T_i(\by)}_2^2 \le \norm{\bx - \by}_2^2 - \frac{1-\theta_i}{\theta_i} \norm{(I-T_i)(\bx) - (I-T_i)(\by)}_2^2 + \delta_i^2\theta_i
    \end{equation}

    Then, let us evaluate the composition function using this property.
    \begin{align*}
        &\norm{T_1 \circ T_2 (\bx) - T_1 \circ T_2 (\by)}_2^2 \\
        \le & \norm{T_2(\bx)-T_2(\by)}_2^2 - \frac{1-\theta_1}{\theta_1} \norm{(I-T_1)(T_2(\bx)) - (I-T_1)(T_2(\by)}_2^2\\
        \le & \norm{\bx - \by}_2^2 - \frac{1-\theta_2}{\theta_2} \norm{(I-T_2)(\bx) - (I-T_2)(\by)}_2^2 + \delta_2^2\theta_2  \nonumber\\
        & \hspace{15mm} -\frac{1-\theta_1}{\theta_1} \norm{(I-T_1)(T_2(\bx)) - (I-T_1)(T_2(\by)}_2^2 + \delta_1^2\theta_1 
    \end{align*}

    From  \citet{bauschke2017correction}[Corollary 2.15], we have, for $\alpha \in \Rbb$,
   \begin{align*}
       &\norm{\alpha \bu + (1-\alpha)\bv}_2^2 + \alpha (1-\alpha) \norm{\bu - \bv}_2^2 = \alpha \norm{\bu}_2^2 + (1-\alpha)\norm{\bv}_2^2 \\
       \implies & \alpha (1-\alpha) \norm{\bu + \bv}_2^2 \le \alpha \norm{\bu}_2^2 + (1-\alpha) \norm{\bv}_2^2
   \end{align*} 

   Now, let $\bu= (I-T_2)(\bx) - (I-T_2)(\by)$, $\bv= (I-T_1)(T_2(\bx))- (I-T_1)(T_2(\by))$, $a=\frac{1-\theta_2}{\theta_2}$, and $b=\frac{1-\theta_1}{\theta_1}$.
   \begin{align*}
        &\norm{T_1 \circ T_2 (\bx) - T_1 \circ T_2 (\by)}_2^2\\
        \le & \norm{\bx - \by}_2^2 - a \norm{\bu}_2^2 - b\norm{\bv}_2^2 + \delta_1^2 \theta_1 + \delta_2^2 \theta_2\\
        = & \norm{\bx - \by}_2^2 - (a+b) \left( \frac{a}{a+b} \norm{\bu}_2^2 + \frac{b}{a+b}\norm{\bv}_2^2 \right) + \delta_1^2 \theta_1 + \delta_2^2 \theta_2\\
        = & \norm{\bx - \by}_2^2 - (a+b) \left( \frac{a}{a+b} \norm{\bu}_2^2 + \left(1- \frac{a}{a+b}\right)\norm{\bv}_2^2 \right) + \delta_1^2 \theta_1 + \delta_2^2 \theta_2\\
   \end{align*}
   Using the above results, we get,
    \begin{align*}
        &\norm{T_1 \circ T_2 (\bx) - T_1 \circ T_2 (\by)}_2^2\\
        \le & \norm{\bx - \by}_2^2 - (a+b) \frac{ab}{(a+b)^2} \norm{\bu + \bv}_2^2  + \delta_1^2 \theta_1 + \delta_2^2 \theta_2 \\
        = & \norm{\bx - \by}_2^2 -  \frac{ab}{(a+b)} \norm{(I-T_1 \circ T_2)(\bx) - (I-T_1 \circ T_2)(\by)}_2^2  + \delta_1^2 \theta_1 + \delta_2^2 \theta_2
    \end{align*}

    Let $\theta = \frac{\theta_1 + \theta_2 -2\theta_1 \theta_2}{1-\theta_1 \theta_2}$. Then, we can see that $\frac{ab}{a+b}= \frac{1-\theta}{\theta}$.
    \begin{align}
        \norm{T_1 \circ T_2 (\bx) - T_1 \circ T_2 (\by)}_2^2 &\le  \norm{\bx - \by}_2^2 -  \frac{1-\theta}{\theta} \norm{(I-T_1 \circ T_2)(\bx) - (I-T_1 \circ T_2)(\by)}_2^2 \nonumber \\
        &\hspace{15mm}+ \delta^2 \theta
    \end{align}
    where, $\delta^2 \theta = \delta_1^2 \theta_1 + \delta_2^2 \theta_2$.
    This implies  that $T_1 \circ T_2$ is $\delta$-weakly $\theta$-averaged with
    \begin{align}
        \theta = \frac{\theta_1 + \theta_2 -2\theta_1 \theta_2}{1-\theta_1 \theta_2}, \; \delta^2 &= \frac{1}{\theta} (\delta_1^2 \theta_1 + \delta_2^2 \theta_2)
    \end{align}
\end{proof}

\begin{lemma}[Proposition 5.4  of \citet{giselsson2017tightgloballinearconvergence}] \label{lemma: averagedness of proximal}
Assume $\ell$ is $\mu$-strongly convex, closed, and proper. Then, $-(2\text{\rm Prox}_{\frac{1}{\rho} \ell} - I)$ is $ \frac{\rho}{\rho+\mu}$-averaged.
\end{lemma}

\begin{lemma} [\citet{pardo2018statistical}] \label{lemma: KL divergence between two gaussians}
The \emph{$\kl$ divergence} between two gaussian distributions $q_1= \setN(\bmu_1, \bSigma_1 )$ and $q_2=\setN(\bmu_2, \bSigma_2)$ in $\Rbb^d$ space is given by
    \begin{equation}
    \kl(q_1|| q_2) = \frac{1}{2}\left[ \log \frac{|\bSigma_2|}{|\bSigma_1|} -d + \tr(\bSigma_2^{-1} \bSigma_1) + (\bmu_2 - \bmu_1)^T \bSigma_2^{-1} (\bmu_2 - \bmu_1) \right]
    \end{equation}
    where $|\cdot|$ denotes the determinant, and $\tr$ denotes the trace of the matrix.
\end{lemma}

\section{Influence of AC-step on bringing close to $\{ \setM_{\sigma(t)} \}_{t=0}^T$} \label{sec: discussion on influence of AC-step}
 Given a noisy image $\tbzk$ at each iteration $k$, the denoising aims to recover the underlying clean image $\bznk\sim p_{0}(\bz)$ such that $\tbzk = \bznk + \bsk$, where $\bsk$ is the noise contained in $\tbzk$. The AC-step aims to bring $\bzack$ closer to the noisy distribution $\setM_{\sigmak}$ on which the $\bs_{\bm \theta}(\cdot, \sigmak)$ was trained on. Lemma \ref{lemma: approximate correction kl divergence} shows that the AC-step tries to match with the distribution induced by the forward diffusion process.
 \begin{lemma} \label{lemma: approximate correction kl divergence}
      The KL divergence between the target distribution $p(\bz_{\sigmak}|\bznk)$ for correction steps and the distribution $p(\bzack|\tbzk)$  induced by the approximate correction step in Algorithm \ref{algo:acdc} is given by
      \begin{equation}
          KL(p(\bz_{\sigmak}|\bznk) || p(\bzack|\tbzk))= \frac{1}{2\left(\sigmak\right)^2 } \norm{\bsk}_2^2
      \end{equation}
  \end{lemma}
  
\begin{proof}
\begin{align}
    &\textbf{Target distribution: } p\left(\bz_{\sigmak} | \bznk \right) = \setN \left(\bznk, \left(\sigmak\right)^2 \bI \right)\\
    &\textbf{AC induced distribution: } p\left(\bzack|\tbzk\right) = \setN \left(\tbzk, \left(\sigmak \right)^2 \bI \right)
\end{align}
The KL divergence between these two distribution can be computed in closed form using Lemma \ref{lemma: KL divergence between two gaussians}.
\begin{align*}
    \kl(q_1||q_2)&=\frac{1}{2}\left(\log 1 - d + \tr(\bI) + (\tbzk - \bznk)^T \left(\sigmak\right)^{-2} \bI (\tbzk - \bznk) \right)\\
                    &= \frac{1}{2} \left(0-d+d + \left(\sigmak\right)^{-2} \norm{\tbzk - \bznk}_2^2\right) \\
                    &= \frac{1}{2\left(\sigmak\right)^2 } \norm{\bsk}_2^2
\end{align*}
where, $\bsk = \tbzk-\bznk$.
\end{proof}

 Lemma~\ref{lemma: approximate correction kl divergence} shows that KL-gap of our approximate AC update; as long as $\sigmak$ is sufficiently large, the two distributions remain close. \citet{alkhouri2023robustphysicsbaseddeepmri}[Theorem 1] showed a result with a similar flavor. Larger noise $\sigmak$ makes the posterior nearly indistinguishable, but it also washes out fine structural details originally present (low \emph{Signal-to-Noise Ratio} with larger $\sigmak$). Existing works often use annealed scheduling $\sigmak\downarrow0$ \citep{zhu2023denoising, renaud2024plug, wang2024dmplug} to preserve image details, implicitly assuming $\norm{\bsk}_2^2$ decays at least as fast as $(\sigmak)^2$. With just the use of annealing $\sigmak$ schedule, it is not sufficient to ensure that $\bzack$ lands in a desired manifold in each ADMM iteration. To bridge this gap, we propose to use DC-step in addition to the widely used annealing $\sigmak$ schedule that explicitly corrects this gap.

\section{Proof of Theorem \ref{thm:convergenceofpnpadmm}} \label{proof: theorem 1 on general weakly non-expansive}
The proof involves showing that the each iteration of ADMM-PnP is weakly non-expansive when the denoiser satisfies Assumption~\ref{assumpt: delta weakly residual}. This weakly nonexpansiveness of each step leads to ball convergence of the algorithm.

Recall that the subproblems at $k$th iteration of ADMM-PnP is given by:
      \begin{subequations} \label{eqns: pnp admm steps}
          \begin{align}
                  \bxkk &= \arg\min_{\bx} \frac{1}{\rho}\ell( \by || \setA(\bx)) + \frac{1}{2} \norm{\bx - \bzk  + \buk}_2^2 \nonumber \\
                        &= \text{Prox}_{\frac{1}{\rho} \ell}(\bzk - \buk)\\
                  \bzkk &= \arg\min_{\bz} \frac{\gamma}{\rho} h(\bz) + \frac{1}{2}\norm{\bxkk - \bz + \buk}_2^2 \nonumber \\
                        &= \text{Prox}_{\frac{\gamma}{\rho}h}(\bxkk + \buk)\nonumber \\
                        &= D_{\sigmak}(\bxkk + \buk)\\
                  \bukk &= \buk + (\bxkk - \bzkk)
          \end{align}
      \end{subequations}

      \begin{lemma}[\citet{ryu2019plugandplaymethodsprovablyconverge}] \label{lemma:pnp admm operator}
      The steps of ADMM-PnP in \eqref{eqns: pnp admm steps} can be expressed as $\bvkk = T(\bvk)$ with $\bvk = \bzk -\buk$ and
      \begin{equation}
          T = \frac{1}{2} I + \frac{1}{2}(2 D_{\sigmak} -I)(2 \text{\rm Prox}_{\frac{1}{\rho} \ell}-I)
      \end{equation}          
      \end{lemma}

\begin{lemma} \label{lemma: averagedness of required denoiser function}
     $D_\sigma: \Rbb^d \to \Rbb^d$ satisfies Assumption \ref{assumpt: delta weakly residual} if and only if $\frac{1}{1+2\epsilon}(2D_{\sigmak}-I)$ is $\Delta$-weakly $\theta$-averaged with $\theta= \frac{2\epsilon}{1+2\epsilon}$ and $\Delta^2 = 4 \delta^2 \frac{(1-\theta)^2}{\theta^2}$.
\end{lemma}
\begin{proof}
    We follow the similar proof structure as in \citet{ryu2019plugandplaymethodsprovablyconverge}. Let $\theta = \frac{2 \epsilon}{1+2\epsilon}$ which implies $\epsilon = \frac{\theta}{2(1-\theta)}$. Here, we can clearly see that $\theta \in [0,1)$. Let us define $G= \frac{1}{1+2\epsilon}(2D_{\sigmak}-I)$ which implies $D_{\sigmak} = \frac{1}{2(1-\theta)}G + \frac{1}{2}I$. Then,
    \begin{align*}
         & \norm{(D_{\sigmak}-I)(\bx) - (D_{\sigmak}-I)(\by) }^2 - \epsilon^2 \norm{\bx-\by}_2^2 \\
        = & \norm{\left( D_{\sigmak}(\bx) - D_{\sigmak}(\by) \right) - \left( \bx - \by \right) }_2^2 - \epsilon^2 \norm{\bx-\by}_2^2\\
        = & \norm{\left(D_{\sigmak}-\frac{1}{2} I \right)(\bx) - \left(D_{\sigmak}-\frac{1}{2} I\right)(\by)}_2^2 + \frac{1}{4}\norm{\bx - \by}_2^2 - \frac{\theta^2}{4(1-\theta)^2} \norm{\bx-\by}_2^2 \nonumber\\
        & \hspace{15mm}-2  \vecdot{  \left(D_{\sigmak}-\frac{1}{2} I \right)(\bx) - \left(D_{\sigmak}-\frac{1}{2}\right)(\by), \frac{1}{2}(\bx - \by)  }  \\
        = & \frac{1}{4(1-\theta)^2} \norm{G(\bx)-G(\by)}_2^2  + \frac{1}{4} \left(1-\frac{\theta^2}{(1-\theta)^2} \right) \norm{\bx - \by}_2^2 \nonumber\\
        & \hspace{15mm} - \frac{1}{2(1-\theta)} \vecdot{ G(\bx)-G(\by), \bx-\by }\\
        = & \frac{1}{4(1-\theta)^2} \left( \norm{G(\bx)-G(\by)}_2^2 - 2(1-\theta) \vecdot{ G(\bx)-G(\by), \bx-\by } + (1-2\theta) \norm{\bx - \by}_2^2 \right)
    \end{align*}

    Now,
    \begin{align*}
        &\frac{1}{4(1-\theta)^2} \left( \norm{G(\bx)-G(\by)}_2^2 - 2(1-\theta) \vecdot{ G(\bx)-G(\by), \bx-\by } + (1-2\theta) \norm{\bx - \by}_2^2 \right) \le \delta^2 \\
        \Leftrightarrow &  \norm{G(\bx)-G(\by)}_2^2 - 2(1-\theta) \vecdot{ G(\bx)-G(\by), \bx-\by } + (1-2\theta) \norm{\bx - \by}_2^2 \le 4\delta^2 (1-\theta)^2 \\
        \Leftrightarrow &  \norm{G(\bx)-G(\by)}_2^2 - 2(1-\theta) \vecdot{ G(\bx)-G(\by), \bx-\by } + (1-2\theta) \norm{\bx - \by}_2^2 \le \Delta^2 \theta^2 
    \end{align*}
    where, $\Delta^2 = 4\delta^2 \frac{(1-\theta)^2}{\theta^2}$.
    From Lemma \ref{lemma: equivalentconditionofweaklyaveraged prop4}, this is equivalent to $G$ being $\Delta$-weakly $\theta$-averaged.
\end{proof}

\subsection{Proof of the theorem}
We follow the procedures in \citet{ryu2019plugandplaymethodsprovablyconverge} and expand the results in the $\delta$-weakly expansive denoisers. We show that each iteration of PnP ADMM is also weakly nonexpansive when the denoiser satisfies Assumption~\ref{assumpt: delta weakly residual}.

\begin{proof}
 From Assumption \ref{assumpt: delta weakly residual}.
 \begin{equation}
     \norm{(D_{\sigmak}-I)(\bx) - (D_{\sigmak}-I)(\by) }_2^2 \le \epsilon^2 \norm{\bx-\by}_2^2 + \delta^2
 \end{equation}
 
    From Lemma \ref{lemma: averagedness of proximal}, we have
 \(
        -(2 \text{Prox}_{\frac{1}{\rho}\ell}-I)
\)
    is $\frac{\rho}{\rho + \mu}$-averaged.

    Then, from Lemma \ref{lemma: averagedness of required denoiser function}, we have 
    \begin{equation}
        \frac{1}{1+2\epsilon} \left( 2D_{\sigmak} - I \right)
    \end{equation}
    is $\delta_1$-weakly $\theta$-averaged with $\theta=\frac{2\epsilon}{1+2\epsilon}$ and $\delta_1^2= 4\delta^2 \frac{(1-\theta)^2}{\theta^2}$.

    By Lemma \ref{lemma: averagedness of composition of functions}, it implies
    \begin{equation}
        -\frac{1}{1+2\epsilon}(2D_{\sigmak}-I)(2 \text{Prox}_{\frac{1}{\rho}\ell} - I)
    \end{equation}
    is  $\delta_{\circ}$-weakly $\theta_{\circ}$-averaged with $\theta_{\circ}=\frac{\rho+2\mu\epsilon}{\rho + \mu + 2\mu \epsilon}$ and $\delta_{\circ}^2=\frac{1}{\theta_{\circ}} \cdot \frac{4 \delta^2}{2\epsilon(1+2\epsilon)}$.

    Now, using the definition of $\delta_{\circ}$-weakly $\theta_{\circ}$-averagedness, we have
    \begin{align*}
        (2D_{\sigmak}-I)(2 \text{Prox}_{\frac{1}{\rho}\ell} - I) &= -(1+2\epsilon) \left( (1-\theta_{\circ}) I + \theta_{\circ} R \right) \\
        &= -(1+2\epsilon) \left( \frac{\mu}{\rho + \mu + 2\mu \epsilon} I + \frac{\rho + 2 \mu \epsilon}{\rho + \mu + 2\mu \epsilon} R \right)
    \end{align*}
    where, $R$ is a certain $\delta_{\circ}$-weakly nonexpansive function.

    Plugging this result into ADMM-PnP operator (Lemma \ref{lemma:pnp admm operator}), we get
    \begin{align*}
         T &= \frac{1}{2}I + \frac{1}{2} (2 D_{\sigmak} -I) (2 \text{Prox}_{\frac{1}{\rho}\ell} - I)\\
         &= \frac{1}{2}I - \frac{1}{2} (1+2\epsilon) \left( \frac{\mu}{\rho + \mu + 2\mu \epsilon} I + \frac{\rho + 2 \mu \epsilon}{\rho + \mu + 2\mu \epsilon} R \right) \\
         &= \underbrace{\frac{\rho}{2(\rho + \mu + 2\mu\epsilon)}}_{a}I - \underbrace{\frac{(1+2\epsilon)(\rho + 2\mu\epsilon)}{2(\rho + \mu + 2\mu\epsilon)}}_{b} R
    \end{align*}
    where, clearly $a>0$ and $b>0$.

    Now,
    \begin{align}
        \norm{T(\bx)-T(\by)}_2^2 = a^2 \norm{\bx - \by}_2^2 + b^2 \norm{R(\bx)-R(\by)}_2^2 -2 \langle a(\bx-\by), b(R(\bx)-R(\by)) \rangle
    \end{align}
    From Young's inequality,  for any $\gamma > 0$, we have
    \begin{equation}
        \langle a(\bx-\by), b(R(\bx)-R(\by)) \rangle \le \frac{1}{2 \gamma} a^2 \norm{\bx-\by}_2^2 + \frac{\gamma b^2}{2} \norm{R(\bx)-R(\by)}_2^2
    \end{equation}
    Plugging this, we get,
    \begin{align}
        \norm{T(\bx)-T(\by)}_2^2 &\le a^2 \left( 1+ \frac{1}{\gamma} \right) \norm{\bx - \by}_2^2 + b^2 (1+\gamma) \norm{R(\bx)-R(\by)}_2^2 \\
        & \le \left(  a^2 \left( 1+ \frac{1}{\gamma} \right) +b^2 (1+\gamma)\right) \norm{\bx - \by}_2^2 + b^2 (1+\gamma) \delta_{\circ}^2
    \end{align}
    where, the second inequality is due to $\delta_{\circ}$-weak nonexpansiveness of $R$.

    Note, that this holds for any $\gamma>0$. When $\gamma = \frac{a}{b}$, we have
    \begin{equation}
        \left(  a^2 \left( 1+ \frac{1}{\gamma} \right) +b^2 (1+\gamma)\right) = (a+b)^2
    \end{equation}

    \begin{align}
          &\norm{T(\bx)-T(\by)}_2^2 \nonumber \\
          \le & (a+b)^2 \norm{\bx - \by}_2^2 + b^2 \left( 1+ \frac{a}{b} \right) \delta_{\circ}^2 \nonumber \\
           = &\underbrace{ \left( \frac{\rho + \rho \epsilon + \mu \epsilon  + 2\mu \epsilon^2}{\rho + \mu + 2\mu\epsilon} \right)^2}_{\epsilon_T^2} \norm{\bx - \by}_2^2 + \underbrace{\frac{ (\rho + \rho \epsilon+\mu\epsilon+2\mu \epsilon^2)\delta^2}{\epsilon(\rho + \mu + 2\mu\epsilon)}}_{\delta_T^2}
    \end{align}
    Hence, we have 
    \begin{equation}
        \norm{T(\bx)-T(\by)}_2^2 \le \epsilon_T^2 \norm{\bx - \by}_2^2 + \delta_T^2
    \end{equation}
    This shows that when $\epsilon_T \le 1$, then, $\exists_{\bv^{*}} \in \Rbb^{d}$, and $K > 0$ such that $\forall k \ge K$ the following holds:
    \begin{align}
       &\norm{\bvk - \bv^*}_2^2 \le \epsilon_T^{2k} \norm{\bv^{(0)} - { \bv^*}}_2^2 + \dfrac{\delta_T^2}{1-\epsilon_T^2} \nonumber \\
       \implies & \lim_{k\to\infty} \norm{\bvk - \bv^*} \le \dfrac{\delta_T}{\sqrt{1-\epsilon_T^2}}
    \end{align}

    Hence with this we have the sequence $\{ \bvk = \bzk -\buk \}_{k\in \Nbbp}$ converges within a ball of radius $\frac{\delta_T}{\sqrt{1-\epsilon_T^2}}$.
    Since, $-(2\text{Prox}_{\frac{1}{\rho}\ell} - I)$ is $\frac{\rho}{\rho + \mu}$-averaged, this implies
    \begin{align}
       \Prox_{\frac{1}{\rho}\ell} = \frac{1}{2} \frac{\rho}{\rho + \mu} (I-R)
    \end{align}
    for some nonexpansive function $R$. \
    With this, we have
    \begin{align}
        &\lim_{k\to\infty}\norm{ \Prox_{\frac{1}{\rho}\ell}(\bv^{(k)}) -\Prox_{\frac{1}{\rho}\ell}(\bv^*)  }_2^2 \le \left( \frac{\rho}{\rho+\mu} \right)^2 \norm{\bv^{(k)} -\bv^*}_2^2 \nonumber \\
        \implies &\lim_{k \to \infty}\norm{\bxk - \bx^*}_2 \le \frac{\rho}{\rho+\mu} \frac{\delta_T}{\sqrt{1-\epsilon_T^2}}
    \end{align}
    where $\bx^*=\Prox_{\frac{1}{\rho}\ell} (\bv^*)$.

    With these results, we know there exists $\bu^*$ such that
    \begin{equation}
        \lim_{k\to\infty} \norm{\bu^k - \bu^{*}}_2 \le \left(1+ \frac{\rho}{\rho+\mu}\right) \frac{\delta_T}{\sqrt{1-\epsilon_T^2}}
    \end{equation}
\end{proof}

\section{Proof of Theorem \ref{theorem: contractive overall denoising for tweedie's lemma}} \label{proof: theorem contractive overall denoising for tweedie's lemma}
Here, we show that our 3-step AC-DC denoiser satisfies Assumption~\ref{assumpt: delta weakly residual} for constants $\epsilon$ and $\delta$. In the following, we first show that each step satisfies the weakly nonexpansive assumption. Therefore, the concatenation of these 3 steps meets Assumption~\ref{assumpt: delta weakly residual}.

\begin{lemma} \label{lemma: Mt smoothness of pt}
     Assume the \emph{Variance Exploding (VE) scheduling} \citep{karras2022elucidatingdesignspacediffusionbased} is used in the diffusion model. Given  that the  log-density $\log p_0$ (i.e. $\log \pdata$) is $M$-smooth (Assumption \ref{assumpt: smoothness of logpdata}), the intermediate noisy  log-densities  $ \{\log p_{t}\}$   are $M_t$-smooth for $t\in [0,T]$ i.e. $
          \norm{\nabla \log p_t (\bx) - \nabla \log p_t(\by)}_2 \le M_t \norm{\bx -\by}_2$
      for all $\bx, \by \in \setX$.
    For $t$ such that  $\sigma^2(t) < 1/M $, the smoothness constant $M_t$ can be upper bounded as 
    \begin{equation}
        M_t \le \dfrac{M}{1+M\sigma^2(t)} \le M
    \end{equation}
\end{lemma}
\begin{proof}
    From Tweedie's lemma, we have,
    \begin{align}
        \implies &\nabla_{\bx_t} \log p_t(\bx_t) = -\frac{1}{\sigma^2(t)} \left( \bx_t - \Exp[\bx_0 | \bx_t] \right) \nonumber\\
        \implies &\nabla_{\bx_t}^2 \log p_t(\bx_t) = -\frac{1}{\sigma^2(t)} (\bI - \nabla_{\bx_t} \Exp[\bx_0|\bx_t]) \label{eqn: M_t base hessian equation}
    \end{align}
    
    Now, let us evaluate the Jacobian $\nabla_{\bx_t} \Exp[\bx_0|\bx_t]$,
    \begin{align}
       \nabla_{\bx_t} \Exp[\bx_0 | \bx_t] & = \nabla_{\bx_t} \int_{\bx_0 \in \setX} \bx_0 p(\bx_0| \bx_t) d\bx_0 \nonumber\\
        &= \int_{\bx_0 \in \setX} \bx_0 \left(\nabla_{\bx_t} \frac{p(\bx_t|\bx_0)} {p_t(\bx_t)} \right) p_0(\bx_0) d\bx_0 \nonumber\\
        &= \int_{\bx_0 \in \setX} \bx_0 \left( \frac{1}{p_t(\bx_t) } \nabla_{\bx_t}p(\bx_t|\bx_0) - p(\bx_t | \bx_0) \frac{1}{p_t^2(\bx_t)} \nabla_{\bx_t} p_t(\bx_t) \right) p_0(\bx_0) d\bx_0  \label{eqn: base nabla E[x0|xt]}
    \end{align}
 Given $\bx_t | \bx_0 \sim \setN(\bx_0 , \sigma(t)^2 \bI)$ and $p(\bx_t)=\int_{\bx_0 \in \setX} p(\bx_0)p(\bx_t | \bx_0) d\bx_0$, we can compute their gradient (similar to \citet{peng2024improvingdiffusionmodelsinverse}) as:
    \begin{align}
        \nabla_{\bx_t} p(\bx_t | \bx_0) &= -\frac{1}{\sigma^2(t)} p(\bx_t|\bx_0)(\bx_t - \bx_0) \nonumber   \\
        \nabla_{\bx_t} p(\bx_t) &= \int_{\bx_0 \in \setX} p_0(\bx_0) \nabla_{\bx_t}p(\bx_t|\bx_0) d\bx_0 \nonumber \\
    &= \int_{\bx_0 \in \setX} p_0(\bx_0) \left( -\frac{1}{\sigma^2(t)}p(\bx_t|\bx_0)(\bx_t - \bx_0)  \right) d\bx_0 \nonumber \\
    &= -\frac{1}{\sigma^2(t)} p_t(\bx_t)\Exp[\bx_t - \bx_0 | \bx_t]
    \end{align}

    Plugging these results in \eqref{eqn: base nabla E[x0|xt]} and using integration by parts,

    \begin{align}
     \nabla_{\bx_t} \Exp[\bx_0 | \bx_t]  =& \frac{1}{p(\bx_t)} \int_{\bx_0 \in \setX} \bx_0 \nabla_{\bx_t} p(\bx_t | \bx_0) p(\bx_0) d\bx_0 - \frac{\nabla_{\bx_t} p(\bx_t)}{p({\bx_t}) } \Exp[\bx_0|\bx_t]
\end{align} 

Substituting $\nabla_{\bx_t}p(\bx_t | \bx_0)=-\nicefrac{1}{(\sigma^2(t))} (\bx_t -\bx_0) p(\bx_t|\bx_0)$, we have
\begin{align}
    \nabla_{\bx_t} \Exp[\bx_0 | \bx_t] =& -\frac{1}{\sigma^2(t) p(\bx_t)} \int_{\bx_0 \in \setX} \bx_0 (\bx_t -\bx_0) p(\bx_t|\bx_0) p(\bx_0) d\bx_0 - \frac{\nabla_{\bx_t} p(\bx_t)}{p({\bx_t}) } \Exp[\bx_0|\bx_t] \nonumber\\
    &= \frac{-1}{\sigma^2(t)} \Exp[\bx_0(\bx_t-\bx_0) |\bx_t] - \frac{\nabla_{\bx_t} p(\bx_t)}{p({\bx_t}) } \Exp[\bx_0|\bx_t]
\end{align}
Substituting $\nabla_{\bx_t} p(\bx_t) = \int_{\bx_0 \in \setX} p(\bx_0) \nabla_{\bx_t} p(\bx_t | \bx_0) d\bx_0=-\nicefrac{p_t(\bx_t)}{\sigma^2(t)}\Exp[\bx_t-\bx_0|\bx_t] $, we have
\begin{align}
    \nabla_{\bx_t} \Exp[\bx_0 | \bx_t] &= \frac{-1}{\sigma^2(t)} \Exp[\bx_0(\bx_t-\bx_0) |\bx_t] +\frac{1}{\sigma^2(t)} \Exp[\bx_0|\bx_t] \Exp[\bx_t - \bx_0 \bx_t] \nonumber \\
    &= -\frac{1}{\sigma^2(t)} \text{Cov}(\bx_0, \bx_t - \bx_0) \nonumber\\
    &= \frac{1}{\sigma^2(t)} \text{Cov}(\bx_0|\bx_t) \label{eqn: nabla E[x0|xt] interms of Var(x0|xt)}
\end{align}

 Note that this result in \eqref{eqn: nabla E[x0|xt] interms of Var(x0|xt)} is similar to \citet{dytso2021generalderivativeidentityconditional}[Proposition 1], and has also been derived in \citet{hatsell1971somegeometric, palomar2006gradient}.

From Assumption \ref{assumpt: smoothness of logpdata}, we have
\begin{equation}
        -M \bI \preceq \nabla_{\bx_0}^2 \log p_0(\bx_0) \preceq M \bI, \; \forall \bx_0 \in \setX
\end{equation}
where, $M>0$ is a constant.

Then, let's analyze the hessian of the log of posterior distribution $p(\bx_0 | \bx_t)$,
\begin{align}
    & \nabla_{\bx_0}^2 \log p(\bx_0 | \bx_t) = \nabla_{\bx_0}^2 \log p_0(\bx_0) + \nabla_{\bx_0}^2 \log p(\bx_t | \bx_0) \\
    \implies &  -M \bI + \frac{1}{\sigma^2(t)} \bI \preceq \nabla_{\bx_0}^2 \log p(\bx_0 | \bx_t)  \preceq M \bI + \frac{1}{\sigma^2(t)} \bI
\end{align}

When $M < \dfrac{1}{\sigma^2(t)}$, then the distribution $\log p(\bx_0 | \bx_t)$ is strongly  log-concave. In this case, the covariance of distribution $p(\bx_0|\bx_t)$ can be bounded \citep{brascamp1976onextensions} as 

\begin{align}
    \left(  M \bI + \frac{1}{\sigma^2(t)} \bI \right)^{-1}  \preceq \text{Cov}(\bx_0 | \bx_t) \preceq \left(-M \bI + \frac{1}{\sigma^2(t)} \bI \right)^{-1} \label{eqn: variance bound}
\end{align}

Combining result from equations \eqref{eqn: M_t base hessian equation}, \eqref{eqn: nabla E[x0|xt] interms of Var(x0|xt)} and \eqref{eqn: variance bound}, we get  
\begin{align}
    \nabla_{\bx_t}^2 \log p_t(\bx_t) &=-\frac{1}{\sigma^2(t)} \left(\bI -  \nabla_{\bx_t} \Exp[\bx_0|\bx_t] \right) \nonumber \\
    &=-\frac{1}{\sigma^2(t)} \left(\bI - \frac{1}{\sigma^2(t)} \text{Cov}(\bx_0|\bx_t) \bI \right) \nonumber \\
     \norm{\nabla_{\bx_t}^2 \log q_t(\bx_t)}_2 &\le \frac{1}{\sigma^2(t)} \left|\left(1 - \frac{1}{\sigma^2(t)} \cdot \frac{\sigma^2(t)}{M\sigma^2(t) + 1}  \right) \right| \nonumber \\
    &= \frac{1}{\sigma^2(t)}   \cdot \frac{M\sigma^2(t)}{1+M(\sigma^2(t)) }  \nonumber \\
    &=  \frac{M}{1+M\sigma^2(t)} 
\end{align}

    Hence, the smoothness constant of $\log q(\bx_t)$ is upper bounded as $M_t \le \dfrac{M}{1+M\sigma^2(t)}$ i.e. $M_t \le M$.

\end{proof}

\begin{lemma}\label{lemma: contractive residual of approximate correction}
    Let $\Hack: \tbzk \mapsto \bzack$ denote the function corresponding to approximate correction in Algorithm~\ref{algo:acdc}. Then, with probability at least $1-e^{-\nu_k}$, the following holds for any $\bx,\by \in \setX$
    \begin{equation}
        \norm{(\Hack-I)(\bx)- (\Hack-I)(\by)}_2^2 \le (\delack)^2 + (\epack)^2 \norm{\bx - \by}_2^2
    \end{equation}
    where, $\delack^2=2 (\sigmak )^2 (d + 2\sqrt{d\nu_k}+2\nu_k)$, and $(\epack)^2=0$.
\end{lemma}
\begin{proof}
 For any $\bx,\by \in \setX$, we have the residuals $\Rack(\bx)=(\Hack - I)(\bx) = \sigmak \bn_1$ and $\Rack(\by)=(\Hack - I)(\by) = \sigmak \bn_2$ where, $\bn_1, \bn_2\overset{\text{i.i.d.}}\sim \setN(\bm 0, \bI)$.

 Then, we can bound the norm of difference of these two residuals as
 \begin{align}
     \norm{\Rack(\bx)-\Rack(\by)}_2^2 &= \norm{\sigmak \bn_1 - \sigmak \bn_2}_2^2 \nonumber\\
     &= (\sigmak)^2 \norm{\bn_{12}}_2^2 \nonumber \\
     &= 2(\sigmak)^2 \chi_d^2
 \end{align}
 where, $\bn_{12}=\bn_1-\bn_2 \sim \setN(\bm 0, 2 \bI)$ and $\chi_d^2$ is standard chi-square distribution with $d$ degree of freedom.

 From \citet{laurent2000adaptiveestimation}[Lemma 1], the following holds with probability at least $1-e^{-\nu_k}$
 \begin{equation}
     \chi_d^2 \le d + 2\sqrt{d \nu_k} + 2 \nu_k
 \end{equation}
Plugging this in proves the lemma.
    
\end{proof}

\begin{lemma}\label{lemma: contractive residual of fine correction}
    Let $\Hdck:\bzack \mapsto \bzdck$ denote the function corresponding to fine correction as defined in Algorithm~\ref{algo:acdc}. Then, with probability at least $1-e^{-\nu_k}$, the following holds for any $\bx,\by \in \setX$ if $\sigma_{\bsk}^2 < \frac{1}{M_t}$:
    \begin{equation}
    \norm{(\Hdck - I)(\bx) - (\Hdck - I)(\by)}_2^2 \le (\deldck)^2 +  (\epdck)^2  \norm{\bx - \by}_2^2
    \end{equation}
    where, $(\deldck)^2=\frac{32d \sigma_{\bsk}^2}{ (1-M_t\sigma_{\bsk}^2)}  \log \frac{2}{\nu_k}$, and $(\epdck)^2=\left( \frac{\sqrt{2} M_t  \sigma_{\bsk}^2}{1- \sigma_{\bsk}^2 M_t}\right)^2$.
\end{lemma}
\begin{proof}
Recall that the target distribution for this step is given by
      \begin{equation}
      \log p(\bz_{\sigmak}|\bzack)\propto \log p(\bzack | \bz_{\sigmak}) + \log p(\bz_{\sigmak}) \label{eqn: posterior of contractive residual of finecorrection}
  \end{equation}
  where, $p(\bzack | \bz_{\sigmak})=\setN(\bz_{\sigmak}, \sigma_{\bsk}^2\bI)$. 

 Under Assumptions~\ref{assumpt: smoothness of logpdata} and \ref{assumpt: coercivity of logpdata}, the target distribution $p(\bz_{\sigmak} | \bzack)$ also inherits smoothness and coercivity properties. These conditions imply the ergodicity of  corresponding Langevin diffusion \citep{mattingly2002ergodicity,chen2020onstationarypoint}. In particular, Fokker-Planck equation \citep{uhlenbeck1930onthetheory} characterizes $p(\bz_{\sigmak} | \bzack)$ as its unique stationary distribution. Consequently, the iterates $\bzdck$ obtained through Langevin dynamics converge to this distribution as the step size $\etak \to 0$ and the number of iterations $J \to \infty$.  The gradient and hessian of the log of this desired distribution are given by
\begin{align}
    \nabla_{\bzdck}\log p(\bzdck|\bzack) &= \frac{1}{\sigma_{\bsk}^2}(\bzack - \bzdck) + \nabla_{\bzdck} \log p_t(\bzdck) \\
    \nabla_{\bzdck}^2 \log p(\bzdck | \bzack) &= -\frac{1}{\sigma_{\bsk}^2} + \nabla_{\bzdck}^2 \log p_t(\bzdck)
\end{align}
Here, $t$ refers to the noise level such that $\sigma(t)=\sigmak$. By the $M_t$-smoothness of $\log p_t$ distribution (Lemma \ref{lemma: Mt smoothness of pt}), we have
\begin{equation}
  - \left( M_t + \frac{1}{\sigma_{\bsk}^2} \right) { \bI}  \preceq \nabla_{\bzdck}^2 \log p(\bzdck | \bzack) \preceq  \left( M_t - \frac{1}{\sigma_{\bsk}^2} \right) { \bI}
\end{equation}
When $M_t < \dfrac{1}{\sigma_{\bsk}^2}$, the hessian is negative semi-definite that implies the distribution being log concave. This also implies that when $M_t \ll \dfrac{1}{\sigma_{\bsk}^2}$, the likelihood term dominates in the posterior \eqref{eqn: posterior of contractive residual of finecorrection}.

Using \citep{wainwright2019high}[Theorem 3.16], the following holds with probability at least $1-\nu_k$
  \begin{align}
       \norm{\bzdck - \Exp[\bzdck | \bzack]}_2 \le  \sqrt{\frac{4}{\lambda_t} \log \frac{2}{\nu_k}} \label{eqn: concentration of log-concave distribution}
    \end{align}
    where $\lambda_t= -M_t + \frac{1}{\sigma_{\bsk}^2}$.

        Now, using Tweedie's lemma, we have
    \begin{align}
        &\Exp[\bzdck | \bzack] = \bzack +  \sigma_{\bsk}^2 \nabla_{\bzack} \log p(\bzack) \nonumber\\
        \implies & \Exp[\bzdck | \bzack] - \bzack = \sigma_{\bsk}^2 \nabla_{\bzack} \log p(\bzack)
    \end{align}

    Combining the above two results,  with probability at least $1-2\nu_k$, the difference of residual for $\bx$ and $\by$ can be bounded as
  \begin{align}
  \norm{\Rdck(\bx) - \Rdck(\by)}_2^2 &\le 2\left(2 \sqrt{\frac{4}{\lambda_t} \log \frac{2}{\nu_k}} \right)_2^2 + \norm{\sigma_{\bsk}^2 \left( \nabla_{\bx} \log p_{\bzack}(\bx) - \nabla_{\by} \log p_{\bzack}(\by) \right)}_2^2 \nonumber\\
  & \le \frac{32d}{\lambda_t} \log \frac{2}{\nu_k}  + 2\norm{\sigma_{\bsk}^2 \left( \nabla_{\bx} \log p_{\bzack}(\bx) - \nabla_{\by} \log p_{\bzack}(\by) \right)}^2 \nonumber \\
    &\le  \frac{32d}{\lambda_t} \log \frac{2}{\nu_k}  + 2\sigma_{\bsk}^4 M_{\bzack}^2 \norm{\bx - \by}^2 \label{eqn: smoothness constant of log p zack}
  \end{align}
  where,  $M_{\bzack}$ is the smoothness constant of $\log p_{\bzack}$ that can be derived using a similar proof procedure as Lemma~\ref{lemma: Mt smoothness of pt}, and $\lambda_t = -M_t + \frac{1}{\sigma_{\bsk}^2}$. 

  Following the similar procedure as in proof of Lemma \ref{lemma: Mt smoothness of pt}, we get,
  \begin{equation}
      M_{\bzack} \le \frac{M_t}{1-\sigma_{\bsk}^2 M_t}
  \end{equation}
Plugging this leads to the lemma.

\end{proof}

\begin{lemma}\label{lemma: contractive residual tweedie's lemma projection to clean}
    Let $\Htwk:\bzdck \mapsto \bztwk$ denote the projection function using Tweedie's lemma defined in Algorithm~\ref{algo:acdc}. Then, we have the following
    \begin{equation}
        \norm{(\Htwk -I)(\bx) - (\Htwk - I)(\by)}_2^2 \le (\eptwk)^2\norm{\bx-\by}_2^2 + \deltwk^2
    \end{equation}
    for any $\bx,\by \in \setX$, where $(\eptwk)^2=(\sigmak)^4 M_t^2$, and $\deltwk^2=0$.
\end{lemma}

\begin{proof}
From Tweedie's lemma, we have
    \begin{align*}
    \bztwk &= \Exp[\bz_0|\bz_t=\bzdck] \\
    &= \bzdck + (\sigmak)^2 \nabla \log p_t(\bzdck)
\end{align*}
where, $t\in [0,T]$ such that $\sigmak = \sigma(t)$.

Then, the residuals are given by
\begin{align}
    \Rtwk(\bx) &= (\sigmak)^2  \nabla \log p_t(\bx) \\
     \Rtwk(\by) &= (\sigmak)^2  \nabla \log p_t(\by)
\end{align}

Now, the norm of the difference of the residuals can be written as
\begin{align}
    \norm{\Rtwk(\bx) - \Rtwk(\by)}_2^2 &= (\sigmak)^4 \norm{\nabla \log p_t(\bx)-\nabla \log p_t(\by)}_2^2 \\
    & \le (\sigmak)^4 M_t^2 \norm{\bx - \by}_2^2
\end{align}
where, $M_t$ is the smoothness constant of $\log p_t(\bx)$.
\end{proof}

\subsection{Main proof:} \label{main proof: contractive of overall 3-step ACDC with tweedie's lemma}
\begin{proof}
\textbf{Part (a)}
Using Lemma \ref{lemma: contractive residual of approximate correction}, \ref{lemma: contractive residual of fine correction}, and \ref{lemma: contractive residual tweedie's lemma projection to clean}, with probability at least $1-2e^{-\nu_k}$, we have
\begin{align}
     &\norm{R_{\sigmak}(\bx) - R_{\sigmak}(\by)}_2^2 \nonumber \\
     \le& 3 \norm{\Rack(\bx) - \Rack(\by)}_2^2 + 3 \norm{\Rdck(\bx) - \Rdck(\by)}_2^2 + 3 \norm{\Rtwk(\bx) - \Rtwk(\by)}_2^2 \nonumber \\
      \le & 3( (\epack)^2 + (\epdck)^2 + (\eptwk)^2 ) \norm{\bx - \by}_2^2 + 3(\delack^2 + \deldck^2 + \deltwk^2) \nonumber \\
      \le & \epsilon_{ k}^2 \norm{\bx - \by}_2^2 + \delta_{ k}^2
\end{align}
where, \[\epsilon_{ k}^2 = 3\left(\left( \frac{\sqrt{2}M_t  \sigma_{\bsk}^2}{1- \sigma_{\bsk}^2 M_t}\right)^2 + (\sigmak)^4 M_t^2 \right)\] 
\[\delta_{ k}^2 = 3 \left(2 (\sigmak )^2 (d + 2\sqrt{d\nu_k}+2\nu_k) + \frac{32d \sigma_{\bsk}^2}{ (1-M_t\sigma_{\bsk}^2)} \log \frac{2}{\nu_k}\right)\]

Using Lemma~\ref{lemma: Mt smoothness of pt} leads to the final theorem.

\textbf{Part (b).}
Let us set $\nu_k = \ln \left( \frac{2\pi^2}{6\eta}\right) + 2 \ln k$. With this, the above weakly nonexpansiveness holds for all $k\in\Nbbp$ with probability at least
\begin{align}
    &1- \sum_{k=1}^{\infty} 2 e^{- \ln \left( \frac{2\pi^2}{6\eta}\right) - 2 \ln k} \nonumber \\
    =& 1- \frac{6\eta}{\pi^2}  \times \sum_{k=1}^\infty \frac{1}{k^2}
\end{align}

Using Riemann zeta function \citep{titchmarsh1986theory} at value 2, we have
\begin{equation}
    \zeta(2) = \sum_{k=1}^\infty \frac{1}{k^2} = \frac{\pi^2}{6}
\end{equation}
Plugging this in we get the probability to be at least $1-\eta$.
 Now, combining the results with Theorem~\ref{thm:convergenceofpnpadmm} leads to the final proof of this part.
\end{proof}

\section{Proof of Theorem~\ref{thm:increasingrhoconvergence}} \label{proof:increasingrhoconvergence}
 Here, show that our 3-step AC-DC denoiser is bounded with high probability. We first show that each of 3 steps are bounded, and then combined them to establish the boundedness of our AC-DC denoiser as a whole. And following the boundedness, we show that AC-DC ADMM-PnP converges to a fixed point with proper scheduling of $\sigmak$ and $\sigma_{\bsk}$.

\begin{lemma}[Uniform score bound]
    Suppose Assumption~\ref{assumpt: smoothness of logpdata} holds. Let
    \[D:={\rm diam}(\setX)=\sup_{\bx,\by \in \setX}\norm{\bx - \by}_2<\infty \text{ and } S:=\inf_{\bx \in \setX} \norm{ \nabla \log \pdata(\bx)}_2 < \infty.\]
    Then, with $L=MD+S$, we have
    $$ \sup_{\bx \in \setX} \norm{\nabla \log \pdata (\bx)}_{\infty} \le L. $$
\end{lemma}
\begin{proof}
    From Assumption~\ref{assumpt: smoothness of logpdata}, we have
    \begin{equation}
        \norm{\nabla \log \pdata(\bx)-\nabla \log \pdata(\by)}_2 \le M \norm{\bx - \by}_2, \; \forall \bx,\by \in \setX
    \end{equation}

    Fix any $\bx_0 \in \setX$. By the triangle inequality, for all $\bx \in \setX$,
    \begin{align}
        \norm{\nabla \log \pdata(\bx)}_2 &\le \norm{\nabla \log \pdata(\bx) - \nabla \log \pdata(\bx_0)}_2 + \norm{\nabla \log \pdata(\bx_0)}_2 \nonumber \\
        &\le M \norm{\bx - \bx_0}_2 + \norm{\nabla \log \pdata(\bx_0)}_2
    \end{align}

    Taking the supremum over $\bx \in \setX$ and then the infimum over $\bx_0\in \setX$ yields
    \begin{align}
        \sup_{\bx \in \setX} \norm{\nabla \log \pdata(\bx)}_2 &\le \sup_{\bx \in \setX} M \norm{\bx - \bx_0}_2 + \inf_{\bx_0 \in \setX}\norm{\nabla \log \pdata(\bx_0)}_2  \\
        &\le MD + S
    \end{align}
    Because $\norm{\bu}_{\infty} \le \norm{\bu}_2$ for any vector,
    \begin{equation}
        \sup_{\bx \in \setX} \norm{\nabla \log \pdata(\bx)}_{\infty} \le MD + S
    \end{equation}
    which proves the above lemma.
\end{proof}

\begin{lemma}\label{lemma: bounded gradient of intermediate levels}
Assuming $\norm{\nabla \log p_{\rm data} (\bx)}_{\infty} \le , \forall \bx \in \setX$ , the score of intermediate noisy distributions $\{ p_t \}_{t \in [0,T]}$ are bounded as.
\begin{equation}
    \norm{\nabla \log p_t(\bx)}_2 \le \sqrt{d} L
\end{equation}
for all $\bx \in \setX$.
\end{lemma}
\begin{proof}
We have $\bx_t = \bx_0 + \sigma(t) \bn$ with $\bn \sim \setN(\bm 0, \bI)$. Let us denote $\bn_1 = \sigma(t) \bn$. Then, the marginal distribution is given by the convolution of two distributions.
\begin{align}
    p_t(\bx) = \int_{\bx_1\in\setX} p_{\bn_1}(\bx_1) p_0(\bx-\bx_1) d\bx_1
\end{align}
Then, the score is given by
\begin{align}
    \nabla \log p_t(\bx) &= \frac{\nabla p_t(\bx)}{p_t(\bx)} \nonumber \\
    &= \frac{1}{p_t(\bx)} \int_{\bx_1\in\setX}  p_{\bn_1}(\bx_1) \nabla_{\bx}p_0(\bx-\bx_1) d\bx_1 \nonumber \\
    &= \frac{1}{p_t(\bx)} \int_{\bx_1\in\setX}  p_{\bn_1}(\bx_1) p_0(\bx-\bx_1) \frac{\nabla_{\bx}p_0(\bx-\bx_1)} {p_0(\bx-\bx_1)} d\bx_1 \nonumber \\
    &= \frac{1}{p_t(\bx)} \int_{\bx_1\in\setX}  p_{\bn_1}(\bx_1) p_0(\bx-\bx_1) \nabla_{\bx} \log p_0(\bx-\bx_1) d\bx_1 
\end{align}
Now the norm can be bounded as
\begin{align}
    \norm{\nabla \log p_t(\bx)}_2 &\le \frac{1}{p_t(\bx)}  \int_{\bx_1\in\setX}  p_{\bn_1}(\bx_1) p_0(\bx-\bx_1) \norm{ \nabla_{\bx} \log p_0(\bx-\bx_1) }_2d\bx_1 \nonumber \\
    &\le \sup_{\bx_2 \in \setX} \norm{\nabla_{\bx} \log p_0(\bx_2)}_2 \frac{1}{p_t(\bx)} \int_{\bx_1\in\setX}  p_{\bn_1}(\bx_1) p_0(\bx-\bx_1)d\bx_1 \nonumber \\
    &= \sqrt{d} L
\end{align}
The final equality is due to the fact $\norm{\bx}_2 \le \sqrt{d} \norm{\bx}_{\infty}$.
    
\end{proof}

\begin{lemma}\label{lemma: boundedness of residual of approximate correction}
    Let $\Hack: \tbzk \mapsto \bzack$ denote the function corresponding to approximate correction to noise level $\sigmak$ defined in Algorithm~\ref{algo:acdc}. Then, with probability at least $1-e^{-\nu}$, the following holds for any $\bx,\by \in \setX$
    \begin{equation}
        \frac{1}{d}\norm{(\Hack-I)(\bx)}_2^2 \le (\sigmak)^2 (1+2\sqrt{\nu} + 2 \nu)
    \end{equation}
\end{lemma}
\begin{proof}
 For any $\bx \in \setX$, we have the residual $\Rack(\bx)=(\Hack - I)(\bx) = \sigmak \bn, \; \bn \sim \setN(\bm 0 , \bI)$. Then,
 \begin{align}
     \norm{\Rack(\bx)}_2^2 = (\sigmak)^2 \norm{\bn}_2^2 = (\sigmak)^2 \chi_d^2
 \end{align}
where $\chi_d^2$ is standard chi-square distribution with $d$ degrees of freedom.
From \citet{laurent2000adaptiveestimation}[Lemma 1], the following holds with probability at least $1-e^{-\nu}$
\begin{equation}
    \chi_d^2 \le d+2\sqrt{d\nu} + 2 \nu
\end{equation}
This implies $\frac{1}{d}\norm{\Hack-I)(\bx)}_2^2 \le (\sigmak)^2 ( 1+2\sqrt{1\nu} + 2 \nu) $ with probability at least $1-e^{-\nu}$ due to $d\ge 1$.
    
\end{proof}

\begin{lemma}\label{lemma: boundedness of residual of fine correction}
    Let $\Hdck:\bzack \mapsto \bzdck$ denote the function corresponding to fine correction defined in Algorithm~\ref{algo:acdc}. Assume $\norm{\nabla \log p_{\rm data} (\bx)}_{\infty} \le L$. Then, with probability at least $1-e^{-\nu}$, the following holds for any $\bx,\by \in \setX$:
    \begin{equation}
    \frac{1}{d} \norm{(\Hdck - I)(\bx)}_2^2 \le \frac{8}{\lambda_t} \log \frac{2}{\nu} + \sigma_{\bsk}^4  L^2
    \end{equation}
    where, $\lambda_t= -M_t + \frac{1}{\sigma_{\bsk}^2}$.
\end{lemma}

\begin{proof}
From \eqref{eqn: concentration of log-concave distribution}, with probability at least $1-2\nu$, the norm can be bounded as
  \begin{align}
  \norm{\Rdck(\bx) }_2^2 & \le \frac{8d}{\lambda_t} \log \frac{2}{\nu}  + \sigma_{\bsk}^4 \norm{ \nabla_{\bx} \log p_{\bzack}(\bx)}_2^2 
  \end{align}
  Then, using Lemma \ref{lemma: bounded gradient of intermediate levels}, we have
  \begin{equation}
    \frac{1}{d}\norm{\Rdck(\bx)}_2^2 \le \frac{8}{\lambda_t} \log \frac{2}{\nu} + \sigma_{\bsk}^4  L^2      
  \end{equation}

\end{proof}

  \begin{lemma}\label{lemma: bounded residual of tweedie's lemma projection to clean}
    Let $\Htwk:\bzdck \mapsto \bztwk$ denote the projection function using Tweedie's lemma defined in Algorithm~\ref{algo:acdc}. Assume $\norm{\nabla \log p_{\rm data} (\bx)}_{\infty} \le L$.  Then, we have the following
    \begin{equation}
       \frac{1}{d} \norm{(\Htwk -I)(\bx)}_2^2 \le (\sigmak)^4 L^2
    \end{equation}
    for any $\bx \in \setX$. 
\end{lemma}

\begin{proof}
From Tweedie's lemma, we have
    \begin{align}
    \bztwk &= \Exp[\bz_0|\bz_t=\bzdck] \nonumber \\
    &= \bzdck + (\sigmak)^2 \nabla \log p_t(\bzdck)
\end{align}
where, $t\in [0,T]$ such that $\sigmak = \sigma(t)$. 

Now, the norm of residual can be written as
\begin{align}
    \norm{\Rtwk(\bx) }_2^2 &= \norm{(\sigmak)^2  \nabla \log p_t(\bx)}_2^2 \nonumber \\
    & \le (\sigmak)^4  L^2 \cdot d
\end{align}
where, $L$ bound of gradient from Lemma~\ref{lemma: bounded gradient of intermediate levels}.
\end{proof}
\subsection{Main proof}
Combining Lemmas~\ref{lemma: boundedness of residual of approximate correction}, \ref{lemma: boundedness of residual of fine correction} and \ref{lemma: bounded residual of tweedie's lemma projection to clean} leads to the proof of part (a) of Theorem \ref{thm:increasingrhoconvergence}.

With probability at least $1-2e^{-\nuk}$, the denoiser satisfies the bounded residual condition.
    \begin{align}
           \nicefrac{1}{d} \norm{(D_{\sigmak}-I)(\bx)}_2^2 \le \ck^2. 
    \end{align}

Let's define the relative residue as:
\begin{equation}
    \betak := \frac{1}{\sqrt{d}} \left( \norm{\bxk - \bx^{(k-1)}}_2 +  \norm{\bzk - \bz^{(k-1)}}_2 + \norm{\buk - \bu^{(k-1)}}_2 \right)
\end{equation}
 For any $\eta \in [0,1)$ and a constant $\gamma > 1$, the penalty parameter $\rhok$ is adjusted at each iteration $k$ according to following rule \citep{chan2016plugandplayadmmimagerestoration}:
\begin{equation} \label{eqn: adaptive penalty rule}
    \rhokk = \begin{cases}
        \gamma \rhok &\text{ if } \betakk \ge \eta \betak  \quad \text{ (Case 1)}\\
        \rhok &\text{ else } \text{ (Case 2)}
    \end{cases}
\end{equation}

 The PnP-ADMM with adaptive penalty involves two cases as shown above.

At iteration $k$, if { Case 1} holds, then by Lemma \ref{lemma: residual bound for bounded denoiser} we have
\begin{equation}
    \betakk \le 6c_k + 2c_{k-1}+\frac{2R}{\rhok}
\end{equation}

On the other hand if { Case 2} holds, then,
\begin{equation}
    \betakk \le \eta \betak
\end{equation}

{
Define $a_k=6c_k+2c_{k-1}+\frac{2R}{\rho_k}$. Combining two cases, we get
\begin{equation}
    \betakk \le \delta \beta_k + a_k, \quad \delta = \begin{cases}
        \eta, \quad &\text{if Case 2 holds at iteration $k$} \\
        0, \quad& \text{if Case 1 holds at iteration $k$}
    \end{cases}
\end{equation}

Note that for Case 1 $\rhokk=\gamma \rhok$ and with $\gamma>1$, we get $\lim_{k\to \infty}\frac{c}{\rho_k}=0$. In addition, with $\nu_k= \ln \frac{2 \pi^2}{6\eta} + 2\ln k$, and the scheduling of $\sigmak$, $\sigma_{\bsk}$ that satisfies 
\[\lim_{k\to\infty} (\sigmak)^2 (2+4\sqrt{\nu_k} + 4 \nu_k)=0, \; \lim_{k\to\infty} \frac{\sigma_{\bsk}^2}{1-M\sigma_{\bsk}^2} \log \frac{2}{\nu_k}=0, \; \lim_{k\to\infty} \sigmak =0, \lim_{k\to\infty} \sigma_{\bsk}=0\] 
results in $\lim_{k\to\infty}a_k = 0$.

}

As $k \to \infty$, 3 different scenarios could occur. Let's analyze each of the scenarios one by one. 
\begin{itemize}
    \item Scene 1: { Case 1} occurs infinitely many times and Case 2 occurs finitely many times.
    
    {
    When Scene 1 occurs, then, there exists a constant $K_1>0$ such that $\beta_{k+1} \le a_k$ for all $k \ge K_1$. Since $\lim_{k\to \infty} a_k = 0$, this leads to $\lim_{k\to\infty} \beta_k =0$.
    }
    \item Scene 2: { Case 2} occurs infinitely many times and Case 1 occurs finitely many times.
    {
    Similarly, there exists a constant $K_2>0$ such that $\beta_{k+1} \le \eta \beta_k$ for all $k \ge K_2$. Then, we have
    \begin{equation}
        \betak \le \eta^{k-K_2} \beta_{K_2}
    \end{equation}
    And with $\eta \in [0,1)$, we have $\lim_{k\to\infty} \betak = 0$.
    }
    \item Scene 3: Both { Case 1} and { Case 2} occurs infinitely many times. With the Scene 1 and Scene 2 converging, the sequence $\lim_{k\to\infty}\beta_k=0$ under this Scene as well.
\end{itemize}

This proves the part (b) of Theorem~\ref{thm:increasingrhoconvergence}.

\begin{lemma} \label{lemma: residual bound for bounded denoiser}
    For any iteration $k$ that falls into Case 1, the following holds
    \begin{equation}
        \betakk \le 6c_k + 2 c_{k-1} + \frac{2R}{\rhok}
    \end{equation}
\end{lemma}

\begin{proof}
Consider the subproblem \eqref{eq:mlesubgeneric} with adaptive penalty parameter $\rhok$ (as defined in \eqref{eqn: adaptive penalty rule}),
\begin{equation}
    \bxkk = \arg \min_{\bx} \frac{1}{\rhok} \ell(\by || {\cal A}(\bx)) + \frac{1}{2} \norm{\bx - \bzk + \buk}_2^2
\end{equation}

 With the first order optimality condition, the solution $\bxkk$ satisfies
 \begin{align}
     \frac{1}{\rhok} \nabla_{\bx} \ell(\by || {\cal A}(\bx)) \big |_{\bx = \bxkk} + (\bxkk - \bzk + \buk) = \bm 0 \\
     \implies \frac{1}{\sqrt{d}} \norm{\bxkk - \bzk + \buk}_2 = \frac{1}{\rhok \sqrt{d}}\norm{\nabla_{\bx} \ell(\by || {\cal A}(\bx)) \big |_{\bx = \bxkk}}_2
 \end{align}

Using the assumption of existence of $R < \infty$ such that $\norm{\nabla_{\bx} \ell(\by || {\cal A}(\bx))}_2 /\sqrt{d} \le R, \; \forall \bx \in \setX$, we get
\begin{equation}
    \frac{1}{\sqrt{d}} \norm{\bxkk - \bzk + \buk}_2 \le \frac{R}{\rhok}
\end{equation}

Since the denoiser $D_{\sigmak}$ is bounded with probability at least $1-2e^{-\nuk}$, we have
\begin{equation}
    \frac{1}{\sqrt{d}} \norm{(D_{\sigmak} - I) (\bx)}_2 \le c_k
\end{equation}
Now,
\begin{align}
    \frac{1}{\sqrt{d}}\norm{\bxkk - \bzkk + \buk}_2 &= \frac{1}{\sqrt{d}}\norm{\bxkk + \buk - D_{\sigmak}(\bxkk + \buk) }_2 \\
    &= \frac{1}{\sqrt{d}} \norm{(D_{\sigmak} - I) (\bxkk + \buk)}_2 \\
    &\le \ck
\end{align}

Now, using triangle inequality, we can bound $\norm{\bzkk - \bzk}_2$ as
\begin{align}
    \frac{1}{\sqrt{d}} \norm{\bzkk - \bzk}_2 &= \frac{1}{\sqrt{d}} \norm{\bzkk- \bxkk - \buk + \bxkk + \buk - \bzk}_2\\
    &\le \frac{R}{\rhok} + \ck
\end{align}

Similarly, it can be shown that
\begin{align}
    \frac{1}{\sqrt{d}}\norm{\bukk}_2 &= \frac{1}{\sqrt{d}} \norm{\buk + \bxkk - \bzkk}_2 \\
    &=  \frac{1}{\sqrt{d}} \norm{\buk + \bxkk - D_{\sigmak}(\bxkk + \buk)}_2 \\
    &= \frac{1}{\sqrt{d}} \norm{(D_{\sigmak}-I)(\bxkk+\buk)}_2 \\
    &= c_k
\end{align}
This implies $\frac{1}{\sqrt{d}}\norm{\bukk - \buk}_2 \le 2c_k$. Finally, we use $\bxkk = \bukk - \buk + \bzkk$ to obtain
\begin{align}
    & \frac{1}{\sqrt{d}}\norm{\bxkk - \bxk}_2 \\
    = & \frac{1}{\sqrt{d}} \norm{\bukk -\buk + \bzkk -\buk + \bu^{(k-1)} -\bzk}_2 \\
    \le & \frac{1}{\sqrt{d}} \norm{\bukk - \buk}_2 + \frac{1}{\sqrt{d}} \norm{\buk - \bu^{(k-1)}}_2 + \frac{1}{\sqrt{d}}\norm{\bzkk - \bzk}_2 \\
    \le & 2 c_k + 2c_{k-1} + \frac{R}{\rhok} + c_k \\
    = & 3c_k + 2c_{k-1} + \frac{R}{\rho_k}
\end{align}
Combining all the bounds results using triangle inequality results in
\begin{align}
    \betakk \le 6c_k + 2c_{k-1} + \frac{2R}{\rho_k}
\end{align}

where, 
$\ck=(\sigmak)^2(2 +4\sqrt{\nuk} + 4 \nuk) + \nicefrac{16 \sigma_{\bsk}^2}{1-M \sigma_{\bsk}^2} \log \nicefrac{2}{\nuk} + 2\sigma_{\bsk}^4  L^2 + 2 (\sigmak)^4 L^2$
\end{proof}

\begin{remark}
    While the proposed method and its theoretical results are based on \emph{Variance Exploding} (VE) scheduling, they can be easily extended to \emph{Variance Preserving} (VP) scheduling case \citep{karras2022elucidatingdesignspacediffusionbased}.
\end{remark}

{
\subsection{Theoretical Results with finite DC steps $J$} \label{section: theoretical results with finite steps}
\begin{lemma}\label{lemma: contractive residual of fine correction under finite steps}
    Let $\Hdck:\bzack \mapsto \bzdck$ denote the function corresponding to fine correction as defined in Algorithm~\ref{algo:acdc} with finite $J$ and $\etak \le 2\sigma_{\bsk}^2$. Also, let $\pik = p(\bz_{\sigmak} | \bzack )$ be the stationary target distribution and $\tpio$ be initial distribution used for the DC at iteration $k$.  Then, with probability at least $1-e^{-\nu_k}$, the following holds for any $\bx,\by \in \setX$ if {$\nicefrac{1}{\sigma_{\bsk}^2} < M_t$}:
    \begin{equation}
    \norm{(\Hdck - I)(\bx) - (\Hdck - I)(\by)}_2^2 \le (\deldck)^2 +  (\epdck)^2  \norm{\bx - \by}_2^2
    \end{equation}
    where,$0<\kappa<1$, $C>0$, $(\deldck)^2=\frac{64d\sigma_{\bsk}^2}{ (1-M_t \sigma_{\bsk}^2)}  \log \frac{2}{\nu_k} + C(1-\kappa)^{2J} \setW_2^2(\tpio, \pik) + O\left((\etak)^2\right)$, and $(\epdck)^2=\left( \frac{2\sqrt{2} M_t \sigma_{\bsk}^2  }{1- M_t \sigma_{\bsk}^2}\right)^2$.
\end{lemma}
\begin{proof}
With finite $J$, the langevin dynamics doesn't necessarily converge to the stationary distribution $\pik = p(\bz_{\sigmak}|\bzack)$. Let $\tpio$ be the initial distribution used to initialize the finite step langevin dynamics and $\tpik$ be the distribution of the iterate after running finite $J$ steps of langevin dynamics. Using \citep{dalalyan2019userfriendly}[Theorem 1], the following holds for $\etak \le 2\sigma_{\bsk}^2$
    \begin{align}
        \setW_2( \tpik, \pik) \le (1-\kappa)^J \setW_2(\tpio,  \pik) + O\left(\etak\right)
    \end{align}
    where, $0<\kappa<1$ and $\setW_2$ is the 2-Wasserstein distance.

    Using Kantorovich and Rubinstein dual representation \citep{villani2008optimal},
    \begin{align}
        \norm{\Exp_{\tpik}[\bz] - \Exp_{\pik}[\bz]}_2 &= \norm{\int \bzdck d\tpik(\bz) - \int \bz d\pik(\bz)}_2 \\
        & \le \setW_1(\tpik, \pik) \\
        & \le \setW_2(\tpik, \pik)
    \end{align}
    The last inequality is due to the Holder's inequality \cite{villani2008optimal}.
    Using triangle inequality leads to following:
    \begin{align}
        \norm{\bzdck - \Exp_{\pik}[\bz]}_2 &\le \norm{\bzdck - \Exp_{\tpik}[\bz]}_2  + \norm{\Exp_{\tpik}[\bz] - \Exp_{\pik}[\bz]}_2 \\
        &\le \norm{\bzdck - \Exp_{\tpik}[\bz]}_2 + (1-\kappa)^J \setW_2(\tpio,  \pik) + O\left(\etak\right)
    \end{align}

    Using this result and following the same procedure as in Lemma~\ref{lemma: contractive residual of fine correction}, we get the theorem.

\end{proof}

\begin{theorem}\label{theorem: contractive overall denoising for tweedie's lemma with finite J}
   Suppose that the assumptions in Theorem~\ref{thm:convergenceofpnpadmm}, Assumption~\ref{assumpt: smoothness of logpdata}, and Assumption~\ref{assumpt: coercivity of logpdata} hold. Further, assume that  the DC steps finite steps $J$ and $\etak \le 2\sigma_{\bsk}^2$. Also, let $\pik = p(\bz_{\sigmak} | \bzack )$ be the stationary target distribution and $\tpio$ be initial distribution used for the DC at iteration $k$.
    Let $D_{\sigmak}:\tbzk \mapsto \bm z^{(k)}_{\rm tw}$ denote the AC-DC denoiser. Then, we have:
    
With probability at least $1-2e^{-\nu_k}$, the following holds for iteration $k$ of ADMM-PnP:
                \begin{equation}
                    \norm{(D_{\sigmak} -I)(\bx) - (D_{\sigmak} - I)(\by)}_2^2 \le \epsilon^2_{ k} \norm{\bx-\by}_2^2 + \delta^2_{ k}
                \end{equation}
                for any $\bx,\by \in \setX$, $k \in \Nbbp$, a constant $0<\kappa<1$ and $C>0$, when ${ \sigma_{\bsk}^2} + (\sigmak)^2 < 1/M$ with 
                \begin{align}
                \epsilon^2_{ k} &= 3(( \nicefrac{2\sqrt{2}M\sigma_{\bsk}^2}{(1-  M} \sigma_{\bsk}^2))^2 + (\sigmak)^4 M^2 ) \\ 
            \delta^2_{ k} &= 3 (2 (\sigmak )^2 (d + 2\sqrt{d\nu_k}+2\nu_k) + \nicefrac{64d\sigma_{\bsk}^2}{ (1-M \sigma_{\bsk}^2)}  \log \nicefrac{2}{\nu_k}) + \nonumber \\
            & \quad \quad C(1-\kappa)^{2J} \setW_2^2(\tpio, \pik) + O\left( (\etak)^2\right). 
                \end{align}
                
In other words, with $\nu_k=\ln \nicefrac{2\pi}{6\eta}+2\ln k$, the denoiser $D_{\sigmak}$ satisfies part (a) for all $k \in \Nbbp$ with probability at least $1-\eta$. 
\end{theorem}

\begin{proof}
    By substituting Lemma~\ref{lemma: contractive residual of fine correction} with Lemma~\ref{lemma: contractive residual of fine correction under finite steps} leads to the theorem.
\end{proof}

\begin{theorem} \label{thm:increasingrhoconvergencefinitesteps}
 Suppose that { Assumptions~\ref{assumpt: smoothness of logpdata}-\ref{assumpt: coercivity of logpdata}} hold. Let
    \(D:={\rm diam}(\setX)=\sup_{\bx,\by \in \setX}\norm{\bx - \by}_2<\infty, \; S:=\inf_{\bx \in \setX} \norm{ \nabla \log \pdata(\bx)}_2 < \infty\) and define $L:=MD+S$.
Let $D_{\sigmak}:\tbzk \mapsto \bm z^{(k)}_{\rm tw}$ denote the AC-DC denoiser.   Further, assume that  the DC steps finite steps $J$ and $\etak \le 2 \sigma_{\bsk}^2$. Also, let $\pik = p(\bz_{\sigmak} | \bzack )$ be the stationary target distribution and $\tpio$ be initial distribution used for the DC at iteration $k$. Then, the following hold:

(\textbf{Boundedness}) With probability at least $1-2e^{-\nuk}$, the denoiser $D_{\sigmak}$ is bounded at each iteration $k$ i.e. \(\frac{1}{d} \norm{(D_{\sigmak} -I)(\bx)}_2^2 \le c_k^2 \) whenever $\sigma_{\bsk}^2 + (\sigmak)^2 < 1/M,$ where $\ck= (\sigmak)^2 (2 +4\sqrt{\nuk} + 4 \nuk) + \nicefrac{32 \sigma_{\bsk}^2}{( 1-M\sigma_{\bsk}^2)} \log \nicefrac{2}{\nuk}+ C(1-\kappa)^{2J} \setW_2^2(\tpio, \pik) + O\left((\etak)^2\right) + {   4 L^2 \sigma_{\bsk}^4} + 2 (\sigmak)^4 L^2$, $0<\kappa<1$, $C>0$ and $\nuk > 0$.\\
    Let $\nu_k = \ln \frac{2\pi^2}{6\eta}+2\ln k$ with $\eta\in (0,1]$. Consequently, the denoiser $D_{\sigmak}$ is bounded for all $k \in \Nbb_+ $ with corresponding $c_k$ and probability at least $1-\eta$.
\end{theorem}
\begin{proof}
    The proof follows similar as in Theorem~\ref{thm:increasingrhoconvergence} by incorporating the effect of finite $J$ in Lemma~\ref{lemma: boundedness of residual of fine correction} as done in Lemma~\ref{lemma: contractive residual of fine correction under finite steps}.
\end{proof}
}

\section{Theoretical results for ODE based denoiser} \label{appendix: ode description and theory}
Refer to the \citet{zhang2024improvingdiffusioninverseproblem} for details on ODE based denoiser.
\subsection{Theoretical results equivalent to Theorem~\ref{theorem: contractive overall denoising for tweedie's lemma}}
\begin{lemma}\label{lemma: contractive residual ode projection to clean}
    Let $\Hodek:\bzdck \mapsto \bzodek$ denote the projection function using ode based denoiser \citep{karras2022elucidatingdesignspacediffusionbased} in Algorithm~\ref{algo:acdc}. Then, we have the following
    \begin{equation}
        \norm{(\Hodek -I)(\bx) - (\Hodek - I)(\by)}_2^2 \le (\epodek)^2 \norm{\bx-\by}_2^2  + \delodek^2
    \end{equation}
    for any $\bx,\by \in \setX$ with $(\epodek)^2= 2\left( \int_{t=t_{\sigmak}}^0 \left(\sigma(t) \sigma'(t) M_t\right)^2  dt \right) $, and $\delodek^2=0$.
\end{lemma}
\begin{proof}
Then, the difference of residual of ode projection i.e. $\Rodek = \Hodek - I$ can be bounded as
\begin{align}
     \norm{\Rodek(\bx) - \Rodek(\by)}_2^2 =& \norm{\int_{t=t_{\sigmak}}^0  -\sigma(t) \sigma'(t) (\nabla \log p_t(\bx) - \nabla \log p_t(\by)) dt }_2^2 \nonumber \\
    \le & 2 \norm{\int_{t=t_{\sigmak}}^0 -\sigma(t) \sigma'(t) (\nabla \log p_t(\bx) - \nabla \log p_t(\by)) dt}_2^2 \nonumber \\
    \le & 2 \int_{t=t_{\sigmak}}^0  \left(\sigma(t) \sigma'(t)\right)^2 \norm{ (\nabla \log p_t(\bx) - \nabla \log p_t(\by)) }_2^2 dt \nonumber \\
    \le & 2 \int_{t=t_{\sigmak}}^0  \left(\sigma(t) \sigma'(t) \right)^2 M_t^2 \norm{ \bx - \by }_2^2 dt \nonumber  \\
    \le & 2\left( \int_{t=t_{\sigmak}}^0 \left(\sigma(t) \sigma'(t)\right)^2 M_t^2 dt \right) \norm{\bx - \by}_2^2 
\end{align}
\end{proof}

 \begin{theorem}\label{theorem: contractive overall denoising for ode}
    Suppose that the assumptions in Theorem~\ref{thm:convergenceofpnpadmm}, Assumption~\ref{assumpt: smoothness of logpdata} and Assumption~\ref{assumpt: coercivity of logpdata} hold. Further, assume that the step size satisfies $\etak \to 0$ and the number of iterations $J \to \infty$.
    Let $D_{\sigmak}:\tbzk \mapsto \bm z^{(k)}_{\rm tw}$ denote the AC-DC denoiser. Then, we have:
    
 (a)   With probability at least $1-2e^{-\nu_k}$, the following holds for iteration $k$ of ADMM-PnP:
                \begin{equation}
                    \norm{(D_{\sigmak} -I)(\bx) - (D_{\sigmak} - I)(\by)}_2^2 \le \epsilon^2_{ k} \norm{\bx-\by}_2^2 + \delta^2_{ k}
                \end{equation}
                for any $\bx,\by \in \setX$  and $k \in \Nbbp$ when $\sigma_{\bsk}^2 + (\sigmak)^2 < 1/M$ with 
                \begin{align}
                \epsilon^2_{ k} &= 3 ( \nicefrac{\sqrt{2}M  \sigma_{\bsk}^2}{1- \sigma_{\bsk}^2 M})^2 + 6\int_{t=t_{\sigmak}}^0 \left(\sigma(t) \sigma'(t)\right)^2 M_t^2 dt ) ) \label{eqn: thm4 epsilonk def} \\ 
            \delta^2_{ k} &= 3 (2 (\sigmak )^2 (d + 2\sqrt{d\nu_k}+2\nu_k) + \nicefrac{32d \sigma_{\bsk}^2}{ (1-M\sigma_{\bsk}^2)}  \log \nicefrac{2}{\nu_k}). \label{eqn: thm4 deltak def}
                \end{align}
In other words, if $\nu_k=\ln \nicefrac{2\pi}{6\eta}+2n k$, the denoiser $D_{\sigmak}$ satisfies part (a) for all $k \in \Nbbp$ with probability at least $1-\eta$. 
            
 (b)  
Assume that $\sigmak$ is scheduled such that 
$\lim_{k\rightarrow \infty } (\sigmak)^2 \nu_k =0$ for $\nu_k=\ln \nicefrac{2\pi}{6\eta}+2n k$, $\epsilon<1$, and $ \nicefrac{\epsilon}{\mu(1+\epsilon - 2\epsilon^2)} < \nicefrac{1}{\rho}$ all hold, where $\epsilon=\lim_{k\to\infty} \sup \epsilon_k$ with $\epsilon_k$ defined in \eqref{eqn: thm4 epsilonk def}. 
Consequently, $\delta=\lim_{k\to\infty} \sup \delta_k$ is finite and ADMM-PnP with the AC-DC denoiser with ode based denoiser converges to an $r$-ball (see $r$ in Theorem~\ref{thm:convergenceofpnpadmm}) with probability at least $1-\eta$.
\end{theorem}
\begin{proof}
    The proof follow similar to the proof of Theorem~\ref{theorem: contractive overall denoising for tweedie's lemma} in Appendix \ref{main proof: contractive of overall 3-step ACDC with tweedie's lemma} with the residual bound of Tweedie's lemma replaced by Lemma~\ref{lemma: contractive residual ode projection to clean}.
\end{proof}

\subsection{Theoretical Results equivalent to Theorem~\ref{thm:increasingrhoconvergence}}
\begin{lemma}\label{lemma: bounded residual of ode projection to clean}
    Let $\Hodek:\bzdck \mapsto \bzodek$ denote the projection function using ode based denoiser \citep{karras2022elucidatingdesignspacediffusionbased} in Algorithm~\ref{algo:acdc}. Assume $\norm{\nabla \log p_{\rm data}(\bx)}_\infty \le L, \; \forall \bx \in \setX$. Then, we have the following
    \begin{equation}
        \frac{1}{d} \norm{(\Hodek -I)(\bx)}_2^2 \le  L^2 \int_{t=t_{\sigmak}}^0 (\sigma(t) \sigma'(t)^2  dt
    \end{equation}
    for any $\bx \in \setX$.
\end{lemma}
\begin{proof}
Then, the residual of ode projection i.e. $\Rodek = \Hodek - I$ can be bounded as
\begin{align}
    \norm{\Rodek(\bx) }_2^2  =& \norm{\int_{t=t_{\sigmak}}^0  -\sigma(t) \sigma'(t) \nabla \log p_t(\bx)  dt }_2^2 \nonumber \\
    \le & d \cdot L^2 \norm{\int_{t=t_{\sigmak}}^0 -\sigma(t) \sigma'(t)  dt}_2^2  \nonumber \\
    \le & d \cdot L^2 \norm{\int_{t=t_{\sigmak}}^0 -\sigma(t) \sigma'(t)  dt}_2^2 \nonumber \\
    \le & d \cdot L^2 \int_{t=t_{\sigmak}}^0 (\sigma(t) \sigma'(t)^2  dt
\end{align}
\end{proof}

{A theorem analogous to Theorem~\ref{thm:increasingrhoconvergence} can also be obtained for ODE-based denoiser. The only difference lies in the expression of the constant $c_k$, which in this case becomes $$c_k=  (\sigmak)^2 (2 +4\sqrt{\nuk} + 4 \nuk) + \nicefrac{16 \sigma_{\bsk}^2}{1-M \sigma_{\bsk}^2} \log \nicefrac{2}{\nuk} + 2\sigma_{\bsk}^4  L^2 + 2 L^2 \int_{t=t_{\sigmak}}^0 (\sigma(t) \sigma'(t)^2  dt.$$ }

\section{Usage of Large Language Models (LLM)}
An LLM was used solely to assist with polishing the writing. LLM played no part in the experiments, results and conclusion.

\section{Experimental Details} \label{appendix: experimental details}

\begin{table}[t!]
  \centering
  \setlength{\tabcolsep}{8pt}
  \renewcommand{\arraystretch}{1.1}
  \caption{Hyperparameter settings for each task}
  \label{tab:hyperparams}
  \begin{tabular}{@{}lcccc@{}}
    \toprule
    \textbf{Task} & \(\rho\) & \(W \) & lr of Adam in \eqref{eq:mlesubgeneric} \\
    \midrule
       Superresolution ($4\times$)          & $100$ & $100$ & $3 \times 10^{-2}$  \\
     Gaussian Deblur           & $100$ & $100$ &  $5\times 10^{-2}$  \\
    HDR                     & $500$&  $100$  & $3\times 10^{-2}$\\
     Inpainting    (Random)            & $500$ & $100$ & $1\times 10^{-1}$  \\
     Inpainting (Box)               & $500$ & $100$ & $1\times 10^{-1}$  \\
     Motion Deblur             & $100$ & $100$ & $1\times 10^{-1}$  \\
     Nonlinear Deblur    & $300$ & $400$ & $3\times 10^{-1}$ \\
     Phase Retrieval        & $100$ &  $400$ & $1\times 10^{-1}$  \\
    \bottomrule
  \end{tabular}
\end{table}

\subsection{Details on task specific data-fidelity loss $\ell$}
We use \emph{mean square error} (MSE) as the data-fidelity loss for every task i.e.
\begin{equation}
    \ell(\by || \bx)=-\log p(\by|\bx) = \frac{1}{2\sigma_n^2}\norm{\by - \setA(\bx)}_2^2
\end{equation}

\subsection{Details on pretrained diffusion models}
The pretrained models provided in \citet{chung2023diffusion} are used in our experiment. Refer to \citet{chung2023diffusion} for more details on these pretrained models.

 \subsection{Baseline Details}
Unless mentioned otherwise, we conduct the experiments in the default settings of their original implementation except for maintaining consistency within the measurement operators.
 \begin{itemize}
     \item \textbf{DDRM  \citep{kawar2021snipssolvingnoisyinverse}: } We use $20$ steps DDIM with $\eta=0.85$ and $\eta_b=1$ as specified in \citet{kawar2022denoisingdiffusionrestorationmodels}.
     \item \textbf{DPS \citep{chung2023diffusion} :} The original implementation is ran in their default settings.
     \item \textbf{DiffPIR \citep{zhu2023denoising}:} The default settings are adopted in the experiments.
     \item \textbf{RED-diff \citep{mardani2023avariational}:} We use $\lambda=0.25$ and $lr=0.5$ as specified in the paper.
     \item \textbf{DAPS \citep{zhang2024improvingdiffusioninverseproblem}:} We use the best performing DAPS-4K version as proposed in the paper.
     \item \textbf{DPIR \citep{zhang2022plugandplay}}: We employ "drunet\_color" as PnP denoiser, while keeping all the other settings at their default values.
     \item \textbf{DCDP \citep{li2025decoupleddataconsistencydiffusion}}:  All the settings are set to their default values.
     \item \textbf{PMC \citep{sun2024provableprobabilisticimagingusing}}: PMC was proposed using different score models for two different tasks with relatively high measurement SNR. For a fair comparison, we used our own implementation with the same score model checkpoints as our methods, and further tuned this method accordingly.
 \end{itemize}

\subsection{Evaluation Metrics}
For all the methods, we use the implementation of PSNR, SSIM, and LPIPS provided in \emph{piq} python package. The default settings for these metrics are used except the average pooling enabled for LPIPS.
 \subsection{Computation Resource Details}
 All the experiments were run on a instance equipped with one Nvidia H100 GPU, 20 cores of 2.0 Ghz Intel Xeon Platinum 8480CL CPU, and 64 GB of RAM.

\section{Additional Experimental Results} \label{appendix: additional experimental results}
{
\subsection{ Illustration of proposed Denoiser}
Figure~\ref{fig: illustration of correction and denoising} illustrates the effect of the proposed correction-denoising procedure. The noisy input image $\tbzk$ typically lies far away from the Gaussian noise manifold, leading to poor denoising performance if directly used. To address this mismatch, our method first performs correction to effectively gaussianize the noise which is then denoised using Tweedie's lemma or ode-style score integration, producing a high-quality clean reconstruction.

\begin{figure}[t]
    \centering
    \includegraphics[width=0.8\linewidth]{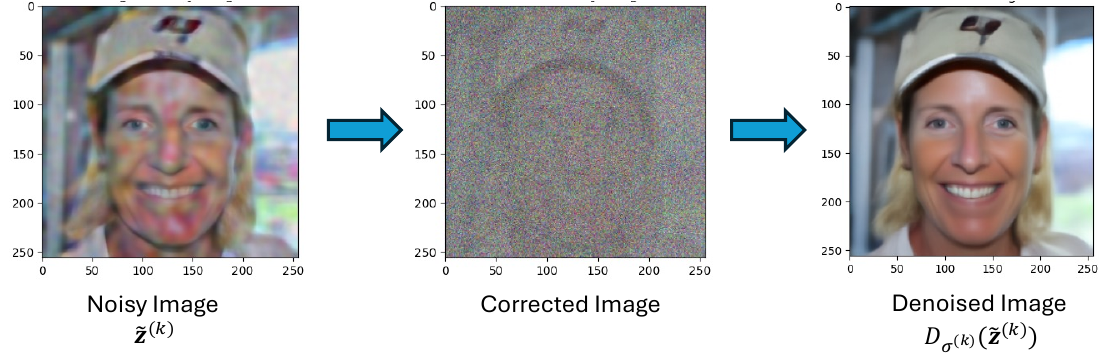}
    \caption{ Illustration of correction and denoising step in proposed method.}
    \label{fig: illustration of correction and denoising}
\end{figure}

\subsection{Empirical validation of Assumption~2 and Assumption~3}
To assess the practicality of the smoothness and coercivity assumptions used in Theorem~2, 
we conduct two diagnostic experiments using a pretrained score model on the validation split 
of the FFHQ dataset. These experiments are designed to evaluate (i) the empirical Lipschitz 
behavior of the score function $\nabla \log p_{\mathrm{data}}(\bx)$ (Assumption~2), and 
(ii) the coercivity of the energy landscape $-\log p_{\mathrm{data}}(\bx)$ (Assumption~3).

\vspace{0.5em}
\paragraph{Empirical smoothness of the score.}
We randomly 1000 samples of $\bx_1, \bx_2$ and compute score differences 
$\|\nabla \log p_{\mathrm{data}}(\bx_1) - \nabla \log p_{\mathrm{data}}(\bx_2)\|_2$ 
and image differences $\|\bx_1 - \bx_2\|_2$. 
Figure~\ref{fig:empirical_smoothness} plots the histogram of their ratio.
The distribution concentrates around a finite value (mostly between $50$ and $160$),
indicating that the score behaves approximately $M$-Lipschitz with a moderate empirical constant.
This supports the smoothness requirement in Assumption~\ref{assumpt: smoothness of logpdata}.

\begin{figure}[t]
    \centering
    \includegraphics[width=0.6\linewidth]{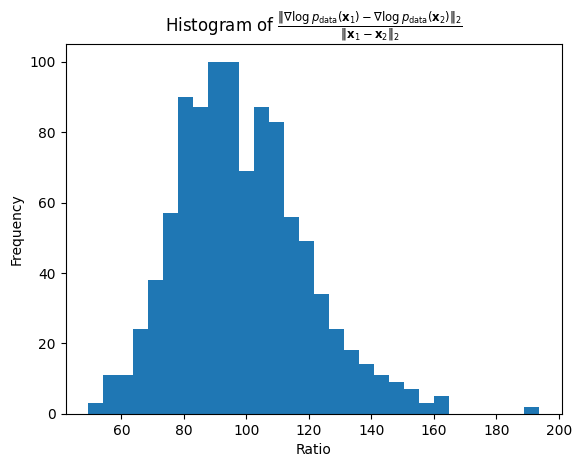}
    \caption{ Histogram of score difference norm ratio:
    $\|\nabla \log p_{\mathrm{data}}(\bx_2)-\nabla \log p_{\mathrm{data}}(\bx_1)\|_2
    / \|\bx_1-\bx_2\|_2$, illustrating empirical smoothness (Assumption~2).}
    \label{fig:empirical_smoothness}
\end{figure}

\vspace{0.5em}
\paragraph{Empirical coercivity.}
To evaluate coercivity, we scale images by factors $c \in \{1,1.5,2,3\}$ and measure 
the quantity $\langle \bx, -\nabla \log p_{\mathrm{data}}(\bx) \rangle$ as a function of 
the squared image norm $\|\bx\|_2^2$.
As shown in Figure~\ref{fig:empirical_coercivity}, the inner product grows approximately 
linearly with $\|\bx\|_2^2$, indicating that the learned score consistently pulls 
large-norm images back toward the data manifold.
This behavior is consistent with the coercivity structure assumed in Assumption~\ref{assumpt: coercivity of logpdata}.

\begin{figure}[t]
    \centering
    \includegraphics[width=0.6\linewidth]{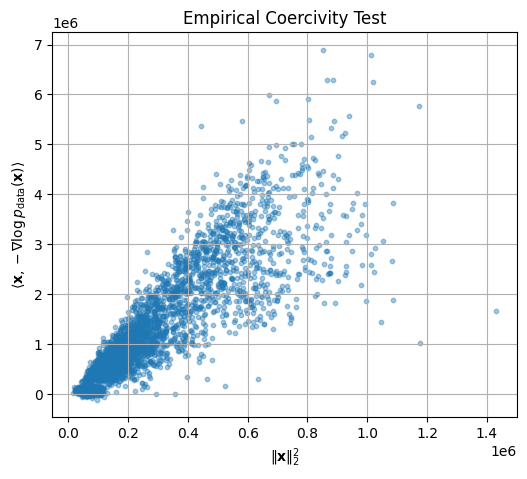}
    \caption{ Empirical coercivity test: relationship between
    $\langle \bx,-\nabla \log p_{\mathrm{data}}(\bx)\rangle$
    and $\|\bx\|_2^2$.
    The strong positive correlation indicates coercive energy behavior (Assumption~\ref{assumpt: coercivity of logpdata}).}
    \label{fig:empirical_coercivity}
\end{figure}

\vspace{0.5em}
\noindent
Together, these empirical diagnostics demonstrate that the theoretical assumptions employed 
in our analysis hold approximately in practice and therefore justify the use of DC 
correction in our AC--DC algorithm.

\subsection{Illustration of Usage of Additional Regularization}
\label{app:additional_reg}

To further demonstrate the flexibility of integrating diffusion-based PnP denoisers within the ADMM framework, we present an example where we employ an additional perceptual regularization term which will be handled in the maximum-likelihood (ML) step. In particular, the $\bx$-update step of ADMM with an LPIPS perceptual regularization ~\citep{zhang2018theunreasonable} becomes:
\begin{equation}
\bx^{(k+1)} = \arg\min_{\bx} \frac{1}{\rho}\ell(y \| \mathcal{A}(\bx)) +
\frac{1}{2}\|\bx - \bz^{(k)} + \bu^{(k)}\|_2^2
+ \lambda_{\text{lpips}} \,
\text{LPIPS}_{\mathrm{VGG}}(\bx, \bx_{\mathrm{ref}}),
\end{equation}
where $\bx_{\mathrm{ref}}$ is the reference image and $\lambda_{\text{lpips}}$ controls the perceptual strength.

This example highlights the flexibility of the proposed method: unlike traditional diffusion-based PnP approaches that struggle in the presence of dual variables, our design enables seamless incorporation of additional regularization terms. In Fig.~\ref{fig:style-reg-demo} we illustrate box inpainting reconstruction task with the perceptual LPIPS-VGG regularization which enhances semantic content consistency while allowing visual style transfer from the reference images. 

\vspace{4pt}
\begin{figure}[t]
    \centering
    \includegraphics[width=0.93\linewidth]{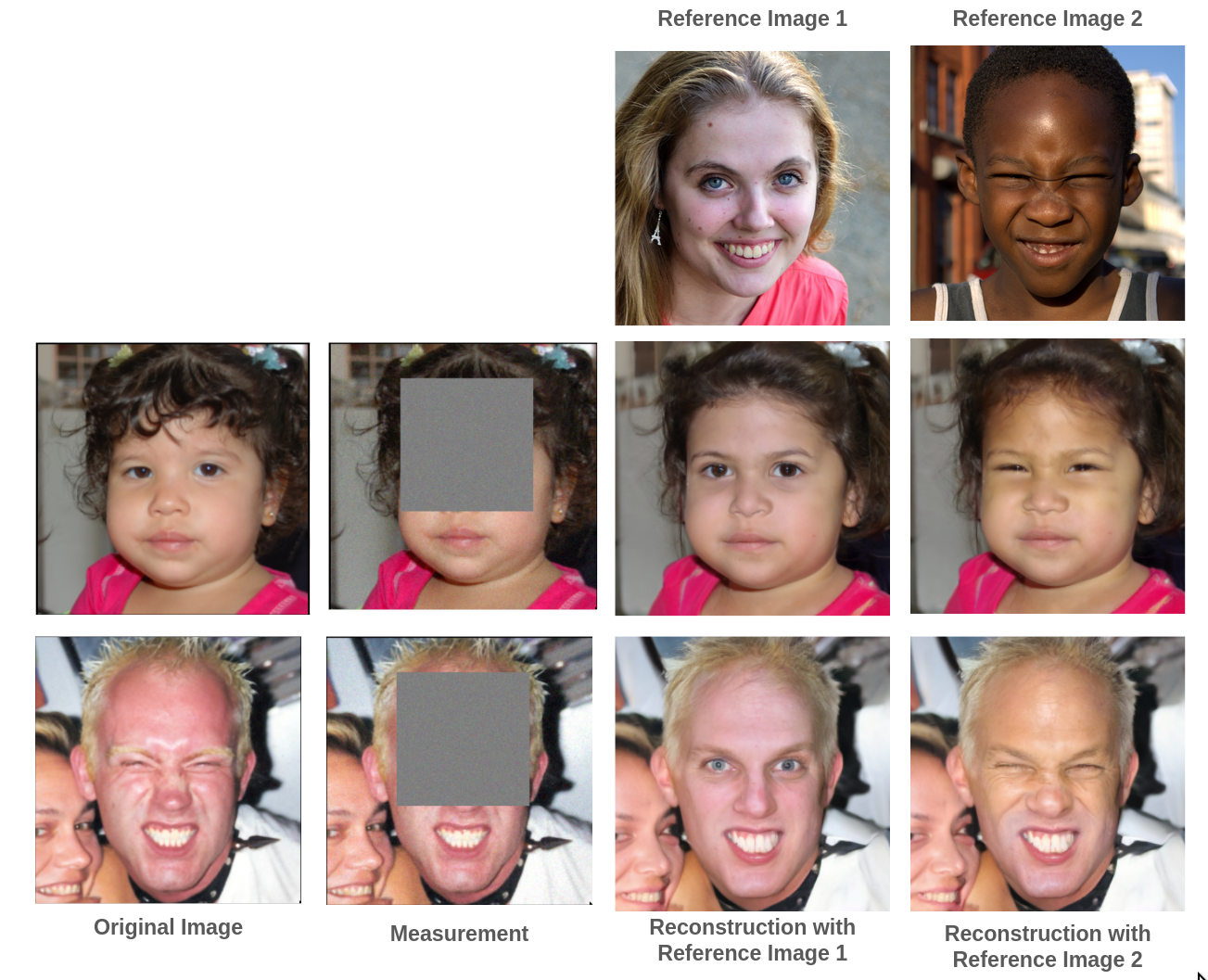}
    \caption{ Demonstration of incorporating additional perceptual regularization.}
    \label{fig:style-reg-demo}
\end{figure}
}

\subsection{Results on additional tasks}
The results on additional two tasks: HDR and nonlinear deblurring are presented in the Table~\ref{tab:reconstruction metrics on additional tasks}.
\begin{table}[t!] 
\centering
\tiny
\setlength{\tabcolsep}{2pt}
\renewcommand{\arraystretch}{0.9}
\caption{Reconstruction metrics (100 images) on FFHQ / ImageNet for additional tasks. \textbf{Bold}: best, {\blue blue}: 2nd best.}
\label{tab:reconstruction metrics on additional tasks}
\begin{minipage}{.48\linewidth}
  \centering
  %===== first block: tasks 1–3 =====%
  \begin{tabular}{l l rrr rrr}
    \toprule
    & & \multicolumn{3}{c}{FFHQ} & \multicolumn{3}{c}{ImageNet} \\
    \cmidrule(lr){3-5}\cmidrule(lr){6-8}
    Task & Method & PSNR↑ & SSIM↑ & LPIPS↓ & PSNR↑ & SSIM↑ & LPIPS↓ \\
    \midrule
    \multirow{7}{*}{\rotatebox{0}{\makecell[c]{HDR}}}
      & Ours–tweedie & \bfseries27.425 & \bfseries0.853  & \blue 0.164
                    & \blue 26.515  & \bfseries0.817  & \blue  0.182  \\
      & DAPS      & \blue 26.94 & \blue 0.852  & \bfseries0.154 
                    & \bfseries26.848  &\blue  0.816  & \bfseries 0.172  \\
    & RED-diff   & 26.815  & 0.836  & 0.241 
                    & 20.794  & 0.771  & 0.232  \\
    &  PMC &  21.582 &  0.707 &  0.291 
            &  22.745 &  0.707 &  0.290 \\
    \midrule
    \multirow{7}{*}{\rotatebox{0}{\makecell[c]{Nonlinear Deblur}}}
      & Ours–tweedie &  \bfseries29.326  &  \bfseries 0.823  &  \blue 0.185
                    & \bfseries27.837  & \blue  0.725  & 0.212  \\
      & DAPS      &  \blue 28.598  & \blue 0.782  & \bfseries0.172 
                    & \blue 27.745  & \bfseries0.739  & \bfseries 0.201  \\
      & DPS          & 23.746  & 0.668  & 0.276 
                    & 22.724  & 0.543  & 0.394  \\
      & RED-diff & 26.9  & 0.72  & 0.234 
                    & 25.488  & 0.72  & \blue 0.207  \\
    &  PMC &  21.102 &  0.0623 &  0.354 &  22.347 &  0.533 &  0.430 \\
    \bottomrule
  \end{tabular}
\end{minipage}
\end{table}

\subsection{Ablation study}
We perform the ablation study on the significance of our proposed correction steps. The results are presented in the Table~\ref{tab: abaltion results for the proposed correction method}.

\begin{table}[t!] 
\tiny
\centering
\caption{ Comparison of our method with and without correction on FFHQ. Best results are highlighted in bold.}
\label{tab: abaltion results for the proposed correction method}
\setlength{\tabcolsep}{2pt}
\renewcommand{\arraystretch}{0.9}
\begin{minipage}{0.48\linewidth}
\begin{tabular}{l rrr rrr}
\toprule
     & \multicolumn{3}{c}{Ours-tweedie without correction} & \multicolumn{3}{c}{Ours-tweedie with correction} \\
    \cmidrule(lr){2-4}\cmidrule(lr){5-7}
    Tasks & PSNR↑ & SSIM↑ & LPIPS↓ & PSNR↑ & SSIM↑ & LPIPS↓ \\
    \midrule
Superresolution (4x)  & 26.915 & 0.730 & 0.314 & \textbf{30.439} & \textbf{0.857} & \textbf{0.178} \\
Gaussian Blur           & 28.896 & 0.788 & 0.275 & \textbf{30.402} & \textbf{0.853} & \textbf{0.175} \\
Inpainting (Box)     & 15.604 & 0.617 & 0.361 & \textbf{24.025} & \textbf{0.859} & \textbf{0.131} \\
Motion Deblur        & 25.123 & 0.538 & 0.370 & \textbf{30.003} & \textbf{0.854} & \textbf{0.179} \\
Nonlinear Deblur   & 21.731 & 0.561 & 0.375 & \textbf{29.326} & \textbf{0.823} & \textbf{0.185} \\
Phase Retrieval & 11.978 & 0.181 & 0.726 & \textbf{27.944} & \textbf{0.793} & \textbf{0.209} \\
\bottomrule
\end{tabular}
\end{minipage}

\end{table}

\subsection{Influence of decay schedule and NFE efficiency}
In our ADMM-PnP scheme, the size of the decay window for $\sigmak$ determines the total number of iterations -- and thus the speed of convergence. A shorter window (small $W$) drives $\sigmak$ down more quickly, often reaching convergence in fewer steps but at the risk of settling in a suboptimal local minimum. 

To study this trade-off, we sweep
\[W \in \{5,10,50,100,200,300,400,500\}\]
for each task. Since each iteration of Ours-tweedie uses $11$ score evaluations (10 for the DC update and $1$ for the Tweedie's lemma based denoiser), these W value translate to
\[\text{Number of Function Evaluations (NFE)}=\{ 55, 110, 550, 1100, 2200, 3300, 4400, 5500  \} \]
By contrast, each Ours-ode iteration costs $20$ NFEs, giving
\[\text{Number of Function Evaluations (NFE)}=\{100, 200, 1000, 2000, 4000, 6000, 8000, 10000 \} \]

Figures~\ref{fig:nfe for superresolution}-\ref{fig:nfe for phase retrieval} plot mean$\pm$std. (standard deviation) performance of our methods and all baselines against NFE over $100$ images of FFHQ dataset. For most tasks, quality saturates after just $10$ iterations (110 NFE for Ours-tweedie, 200 NFE for Ours-ode), showing a rapid decay schedule suffices to achieve near-peak results. However, on the hardest inverse problems (phase retrieval and nonlinear blur), gradually decaying noise (larger $W$) and more NFEs yield significantly better reconstructions--far outpacing every baseline. Thus, while aggressive schedules excel on simple tasks, challenging problems benefit from extended iteration and gentler annealing; given enough NFEs, our approach establishes state-of-the-art performance across most of the tasks.

\begin{figure}[!t]
    \centering
    \includegraphics[width=1\linewidth]{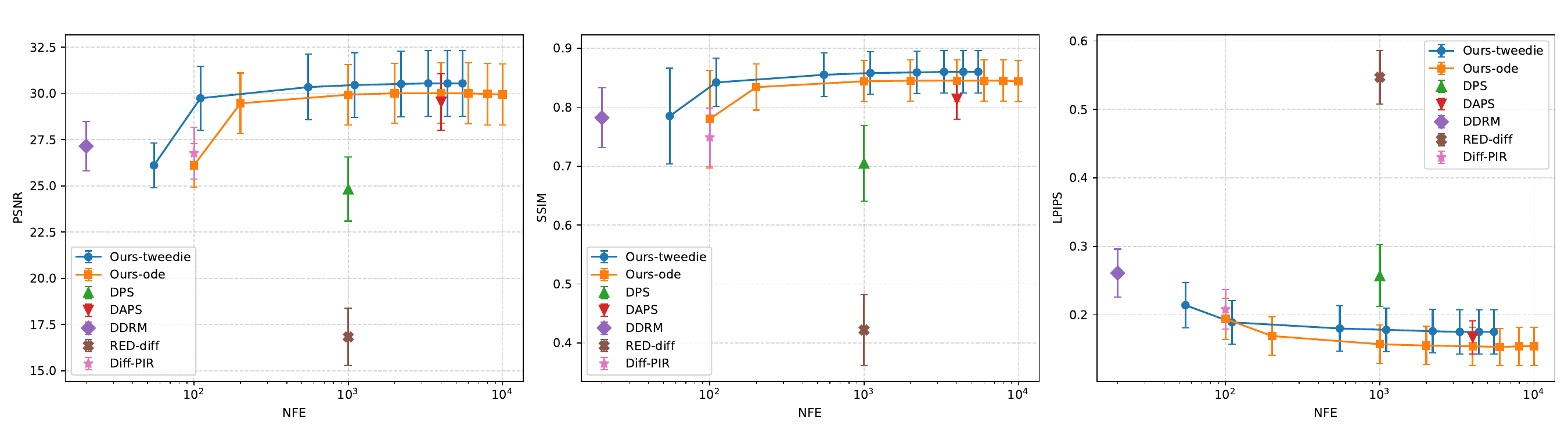}
    \caption{Performance with respect to NFE for Superresolution task (FFHQ)}
    \label{fig:nfe for superresolution}
\end{figure}
\begin{figure}[!t]
    \centering
    \includegraphics[width=1\linewidth]{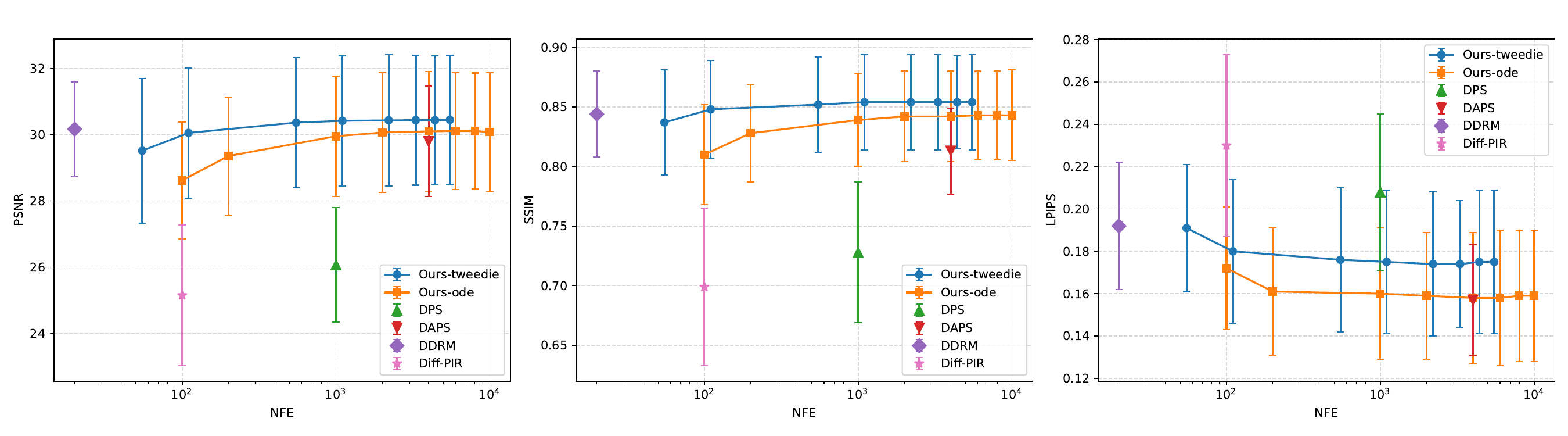}
    \caption{Performance with respect to NFE for Gaussian deblurring task (FFHQ)}
    \label{fig:nfe for gaussian deblurring}
\end{figure}

\begin{figure}[!t]
    \centering
    \includegraphics[width=1\linewidth]{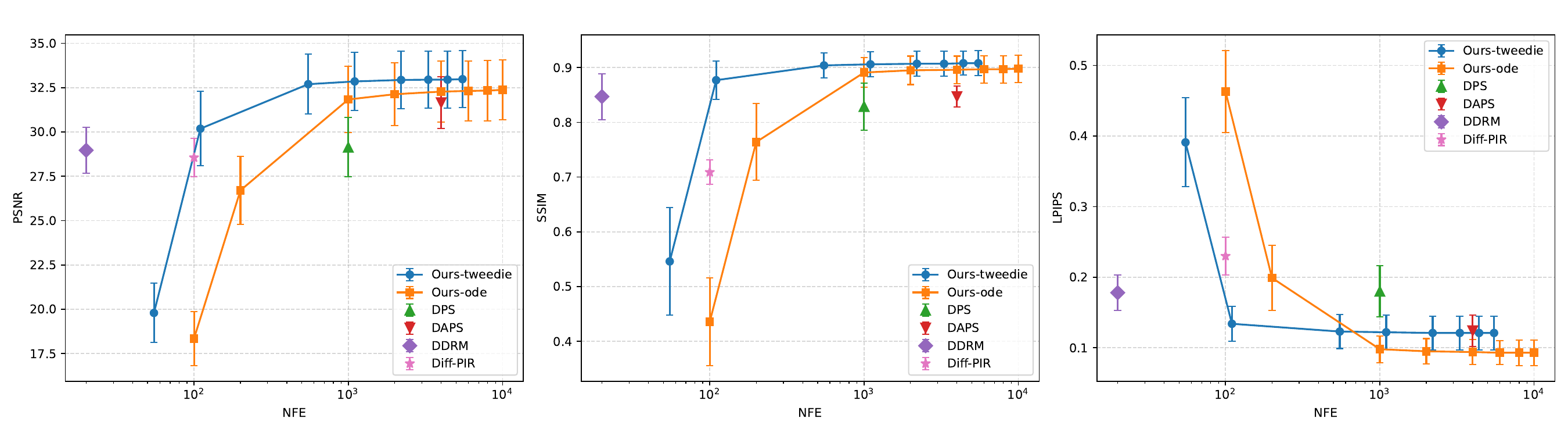}
    \caption{Performance with respect to NFE for Inpainting with random missings (FFHQ)}
    \label{fig:nfe for random impainting}
\end{figure}

\begin{figure}[!t]
    \centering
    \includegraphics[width=1\linewidth]{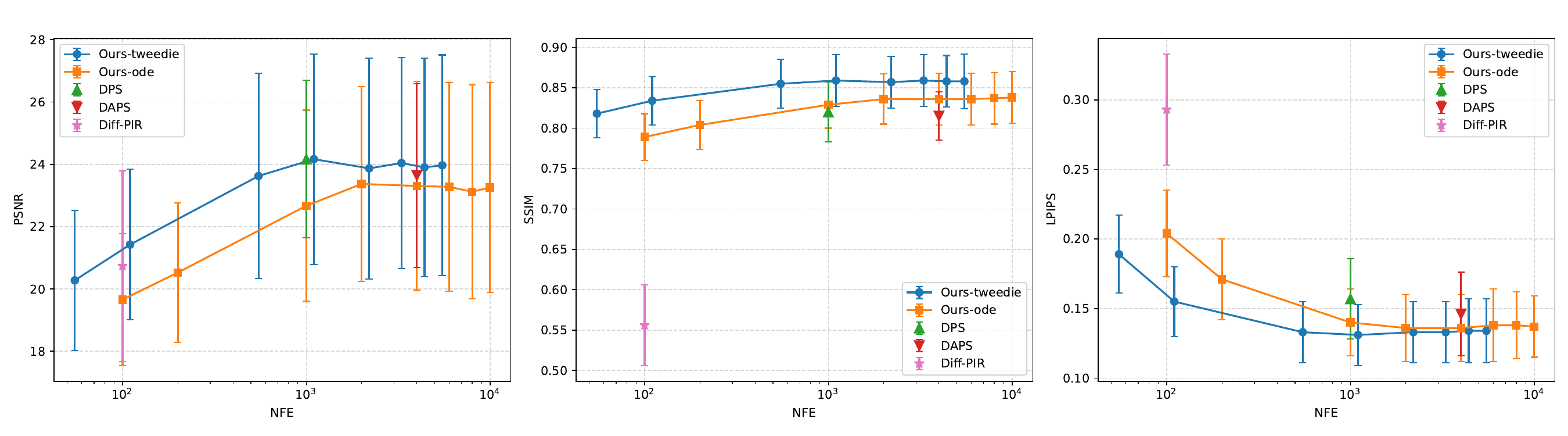}
    \caption{Performance with respect to NFE for Inpainting with box missing (FFHQ)}
    \label{fig:nfe for box impainting}
\end{figure}

\begin{figure}[!t]
    \centering
    \includegraphics[width=1\linewidth]{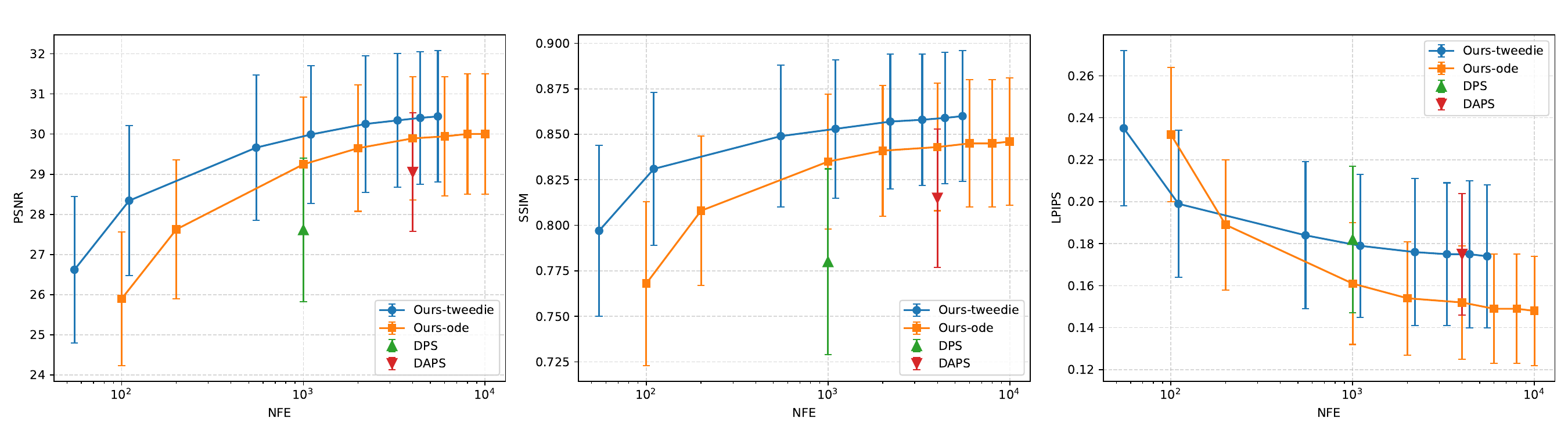}
    \caption{Performance with respect to NFE for Motion blur (FFHQ)}
    \label{fig:nfe for motion blur}
\end{figure}

\begin{figure}[!t]
    \centering
    \includegraphics[width=1\linewidth]{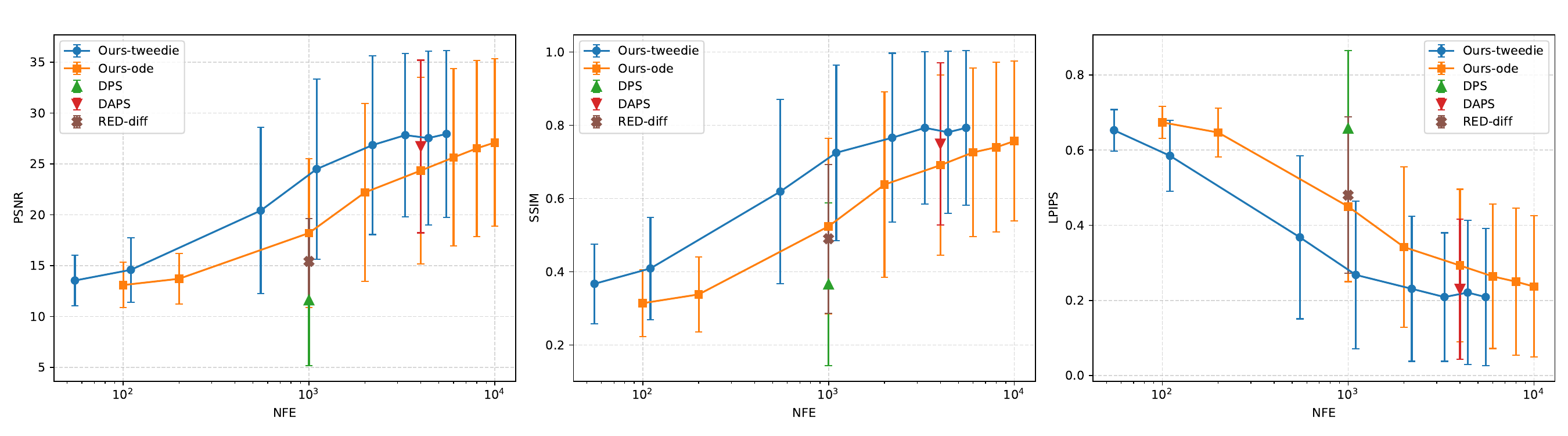}
    \caption{Performance with respect to NFE for Phase retrieval (FFHQ)}
    \label{fig:nfe for phase retrieval}
\end{figure}

\subsection{More qualitative results}
\begin{figure}[t!]
    \centering
    \includegraphics[width=\linewidth]{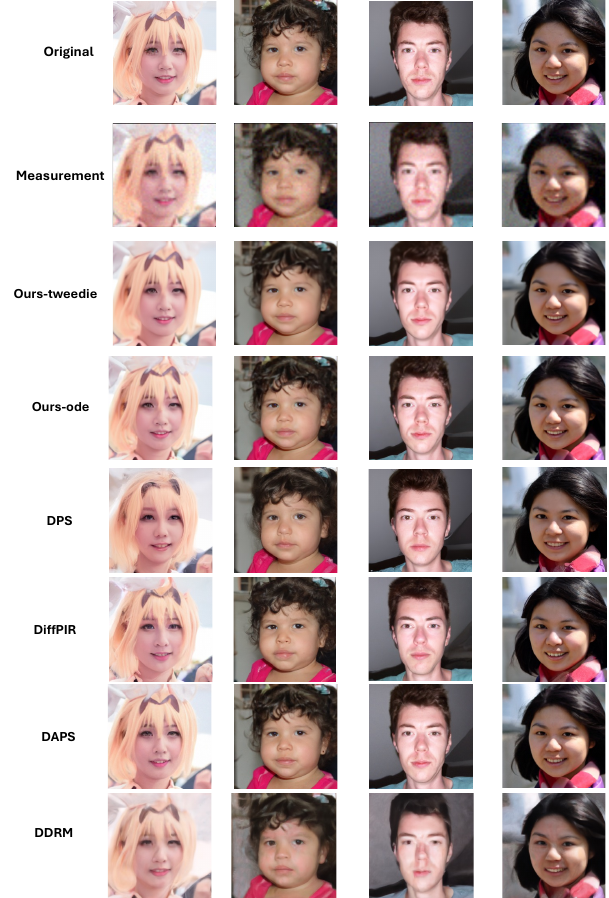}
    \caption{Recovery under $4\times$ superresolution task on FFHQ}
    \label{fig:addiitonal qualitative results under superresolution}
\end{figure}

  \begin{figure}[t!]
    \centering
    \includegraphics[width=\linewidth]{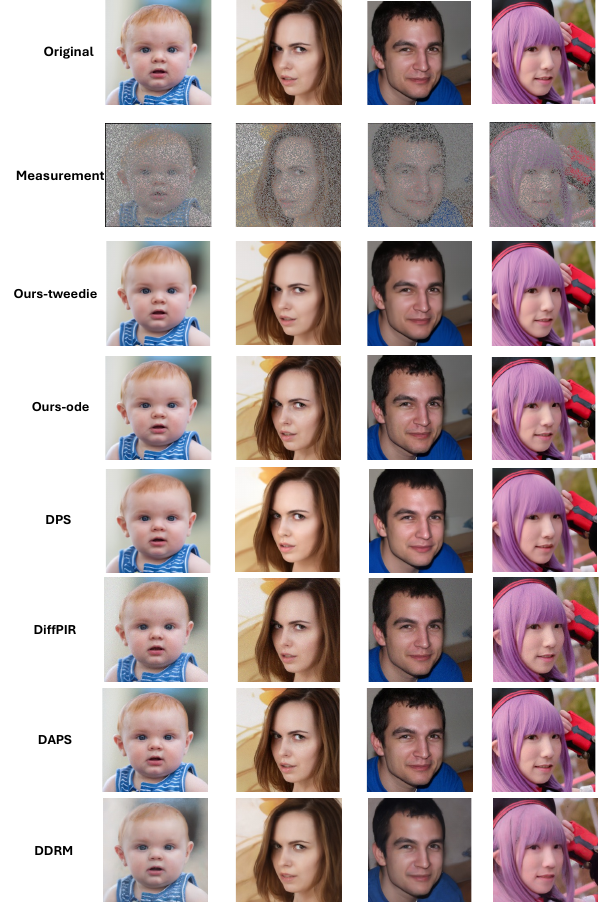}
    \caption{Recovery under inpainting with random missings on FFHQ}
    \label{fig:addiitonal qualitative results under random inpainting}
\end{figure}

  \begin{figure}[t!]
    \centering
    \includegraphics[width=\linewidth]{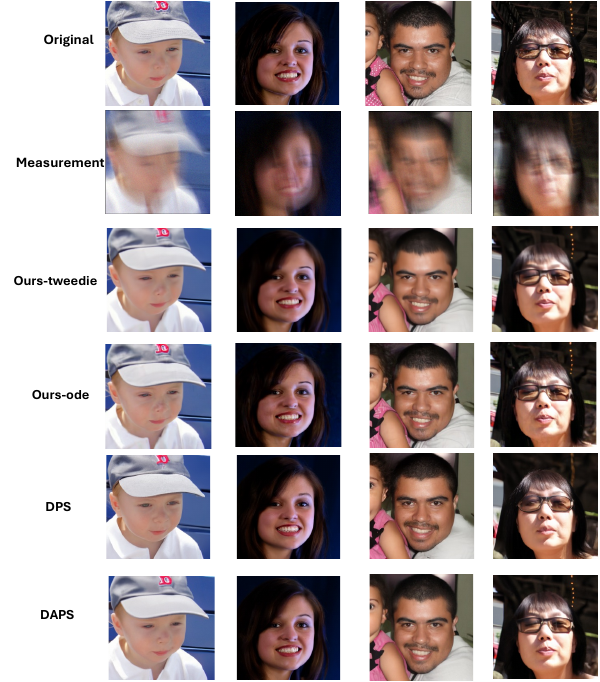}
    \caption{Recovery under motion blur task on FFHQ}
    \label{fig:addiitonal qualitative results under motion blur}
\end{figure}
\clearpage

\clearpage

\end{document}